%% file: main.tex
\documentclass[11pt]{article}
\usepackage{enumerate}
\usepackage{url}
\usepackage{pdflscape} 
\usepackage{rotating} 
\usepackage{booktabs}
\usepackage{threeparttable} 

\usepackage{geometry}
\RequirePackage{amsthm,amsmath,amsfonts,amssymb}
\RequirePackage[authoryear,sort]{natbib} 
\usepackage{xcolor}
\usepackage{xr-hyper}
\usepackage[pdftex,bookmarks,colorlinks,breaklinks]{hyperref} 

\definecolor{dullmagenta}{rgb}{0.4,0,0.4} 
\definecolor{darkblue}{rgb}{0,0,0.4}
\definecolor{coquelicot}{rgb}{0.20, 0.12, 0.72}
\definecolor{navyblue}{rgb}{0,0,0.5}
\hypersetup{linkcolor=blue,citecolor=blue,filecolor=blue,urlcolor=blue}

\newcommand\independent{\protect\mathpalette{\protect\independenT}{\perp}}
\def\independenT#1#2{\mathrel{\rlap{$#1#2$}\mkern2mu{#1#2}}}

\usepackage{epstopdf}
\usepackage{tablefootnote}

\usepackage{multirow}
\usepackage{mathtools}
\usepackage{bbm}
\usepackage{graphicx} 
\usepackage{flafter} 
\usepackage{subcaption}
\usepackage{graphicx}

\usepackage{tikz}
\usetikzlibrary{positioning}
\usetikzlibrary{arrows.meta,arrows}

\usetikzlibrary{arrows, decorations.markings,shapes,arrows,fit}
\tikzset{box/.style={draw, minimum size=2em, text width=4.5em, text centered},
	bigbox/.style={draw, inner sep=20pt,label={[shift={(-3ex,3ex)}]south east:#1}}
}

\usepackage{bm} 

\usepackage{colortbl}
 \usepackage{arydshln}
\usepackage{booktabs}

\usepackage{algorithm}
\usepackage{algpseudocode}

\definecolor{coquelicot}{rgb}{0.90, 0.42, 0.72}
\definecolor{burntorange}{rgb}{0.8, 0.33, 0.0}

\definecolor{burntblue}{RGB}{0, 114, 206}

\input{Definitions.tex}

\theoremstyle{plain}
\newtheorem{theorem}{Theorem}[section]

\newtheorem{lemma}[theorem]{Lemma}

\theoremstyle{remark}
\newtheorem{assumption}{Assumption}
\newtheorem{remark}{Remark}
\newtheorem{example}{Example}

\usepackage{xr}

\begin{document}

\date{}


\title{\bf Causal inference through multi-stage learning and doubly robust deep neural networks}

\author{Yuqian Zhang\thanks{Institute of Statistics and Big Data, Renmin University of China} \and Jelena Bradic\thanks{Department of Mathematics and Halicioglu Data Science Institute, University of California, San Diego, E-mail: \href{mailto:jbradic@ucsd.edu}{jbradic@ucsd.edu} }}

\maketitle

\begin{abstract}
Deep neural networks (DNNs) have demonstrated remarkable empirical performance in large-scale supervised learning problems, particularly in scenarios where both the sample size $n$ and the dimension of covariates $p$ are large. This study delves into the application of DNNs across a wide spectrum of intricate causal inference tasks, where direct estimation falls short and necessitates multi-stage learning. Examples include estimating the conditional average treatment effect and dynamic treatment effect. In this framework, DNNs are constructed sequentially, with subsequent stages building upon preceding ones. To mitigate the impact of estimation errors from early stages on subsequent ones, we integrate DNNs in a doubly robust manner. In contrast to previous research, our study offers theoretical assurances regarding the effectiveness of DNNs in settings where the dimensionality $p$ expands with the sample size. These findings are significant independently and extend to degenerate single-stage learning problems.
\end{abstract}

\section{Introduction}\label{sec:intro}

In numerous biomedical, economic, and political studies, we grapple with the challenges posed by modern large-scale data. On one hand, the substantial sample size allows us to surpass the limitations of traditional linear or parametric models, facilitating the derivation of more accurate and robust conclusions through non-parametric methods. However, on the other hand, modern big data often harbors copious amounts of redundant information, exacerbating what is commonly referred to as the ``curse of dimensionality.'' This phenomenon frequently renders traditional non-parametric statistical methods impractical. Navigating the task of deriving reasonable causal relationships from this extensive dataset to inform real-life decision-making is a prevalent yet formidable challenge.

Neural networks are arguably the most popular machine learning methods in large-scale industry applications these days. Unlike traditional non-parametric regression methods, such as kernel regression and smoothing splines, deep neural networks (DNNs) paired with rectified linear units (ReLU) activation functions \citep{nair2010rectified} and stochastic optimization techniques \citep{kingma2014adam} deliver outstanding empirical performance in large-scale complex estimation problems where both the sample size and the covariates' dimension are relatively large.

While the use of neural networks in simple regression problems already comes with certain theoretical guarantees, in many scenarios, we are no longer merely satisfied with predicting the future based on existing data. Instead, we aim to understand the causal relationships between variables to make informed decisions. For causal inference problems, one of the main differences from traditional regression and prediction problems lies in the potential outcome framework, where we can only observe the potential outcome corresponding to the treatment that the individual has been assigned to. In other words, we essentially face a missing data problem. As a result, the parameters or functions we are interested in cannot be directly estimated by performing regression methods on observable samples, as some of the related variables are missing.

Double machine learning (DML) \citep{chernozhukov2017double}, also known as doubly robust or augmented inverse probability weighting, provides a potent framework for such problems. When our object of interest is a finite-dimensional parameter $\theta$, we typically seek a representation $\theta=E\{\psi(\bZ;\eta^0)\}$, where $\bZ$ denotes the observable variables, $\eta^0$ represents a set of nuisance functions, and $\psi$ is a (non-centered) doubly robust score function satisfying the Neyman orthogonality condition
\begin{equation}\label{def:NO}
\partial_\eta E\{\psi(\bZ;\eta^0)\}(\eta-\eta^0)=0,
\end{equation}
where $\partial_\eta$ denotes the Gateaux derivative operator with respect to $\eta$. To estimate $\theta$, it suffices to estimate the nuisance functions $\eta^0$ and take the empirical average of the score functions after plugging in the nuisance estimates. In contrast to other methods, including inverse probability weighting (IPW) and G-computation, DML only requires a subset of the nuisance functions to be correctly specified for achieving consistent estimates of the final parameters of interest. Additionally, when all nuisance functions can be consistently estimated, DML typically leads to faster convergence rates. Examples of such finite-dimensional parameters include the regression coefficient in a partially linear regression model, the average treatment effect (ATE), the average treatment effect for the treated (ATTE), and the local average treatment effects (LATE); refer to \cite{chernozhukov2017double} for more details.

In more intricate causal inference scenarios, our target of estimation is a function $\theta(\bx)$ defined as the conditional expectation of unobservable variables, such as the conditional average treatment effect (CATE) in Example \ref{ex:CATE} below. When aiming to estimate such a conditional mean function, we look for representations of the form
\begin{equation}\label{rep:CE}
\theta(\bx)=E\{\psi(\bZ;\eta^0)\mid\bX=\bx\},
\end{equation}
where $\eta^0$ represents nuisance functions and $\psi$ is a score function. Given that the function $\theta(\bx)$ resides in a significantly more complex space, estimating such a function presents greater challenges than estimating a one-dimensional parameter, such as the ATE. Rather than simply taking the empirical average over the estimated score functions $\psi(\bZ_i;\etahat)$, specific regression methods are required to estimate the conditional mean function, utilizing the first-stage estimates $\etahat$. In this study, we focus specifically on the utilization of DNNs. Estimating $\theta(\bx)$ necessitates the construction of two DNNs in a sequential manner: one for related nuisance functions and another for the final target of interest, with the second DNN being constructed based on the first one. We refer to such a problem as a \emph{two-stage learning} problem.

While there are typically multiple choices for the score function $\psi$ that fulfill the representation \eqref{rep:CE}, we seek a score function designed to satisfy the following \emph{generalized} Neyman orthogonality condition: 
\begin{equation}\label{def:CNO}
\partial_\eta E\{\psi(\bZ;\eta^0)\mid\bX=\bx\}(\eta-\eta^0)=0,\quad\forall\bx\in\mathcal X,
\end{equation}
where $\mathcal{X}$ is the support of $\bX$. Note that the usual Neyman orthogonality condition \eqref{def:NO} can be viewed as a special case of the generalized version \eqref{def:CNO} when $\psi(\bZ;\eta)\independent\bX$ and $\theta(\bx)$ degenerates to a constant function. 

Similar to the usual Neyman orthogonality condition \eqref{def:NO}, the generalized version \eqref{def:CNO} is beneficial in that it leads to faster convergence rates and is suitable to use when the target parameter is infinite dimensional. Denote $Y^\#=\psi(\bZ;\eta^0)$ and $\Yhat=\psi(\bZ;\etahat)$ and as the true score and its plug-in estimate. As illustrated in Lemmas \ref{lemma:imput-err-CATE} and \ref{lemma:imput-err} below, the first-stage estimation error impacts the second learning stage through the deviation $\Yhat-Y^\#=\Delta_1+\Delta_2$, where $\Delta_1$ and $\Delta_2$ depend linearly and quadratically on the first-stage estimation error, respectively. When the condition \eqref{def:CNO} is fulfilled, $\Delta_1$ exhibits a conditional mean of zero and can be regarded as additional noise in the second learning stage, exerting no influence on the convergence rate for the second-stage estimation. This observation is underscored in our Theorems \ref{thm:imp-DNN}-\ref{thm:imp-MLP} for DNNs. Consequently, only the smaller second-order term $\Delta_2$ contributes to the second-stage error. It becomes evident that employing score functions that satisfy condition \eqref{def:CNO} is advantageous in mitigating the impact of estimation errors from preceding learning stages on subsequent ones.

In even more complex scenarios, a conditional mean function $\theta(\bx)$ that necessitates a two-stage learning estimation serves merely as a nuisance function in obtaining estimates for final parameters of interest. Examples include the estimation of the dynamic treatment effect (DTE) and the controlled direct effect, as illustrated in Examples \ref{ex:DTE} and \ref{ex:CDE} below. Once an estimate of $\theta(\bx)$ is obtained, an additional doubly robust estimation step is necessary to estimate the final parameter of interest. Under such circumstances, a \emph{three-stage learning} approach is required: first, construct DNN estimates for the nuisance functions that are directly estimable using the observed samples; then, construct DNN estimates for the remaining nuisance functions based on the results from the first stage; and lastly, obtain an estimate for the final parameter of interest utilizing all the nuisance estimates. This framework can be extended to \emph{multi-stage learning}, involving more than three learning stages. For the sake of readability and simplicity, we only present results for two- and three-stage learning.


It is worth noting that the issues of confounding and model misspecification are two main challenges that hinder accurate causal conclusions from observational studies. To ensure the ignorability conditions (e.g., Assumptions \ref{cond:ign} and \ref{cond:seq-ign}), researchers typically need to collect a sufficient amount of covariates to ensure that most of the confounding variables are included, even though many of the collected covariates might be redundant. Consequently, to reduce the impact of unmeasured confounders, it is typically inevitable to face the challenge of high dimensionality. Traditional non-parametric regression methods, such as kernel regression and splines, often struggle with the ``curse of dimensionality,'' as mentioned above. Meanwhile, although regularized parametric regression methods have been developed and serve as natural choices in high dimensions, their usage may lead to significant bias when the parametric models deviate from the truth. As a result, it is crucial to develop methods that are suitable for high-dimensional covariates while still maintaining model robustness.

Neural networks stand out as a non-parametric approach that exhibits exceptional empirical performance when both the sample size $n$ and covariate dimension $p$ are relatively large. Initial research on neural networks focused mainly on \emph{shallow} structures with \emph{smooth} activation functions and dates back to the 1980s; see, e.g., \cite{cybenko1989approximation,hornik1991approximation,hornik1989multilayer}. However, until recently, there have been few advances in the statistical theory. \cite{bauer2019deep,kohler2022estimation} recently explored the non-parametric regression problem for smooth and sparse composite functions. When the dimension $p$ is fixed, their network exhibits an $L_2$ error rate bounded by $n^{-2\beta/(2\beta+q)}$ (accompanied by logarithmic terms), where $\beta$ signifies the smoothness level, and $q$ denotes the sparsity level. However, such theoretical findings did not extend to the high-dimensional case with a growing $p$. Additionally, due to the challenges of numerical optimization and suboptimal empirical performance, shallow networks with smooth activation functions have become marginalized in recent times.

Recent advancements in stochastic gradient descent have significantly advanced networks based on deep structures and non-smooth Rectified Linear Unit (ReLU) activation functions, achieving notable empirical success in large-scale real-world applications. \cite{yarotsky2017error} first established upper bounds for the approximation error of DNNs with ReLU activation functions. Subsequently, \cite{schmidt2020nonparametric} explored sparse ReLU DNNs, while \cite{kohler2021rate} extended this work to fully connected networks with very deep or wide structures. They focused on the composition structure as in \cite{bauer2019deep}, achieving an \(L_2\) error of approximately \(n^{-2\beta/(2\beta+q)}\) for a conditional mean function with sparsity level \(q\) and smoothness level \(\beta\). However, these results hold only if the dimension \(p\) remains fixed. As noted in \cite{schmidt2020rejoinder}, ``since the input dimension \(p\) in deep learning applications is typically extremely large, a possible future direction would be to analyze neural networks with high-dimensional \(p \to \infty\) and compare the rates to other nonparametric procedures.'' In this work, we provide an initial attempt considering a growing dimension with the sample size.

Further research, including \cite{jiao2023deepb, liu2021besov, schmidt2019deep, nakada2020adaptive, chen2022nonparametric}, has investigated another form of sparse structure, assuming the covariates \(\bX\) are supported on a low-dimensional manifold. This assumption is particularly suitable for data types where the information of \(\bX\) can be represented in a low-dimensional space, such as image and voice data. For the causal effect estimation problem addressed in this paper, sparsity conditions on the conditional distributions of the type \(P_{Y\mid\bX}\) seem more appropriate and have been extensively studied in the existing literature under parametric models; see, e.g., \cite{farrell2015robust, tan2020model, avagyan2021high, smucler2019unifying, bradic2024high, zhang2021dynamic, athey2018approximate}. Therefore, we adhere to sparsity structures on \(P_{Y\mid\bX}\) and provide an initial attempt that allows for a growing dimension.



In the realm of causal setups, \cite{farrell2021deep} also considered a fixed dimension $p$ and investigated the DNNs' application to ATE estimation problems. Given that all the nuisance functions involved can be directly identified through observable variables, employing DNNs with the observed data alone is sufficient. In our work, we extend theoretical guarantees for the application of DNNs to address more intricate causal problems, such as the estimation of CATE and DTE. These problems involve constructing nested DNNs in a doubly robust manner -- the DNNs constructed at subsequent stages depends on previous ones, along with the usage of doubly robust score functions satisfying the condition \eqref{def:CNO}. We conduct a comprehensive analysis for the \emph{doubly robust nested DNNs}. Importantly, unlike the existing works, our analysis accommodates a growing number of the covariate dimension $p$.

We commence by presenting general theories for doubly robust nested DNNs; refer to Theorems \ref{thm:imp-DNN}-\ref{thm:imp-MLP}. Even in the degenerate non-nested non-parametric regression problem where single-stage learning is sufficient, to the best of our knowledge, these are the first results for DNNs allowing the covariates' dimension to grow with the sample size while imposing a sparse structure on the conditional distribution $P_{Y\mid\bX}$ (instead of $P_{\bX}$). Additionally, unlike \cite{schmidt2020nonparametric}, we do not require additional constraints to bound the weights of the DNNs, simplifying the optimization procedure. Moreover, Theorem \ref{thm:imp-MLP} considers the user-friendly multilayer perceptrons (i.e., fully connected networks), and our theory does not depend on an optimal structure among a collection of sparse networks. In contrast, finding the optimal structure is an NP-hard problem and is practically unattainable. To sum up, our considered networks are more computationally feasible. While \cite{farrell2021deep} also explored the performance of multilayer perceptrons under the degenerate non-nested scenario, our Theorem \ref{thm:imp-MLP} establishes a faster convergence rate when the true conditional mean function exhibits approximate sparsity and is applicable when the dimension $p$ grows; see Remark \ref{remark:rate-MLP} for further details.

\subsection{Organization}

In Section \ref{sec:msl}, we introduce general multi-stage learning problems and provide specific examples within the realm of causal inference. Section \ref{sec:imp-DNN} lays the foundation with a comprehensive theory for doubly robust nested DNNs, crucial for analyzing DNN-based multi-stage learning. Sections \ref{sec:CATE} and \ref{sec:DTE} explore the estimation of CATE and DTE using DNNs, exemplifying the two-stage and three-stage learning problems introduced in Section \ref{sec:msl}. Further discussion is provided in Section \ref{sec:dis}.


\subsection{Notation}

For any $a,b\in\R$, we denote $\lceil a\rceil$ be the smallest integer that is larger or equal to $a$, $a\land b:=\min(a,b)$ and $a\lor b:=\max(a,b)$. For sequences $a_n,b_n\geq0$, $a_n\ll b_n$ denotes $a_n=O(b_n)$; $a_n\gg b_n$ denotes $b_n=O(a_n)$; $a_n\asymp b_n$ denotes both $a_n=O(b_n)$ and $b_n=O(a_n)$ holds. For any set $A$, denote $|A|$ as the number of elements in $A$. For any sub-sample $\S'\subseteq\S$, define $E_{\S'}(\cdot)$ as the expectation taken over $\S'$ and $\Var_{\S'}(\cdot)$ as the corresponding variance. For any random vector $\bZ$, define $P_\bZ(\cdot)$ and $E_\bZ(\cdot)$ as the probability measure and the expectation taken over the joint distribution of $\bZ$, respectively. For any random variable $X$, define $\|X\|_\infty:=\inf\{\alpha>0:P(|X|>\alpha)=0\}$ as the essential supremum and $\|X\|_{\infty,P_\bZ}:=\inf\{\alpha>0:P_\bZ(|X|>\alpha)=0\}$ as the essential supremum defined through the probability measure $P_\bZ$. 

\section{Multi-stage learning}\label{sec:msl}

In traditional point estimation and supervised learning problems, the parameter of interest can be directly identified as the (conditional) expectation of observable variables. This includes, for instance, the population mean \(\theta = E(Y)\) with observable \(Y\) and the conditional expectation \(\theta(\bx) = E(Y \mid \bX = \bx)\) with observable \((\bX, Y)\). After collecting identical and independent (i.i.d.) samples, it suffices to take the empirical average or perform certain regression methods to estimate the parameter of interest. We refer to such problems as \emph{single-stage learning}. 

In this work, we focus on more complex \emph{multi-stage learning} problems where additional learning stages are required to identify and estimate the parameter of interest.

\subsection{Two-stage learning}

In complex scenarios, such as causal setups, some of the underlying variables suffer from missingness. Consequently, the parameter of interest cannot be directly identified through the observable variables within a single step. We illustrate the average treatment effect (ATE) estimation problem below, which is a degenerate special case of the two-stage learning problem considered in this work.

\begin{example}[Average treatment effect]\label{ex:ATE}
Let $\bZ=(\bS,T,Y)$ represent the observable variables, where $\bS\in\R^p$ denotes the covariates, $T\in\{0,1\}$ is a binary treatment indicator, and $Y\in\R$ is the observed outcome. Consider the potential outcome framework and suppose the existence of potential outcomes $Y(1),Y(0)\in\R$. For each individual, only one of the potential outcomes is observed with $Y=Y(T)$; the counterfactual one is missing. The average treatment effect (ATE) is defined as $\theta_{\mbox{\tiny ATE}}:=E\{Y(1)-Y(0)\}$. Since the difference $Y(1)-Y(0)$ is unobservable, we cannot directly estimate the ATE through the empirical average of the sample potential outcomes' differences. Instead, we need to represent the ATE parameter in a two-stage fashion. Under the ignorability (also known as no unmeasured confounding) condition $\{Y(1),Y(0)\}\independent T\mid\bS$, the ATE can be identified as:
\begin{equation}\label{rep:ATE}
\theta_{\mbox{\tiny ATE}}=E\left\{\mu^0(1,\bS)+\frac{T\{Y-\mu^0(1,\bS)\}}{\pi^0(\bS)}-\mu^0(0,\bS)-\frac{(1-T)\{Y-\mu^0(0,\bS)\}}{1-\pi^0(\bS)}\right\},
\end{equation}
where $\pi^0(\bs):=P(T=1\mid\bS=\bs)$ and $\mu^0(t,\bs)=E(Y\mid\bS=\bs,T=t)$ for each $t\in\{0,1\}$.
\end{example}

Example \ref{ex:ATE} illustrates a scenario where the final parameter of interest $\theta$ can be identified through a representation $\theta=E\{\psi(\bZ;\eta^0)\}$ and all the nuisance functions $\eta^0=(\pi^0,\mu^0)$ can be identified as the conditional expectations through the observable variables. Such problems have been extensively studied in the existing literature; for recent advances in high-dimensional or non-parametric problems, refer to \cite{chernozhukov2017double, farrell2021deep, farrell2015robust, tan2020model, smucler2019unifying, bradic2019sparsity}. The ATE estimation problem in Example \ref{ex:ATE} is relatively simple, as the final parameter of interest is finite-dimensional, and the related nuisance functions are easy to estimate. In this work, we further consider more complex causal inference problems and study the applications of neural networks under more complicated situations. 

\begin{example}[Conditional average treatment effect]\label{ex:CATE}
Consider the same setup as in Example \ref{ex:ATE}. The conditional average treatment effect (CATE) is defined as $\theta_{\mbox{\tiny CATE}}(\bs):=E\{Y(1)-Y(0)\mid\bS=\bs\}$. Since $Y(1)-Y(0)$ is unobservable, we cannot directly employ regression methods based on the observable variables. Under the same conditions as in Example \ref{ex:ATE}, one way to identify the CATE is through the following doubly robust representation:
\begin{align}
\theta_{\mbox{\tiny CATE}}(\bs)&=E(Y^\#\mid\bS=\bs),\;\;\mbox{where}\nonumber\\
Y^\#&=\mu^0(1,\bS)+\frac{T\{Y-\mu^0(1,\bS)\}}{\pi^0(\bS)}-\mu^0(0,\bS)-\frac{(1-T)\{Y-\mu^0(0,\bS)\}}{1-\pi^0(\bS)}.\label{rep:CATE}
\end{align}
\end{example}

In Example \ref{ex:CATE}, our ultimate object of interest is a function that can be represented as $\theta(\bx) = E\{\psi_2(\bZ; \eta^0) \mid \bX = \bx\}$, where $\bX$ is a sub-vector of $\bZ$, $\psi_2$ is a given score function, and $\eta^0 = (\eta_1^0, \dots, \eta_w^0)$ denote the nuisance functions, with $w$ being a positive integer. Furthermore, the nuisance functions can be identified through the form $\eta_j(\bx^j) = E\{\psi_{1,j}(\bZ) \mid \bX^j = \bx^j\}$ for each $j \leq w$, with $\bX^j$ being a sub-vector of $\bZ$ and $\psi_{1,j}$ being a given score function. In general, to estimate any function $\theta(\bx)$ that can be identified through the above two-stage representations, it suffices to perform the following two-stage learning based on i.i.d. samples $(\bZ_i)_{i=1}^N \sim P_\bZ$: 
\begin{itemize}
\item[1.] Regressing $\psi_{1,j}(\bZ_i)$ on $\bX_i^j$ for each $j \leq w$ and obtain $\etahat = (\etahat_1, \dots, \etahat_w)$.
\item[2.] Regressing $\psi_{2}(\bZ_i, \etahat)$ on $\bX_i$ and obtain $\thetahat(\bx)$. 
\end{itemize}
One can also perform sample splitting (or cross-fitting) techniques to reduce the bias resulting from the nuisance estimation, as seen in \cite{chernozhukov2017double}.

Although both the ATE and CATE estimation problems introduced in Examples \ref{ex:ATE} and \ref{ex:CATE} above lie within the scenario of two-stage learning, estimating the CATE is harder than the ATE. This is because the CATE is an infinite-dimensional parameter (unless assuming certain parametric models) and hence lies within a much more complex space than the one-dimensional ATE. On the other hand, the ATE estimation problem is a degenerate situation of the two-stage learning, as the second learning stage degenerates to a constant function estimation problem and can be easily estimated through taking the empirical average.

While designing the first-stage score functions $\psi_{1,j}$ is natural, as the nuisance functions can be directly identified through the observable variables, the choice of the second-stage score function $\psi_2$ is usually non-unique. In this work, we specifically focus on the doubly robust score functions that satisfy the generalized Neyman orthogonality condition \eqref{def:CNO}. Such doubly robust score functions are helpful to obtain more accurate estimates for the final parameter of interest $\theta(\bx)$. For instance, although the CATE can be readily identified through the difference of conditional means, $\theta_{\mbox{\tiny CATE}}(\bs) = \mu^0(1,\bs) - \mu^0(0,\bs)$, relying on this straightforward identification may yield suboptimal results when the conditional mean functions $\mu^0(1,\bs)$ and $\mu^0(0,\bs)$ are significantly more complex than the CATE itself. The estimation of CATE based on the doubly robust representation \eqref{rep:CATE}, known as the DR-learner and initially proposed by \cite{van2006statistical}, yields superior convergence results, especially when the propensity score function $\pi^0(\bs)$ can be accurately estimated. 

Recently, \cite{kennedy2023towards} explored the use of linear smoothers for the CATE estimation problem, which is suitable only for classical low-dimensional settings. \cite{foster2023orthogonal} proposed a general framework allowing for flexible regression methods, including a brief introduction to results based on DNNs. However, they simply assumed that the CATE lies within the class of DNNs and provided a slower convergence rate along with less robust results; see detailed comparisons in Remark \ref{remark:CATE}. Notably, these existing works only addressed two-stage learning problems. In this work, we aim to offer a comprehensive analysis of employing DNNs in multi-stage learning problems. Section \ref{sec:CATE} will focus on DNN-based CATE estimation, serving as an illustration for two-stage learning problems, though the method's applicability extends to broader scenarios, including other two-stage learning problems and problems requiring three or more stages.


\subsection{Three-stage learning}

In the following, we further consider situations where estimating the ultimate objects of interest requires more than two stages of learning.

\begin{example}[Dynamic treatment effect with two exposures]\label{ex:DTE}
Consider a dynamic setup with two exposures of treatment assignments. Let $\bZ=(\bS_1,T_1,\bS_2,T_2,Y)$ be the observable variables, where $\bS_1\in\R^{d_1}$ and $\bS_2\in\R^{d_2}$ denote the covariates evaluated before the first and second exposures, $T_1,T_2\in\{0,1\}$ are binary treatment indicators assigned at the first and second exposures, and $Y\in\R$ is the observed outcome at the last stage. Let $\bT:=(T_1,T_2)$ and $Y(\bt)\in\R$ be the (unobservable) potential outcome for any $\bt=(t_1,t_2)\in\{0,1\}^2$. Assume the consistency $Y=Y(T_1,T_2)$ and the sequential ignorability $Y(\bt)\independent T_1\mid\bS_1$ and $Y(c)\independent T_2\mid(\bSbar_2,T_1=t_1)$, where $\bSbar_2:=(\bS_1^\top,\bS_2^\top)^\top$. For any $\ba,\ba'\in\{0,1\}^2$, the dynamic treatment effect (DTE) between paths $\ba$ and $\ba'$, is defined as $\theta_{\mbox{\tiny DTE}}:=E\{Y(\ba)-Y(\ba')\}$. In the following, we introduce the identification of $E\{Y(1,1)\}$, which can be generalized to the identification of $E\{Y(\bt)\}$ for any $\bt\in\{0,1\}^2$:
\begin{equation}\label{rep:DTE}
E\{Y(1,1)\}=E\left[\mu^0(\bS_1)+\frac{T_1\{\nu^0(\bSbar_2)-\mu^0(\bS_1)\}}{\pi^0(\bS_1)}+\frac{T_1T_2\{Y-\nu^0(\bSbar_2)\}}{\pi^0(\bS_1)\rho^0(\bSbar_2)}\right],
\end{equation}
where $\pi^0(\bs_1):=P(T_1=1\mid\bS_1=\bs_1)$, $\rho^0(\bsbar_2):=P(T_2=1\mid\bSbar_2=\bsbar_2,T_1=1)$, $\mu^0(\bs_1):=E\{Y(1,1)\mid\bS_1=\bs_1,T_1=1\}$, and $\nu^0(\bsbar_2):=E\{Y(1,1)\mid\bSbar_2=\bsbar_2,T_1=T_2=1\}$ are the nuisance functions. Note that three of the nuisance functions, $\pi^0$, $\rho^0$, and $\nu^0$ can be directly identified through observable variables. However, the remaining nuisance function $\mu^0$, representing the conditional potential outcome given all the information at and before the first treatment exposure, cannot be directly identified as $Y(1,1)$ is not observable within the group of $T_1=1$. To identify the function $\mu^0$, we can further consider a doubly robust representation:
\begin{equation}\label{rep:mu}
\mu^0(\bs_1)=E\left(Y^\#\mid\bS_1=\bs_1, T_1=1\right),\;\;\mbox{where}\;\;Y^\#=\nu^0(\bSbar_2)+\frac{T_2\{Y-\nu^0(\bSbar_2)\}}{\rho^0(\bSbar_2)}.
\end{equation}
\end{example}

Generally speaking, Example \ref{ex:DTE} illustrates a situation where the parameter of interest can be represented as $\theta = E\{\psi_3(\bZ,\eta_1^0,\eta_2^0)\}$, where $\psi_3$ is a score function, $\eta_1^0=(\eta_{1,1}^0,\dots,\eta_{1,w_1}^0)$ are the directly-identifiable nuisance functions, and $\eta_2^0=(\eta_{2,1}^0,\dots,\eta_{2,w_2}^0)$ are the remaining nuisance functions, with $w_1,w_2$ being positive integers. The nuisance functions $\eta_2^0$ can be further identified as $\eta_{2,j}^0(\bx^{2,j})=E\{\psi_{2,j}(\bZ,\eta_1^0)\mid\bX^{2,j}=\bx^{2,j}\}$ for each $j\leq w_2$, where $\psi_{2,j}$ is the score function for $\eta_{2,j}^0$ and $\bX^{2,j}$ is a sub-vector of $\bZ$. Lastly, $\eta_1^0$ can be directly identified as $\eta_{1,j}^0(\bx^{1,j})=E\{\psi_{1,j}(\bZ)\mid\bX^{1,j}=\bx^{1,j}\}$ for each $j\leq w_1$, where $\psi_{1,j}$ is the score function for $\eta_{1,j}^0$ and $\bX^{1,j}$ is a sub-vector of $\bZ$. Based on the above three-stage representations, we can estimate $\theta$ through a three-stage learning process:
\begin{itemize}
\item[1.] Regressing $\psi_{1,j}(\bZ_i)$ on $\bX_i^{1,j}$ for each $j \leq w_1$ and obtain $\etahat_1 = (\etahat_{1,1}, \dots, \etahat_{1,w_1})$.
\item[2.] Regressing $\psi_{2,j}(\bZ_i, \etahat_1)$ on $\bX_i^{2,j}$ for each $j\leq w_2$ and obtain $\etahat_2 = (\etahat_{2,1}, \dots, \etahat_{2,w_2})$. 
\item[3.] Taking the empirical average over $\psi_3(\bZ_i,\etahat_1^0,\etahat_2^0)$ and obtain $\thetahat$.
\end{itemize}

Although the ultimate object of interest considered in Example \ref{ex:DTE} is also a one-dimensional parameter, its identification is more complicated than the ATE in Example \ref{ex:ATE} -- as discussed above, the identification of related nuisance functions is no longer a trivial problem. As a result, an additional learning stage is required to estimate the DTE parameter. 

In Example \ref{ex:DTE}, we consider doubly robust score functions $\psi_{2,j}$ and $\psi_3$ satisfying the generalized Neyman orthogonality condition \eqref{def:CNO}. Although both $E\{Y(\bt)\}$ and $\mu^0$ can be identified through other representations, e.g., the G-formulas $E\{Y(\bt)\}=E\{\mu_c^0(\bS_1)\}$ and $\mu_c^0(\bs_1) = E\{\nu_c^0(\bSbar_2)\mid\bS_1=\bs_1,T_1=c_1\}$, the doubly robust representations \eqref{rep:DTE} and \eqref{rep:mu} enhance robustness. This strategy has been introduced by \cite{bradic2024high, diaz2023nonparametric, luedtke2017sequential, rotnitzky2017multiply}. Among these works, \cite{diaz2023nonparametric} provided results for the final DTE estimate, assuming certain conditions on the nuisance estimates without analyzing the nuisance estimation. However, it is worth noting that the nuisance estimation problem is arguably one of the most challenging aspects of the problem -- the nuisance functions are estimated in a sequential manner, and understanding how the estimation errors at previous stages affect the subsequent ones and the final DTE needs thorough study. \cite{luedtke2017sequential} considered Donsker class of nuisance estimates, primarily suitable for low-dimensional parametric models. \cite{rotnitzky2017multiply} explored linear operators, effective only in low dimensions with data-independent representations. \cite{bradic2024high} extended the study to include Lasso-type nuisance estimates in high dimensions but focused exclusively on parametric working models. Additionally, \cite{zhang2021dynamic} proposed a sequential model doubly robust approach that provides valid inference allowing for certain model misspecification in high dimension, but still requiring some of the parametric models to be correctly specified. Notably, all existing works either addressed low-dimensional situations or relied on parametric models, with none accommodating both high-dimensional and non-parametric models. In Section \ref{sec:DTE}, we will introduce a DNN-based DTE estimator and establish theoretical guarantees, only necessitating smoothness conditions while accommodating diverging dimensions.


Below, we further illustrate a causal mediation analysis problem, which is analogous to Example \ref{ex:DTE}, considering the mediator as the second-stage treatment variable.

\begin{example}[Controlled direct effect]\label{ex:CDE}
Consider the causal mediation analysis. Let $\bZ=(\bS_1,T,\bS_2,M,Y)$, where $\bS_1\in\R^{d_1}$ denotes the baseline covariates and $\bS_2\in\R^{d_2}$ represents the covariates evaluated subsequent to the treatment assignment but before the mediation. Denote $\bSbar_2:=(\bS_1^\top,\bS_2^\top)^\top$. Let $M\in\mathcal M$ be a discrete mediator with support $\mathcal M\subseteq\R$ and $Y\in\R$ be the observed outcome variable. Denote $M(t)$ and $Y(t,m)$ be the potential mediator and outcome, respectively, for any $t\in\{0,1\}$ and $m\in\mathcal M$. The observed mediator and outcomes are $M=M(T)$ and $Y=Y(T,M(T))$. Assume the ignorability conditions $Y(t,m)\independent T\mid\bS_1$ and $Y(t,m)\independent M\mid(\bSbar_2,T=t)$. Then, for any $m\in\mathcal M$, the controlled direct effect, $\theta_{\mbox{\tiny CDE}}(m):=E\{Y(1,m)-Y(0,m)\}:=\theta_{1,m}-\theta_{0,m}$, can be identified through the following representation: for each $t\in\{0,1\}$,
\begin{equation*}
\theta_{1,m}=E\left[\mu_{t,m}^0(\bS_1)+\frac{\mathbbm1_{T=t}\{\nu_{t,m}^0(\bSbar_2)-\mu_{t,m}^0(\bS_1)\}}{\pi_t^0(\bS_1)}+\frac{\mathbbm1_{T=t,M=m}\{Y-\nu_{t,m}^0(\bSbar_2)\}}{\pi_t^0(\bS_1)\rho_{t,m}^0(\bSbar_2)}\right],
\end{equation*}
where $\pi_t^0(\bs_1):=P(T=t\mid\bS_1=\bs_1)$, $\rho_{t,m}^0(\bsbar_2):=P(M=m\mid\bSbar_2=\bsbar_2,T=t)$, $\mu_{t,m}^0(\bs_1):=E\{Y(t,m)\mid\bS_1=\bs_1,T=t\}$, and $\nu_{t,m}^0(\bsbar_2):=E(Y\mid\bSbar_2=\bsbar_2,T=t,M=m)$. Similarly as in Example \ref{ex:DTE}, the nuisance function $\mu_{t,m}^0$ can be identified as
\begin{equation*}
\mu_{t,m}^0(\bs_1)=E\left[\nu_{t,m}^0(\bSbar_2)+\frac{\mathbbm1_{M=m}\{Y-\nu_{t,m}^0(\bSbar_2)\}}{\rho_{t,m}^0(\bSbar_2)}\mid\bS_1=\bs_1, T=t\right].
\end{equation*}
\end{example}

Examples \ref{ex:CATE} and \ref{ex:CDE} illustrate situations where three learning stages are required. In fact, the framework can be extended to more general situations with more than three learning stages. This includes, for instance, the estimation of dynamic treatment effects (DTE) with more than two time exposures; see, for example, Section 5 of \cite{bradic2024high}. Similar results can be extended under such general cases, though additional proofs, similar but cumbersome, would be required.

The multi-stage learning problems described above necessitate the sequential application of regression methods. In this paper, our particular focus is on employing state-of-the-art DNNs, known for their exceptional empirical performance. To rigorously examine the multi-stage learning problem, we begin by presenting results for the doubly robust nested DNNs, recognizing that DNNs constructed at subsequent learning stages depend on those constructed at preceding stages.

\section{Doubly robust nested DNNs}\label{sec:imp-DNN}

In this section, we introduce the construction of DNNs and present general theoretical results for DNNs based on nested structure. 

\subsection{Construction of DNNs}\label{sec:DNN}

We begin by introducing the construction of DNNs based on rectified linear unit (ReLU) activation functions. A neural network is a function $f(\bx)$ that maps $\R^p\mapsto\R$. The network is constructed with an input layer comprising $p$ input units, $L$ hidden layers, and an output unit. In each hidden layer $l\in\{1,\dots,L\}$, there are $H^{(l)}$ hidden units denoted as $z_1^{(l)},\dots,z_{H^{(l)}}^{(l)}$. The output unit is represented as $f(\bx)=z_1^{(L+1)}$ with $H^{(L+1)}=1$. Initially, we consider a general class of feedforward networks.

For any $1\leq l\leq L+1$ and $1\leq k\leq H^{(l)}$, a hidden or output unit $z_k^{(l)}$ is connected with certain units, possibly belonging to any of the preceding layers. We denote $\mathcal E_k^{(l)}\subseteq\{(l',k'):0\leq l'\leq l-1,1\leq k'\leq H^{(l')}\}$ as the collection of pairs $(l',k')$ such that $z_{k'}^{(l')}$ is connected with $z_k^{(l)}$. For simplicity, we denote the input units as $z_1^{(0)}=\bx_1,\dots,z_p^{(0)}=\bx_p$ with $H^{(0)}=p$. The hidden and output units are then constructed as follows:
\begin{align}\label{def:zkl}
z_k^{(l)}=\sigma^{(l)}\left(\sum_{(l',k')\in\mathcal E_k^{(l)}}w_{k,k'}^{(l,l')}z_{k'}^{(l')}+b_k^{(l)}\right),\quad\forall l\in\{1,\dots,L+1\},\;\;k\in\{1,\dots,H^{(l)}\},
\end{align}
where $w_{k,k'}^{(l,l')}$ and $b_k^{(l)}$ are the weight and intercept (bias) parameters. For each $l\leq L$, we utilize the ReLU activation function $\sigma^{(l)}(x)=x\lor0$; whereas, the output layer employs the identity activation function $\sigma^{(L+1)}(x)=x$. Note that we allow connected edges between units from non-neighboring layers.

\begin{figure}[h]
 \centering
 \begin{subfigure}[b]{0.45\textwidth}
 \includegraphics[width=\textwidth]{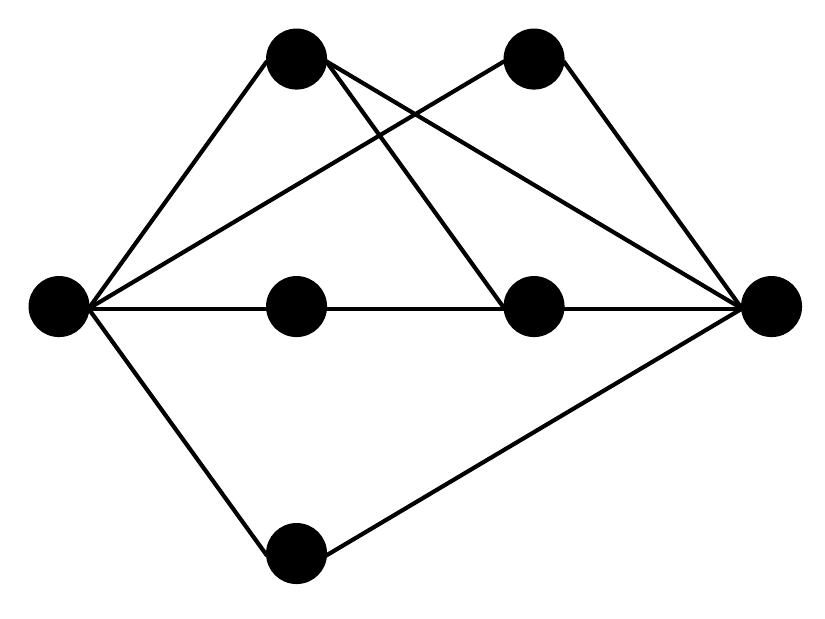}
 \caption{A network belonging to $\mathcal F_\mathrm{DNN}(2,6,10)$}
 \label{fig:network-DNN}
 \end{subfigure}
 ~
 \begin{subfigure}[b]{0.45\textwidth}
 \includegraphics[width=\textwidth]{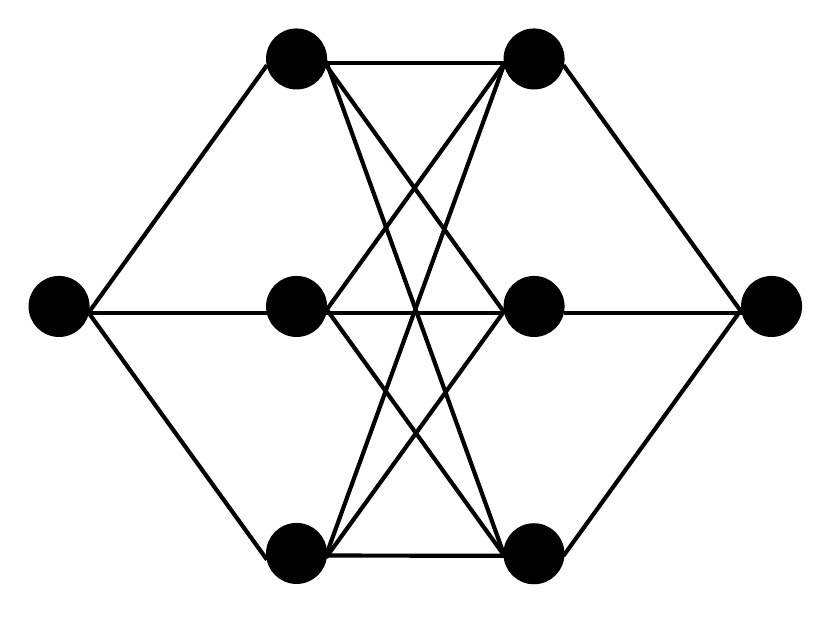}
 \caption{The multilayer perceptron $\mathcal F_{\mbox{\tiny MLP}}(2,3)$}
 \label{fig:network-MLP}
 \end{subfigure}
 \caption{Illustrations of general feedforward networks (left) and multilayer perceptrons (right)}
 \label{fig:network}
\end{figure}

Let $\mathcal F_\mathrm{DNN}=\mathcal F_\mathrm{DNN}(L,(H^{(l)})_{l\leq L+1},(\mathcal E_k^{(l)})_{l\leq L+1,k\leq H^{(l)}})$ be the collection of neural networks defined through \eqref{def:zkl} with a user-chosen network architecture. Additionally, we focus on two specific classes of network architectures. Denote $\mathcal F_\mathrm{DNN}(L,U,W)$ as the collection of feedforward neural networks with $L$ hidden layers, $U$ computation units, and $W$ weights. In other words, it is the union of $\mathcal F_\mathrm{DNN}(L,(H^{(l)})_{l\leq L+1},(\mathcal E_k^{(l)})_{l\leq L+1,k\leq H^{(l)}})$ satisfying $\sum_{l=1}^{L+1}H^{(l)}=U$ and $\sum_{l\leq L+1,k\leq H^{(l)}}|\mathcal E_k^{(l)}|=W$; see an example in Figure \ref{fig:network-DNN}. Additionally, we denote $\mathcal F_{\mbox{\tiny MLP}}(L,H)$ as multilayer perceptrons consisting of $L$ hidden layers with $H$ units in each hidden layer. In this architecture, the units are connected to all units from neighboring layers and remain unconnected with units from non-neighboring layers; see an example in Figure \ref{fig:network-MLP}.

\subsection{DNNs through nested doubly robust regression}\label{sec:imp-DNN-theory}

Let $\S^1:=(\bX_i,R_i,Y_i^\#)_{i=1}^n$ be independent and identically distributed (i.i.d.) samples, where $\bX_i\in\R^p$ is the covariate vector, $R_i\in\{0,1\}$ is a sub-group indicator, and $Y_i^\#\in\R$ denotes the outcome variable of interest. The indicator $R_i$ is introduced to allow the established results to be easily applied to problems involving regression within a specific subgroup. If the target population is the entire population, one can simply set $R_i \equiv 1$. 
Consider situations where $Y_i^\#$ is possibly non-observable and let $\Yhat_i=\hhat(\bX_i;\S^2)$ be an estimate of $Y_i^\#$, where $\S^2\independent\S^1$ is another set of random samples. Let $(\bX,R,Y^\#)$ be an independent copy of $\S^1$ and $\Yhat=\hhat(\bX;\S^2)$. Define the nested DNN estimate $\fhat$ as:
\begin{align}\label{def:fhat}
\fhat\in\arg\min_{f\in\mathcal F_\mathrm{DNN}:\|f\|_\infty\leq2M}n^{-1}\sum_{i=1}^nR_i\left\{\Yhat_i-f(\bX_i)\right\}^2.
\end{align}
This estimate is for $f^0(\bx):=E(Y^\#\mid\bX=\bx,R=1)$ based on the user-chosen architecture $\mathcal F_\mathrm{DNN}$. Here, $M$ is an arbitrarily large constant independent of the sample size $n$. The bound $\|f\|_\infty\leq2M$ enforces uniform boundedness, constituting the weakest form of constraint; see, for example, \cite{farrell2021deep, schmidt2020nonparametric, fan2020theoretical}. Define the approximation error $\epsilon_n$ as:
\begin{align}\label{def:eps}
\epsilon_n:=\inf_{f\in\mathcal F_\mathrm{DNN}:\|f\|_\infty\leq2M}\|f-f^0\|_\infty.
\end{align}
Now, we present the following theorem characterizing the convergence results of the nested DNN $\fhat$.

\begin{theorem}\label{thm:imp-DNN}
Let $\Yhat-Y^\#=\Delta_1+\Delta_2$ with some $\Delta_1=\Delta_1(\bZ;\S^2)$ and $\Delta_2=\Delta_2(\bZ;\S^2)$ satisfying
\begin{align}\label{cond:Deltas}
E_\bZ(\Delta_1\mid\bX,R=1)=0\;\;\text{almost surely}\;\;\text{and}\;\;E_\bZ(\Delta_2^2\mid R=1)=O_p(e_n^2).
\end{align}
Assume $\bX\subseteq[-1,1]^p$, $\|RY^\#\|_\infty\leq M$, $\|R\Delta_1\|_{\infty,P_\bZ}=O_p(1)$. Let $n\gg WL\log W$, where $L$ and $W$ denote the numbers of hidden layers and weights, respectively. Then, as $n\to\infty$,
\begin{align}\label{rate:f-gen}
E_\bX\left[R\{\fhat(\bX)-f^0(\bX)\}^2\right]=O_p\left(\frac{WL\log W\log n}{n}+e_n^2+\epsilon_n^2\right).
\end{align}
\end{theorem}

\begin{remark}[Nested doubly robust regression and generalized Neyman orthogonality]\label{remark:DR-DNN}
In the following, we explain how Theorem \ref{thm:imp-DNN} applies to doubly robust regression problems based on representations of the forms \eqref{rep:CE}-\eqref{def:CNO}. Let $R\equiv1$, $Y^\#=\psi(\bZ;\eta^0)$, and $\Yhat=\psi(\bZ;\etahat)$, where the nuisance estimates $\etahat=\etahat(\S^2)$ are obtained from the training samples $\S^2$. Using Taylor's theorem, we can decompose the first-stage error as $\Yhat-Y^\#=\psi(\bZ;\etahat)-\psi(\bZ;\eta^0)=\Delta_1+\Delta_2$. Here, $\Delta_1=\partial_\eta\psi(\bZ;\eta^0)(\etahat-\eta^0)$ represents the linear approximation of the first-stage error, and $\Delta_2$ is a remainder term. When the score function $\psi$ satisfies generalized Neyman orthogonality \eqref{def:CNO}, the linear approximation has a zero conditional mean, i.e., $E_\bZ(\Delta_1\mid\bX)=0$, where the expectation is taken only over a new observation $\bZ$ independent of $\S^2$.

The nested DNN \eqref{def:fhat} yields a convergence rate \eqref{rate:f-gen} while estimating the conditional mean function $f^0(\bx)=\theta(\bx)=E\{\psi(\bZ;\eta^0)\mid\bX=\bx\}$. It is essential to note that the convergence rate \eqref{rate:f-gen} does not depend on the linear approximation part of the first-stage error, $\Delta_1$; it only involves the remainder term $\Delta_2$ (characterized by $e_n^2$), which is potentially small.
\end{remark}


Theorem \ref{thm:imp-DNN} provides general estimation results for nested DNNs, where the network is based on an arbitrary user-chosen architecture $\mathcal F_\mathrm{DNN}$. The final estimation rate \eqref{rate:f-gen} involves the approximation error $\epsilon_n$ based on the chosen architecture $\mathcal F_\mathrm{DNN}$, which necessitates further study and control. For smooth functions, Theorem 1 of \cite{yarotsky2017error} offers an upper bound for the approximation error over the class $\mathcal F_\mathrm{DNN}(L,U,W)$ when $L, U, W$ are sufficiently large. For any $\beta, q \geq 0$, we define the $q$-dimensional H\"{o}lder ball with smoothness $\beta$ as:
\begin{align}\label{def:Holder}
\mathcal W^{\beta,\infty}([-1,1]^q):=\left\{f:\max_{|\balpha|\leq\beta}|D^{\balpha}f(\bx)|\leq1,\forall\bx\in[-1,1]^q\right\},
\end{align}
where $\balpha=(\alpha_1,\dots,\alpha_q)$, $|\balpha|=\sum_{j=1}^q\alpha_j$, and $D^{\balpha}f$ is the weak derivative.

Consider the following ``optimal'' network:
\begin{align}\label{def:fhat-opt}
\fhat_\mathrm{opt}\in\arg\min_{f\in\mathcal F_\mathrm{DNN}(L,U,W):\|f\|_\infty\leq2M}n^{-1}\sum_{i=1}^n\left\{\Yhat_i-f(\bX_i)\right\}^2,
\end{align}
which is the nested DNN estimate based on the optimal architecture over a collection of architectures with a given triple $(L, U, W)$. The following theorem characterizes the convergence rate of $\fhat_\mathrm{opt}$ with optimal choices of parameters $(L, U, W)$.

\begin{theorem}\label{thm:imp-opt}
Let the assumptions in Theorem \ref{thm:imp-DNN} hold. Assume that $f^0(\bx)$ can be approximated by a sparse smooth function $f_Q^0(\bx_Q):\R^q\mapsto\R$ that
\begin{align}\label{bound:r_n}
\|f^0(\bX)-f_Q^0(\bX_Q)\|_\infty\leq r_n,
\end{align}
where $\bx_Q:=(\bx_j)_{j\in Q}$, $\bX_Q:=(\bX_j)_{j\in Q}$, $|Q|=q\geq1$ is a constant independent of $n$, $r_N=o(1)$ is a non-negative sequence and $f_Q^0\in\mathcal W^{\beta,\infty}([-1,1]^{q})$ with some constant smoothness parameter $\beta\geq0$. Set $\bar\epsilon_n:=n^{-\beta/(2\beta+q)}\log^{4\beta/(2\beta+q)}n$, let $L$, $U$, and $W$ satisfy $L\geq c\{\log(1/\bar\epsilon_n)+1\}$, $U,W\geq c\bar\epsilon_n^{-q/\beta}\{\log(1/\bar\epsilon_n)+1\}$ with $L\asymp\log n$ and $U\asymp W\asymp \bar\epsilon_n^{-q/\beta}\log(1/\bar\epsilon_n)\asymp n^{q/(2\beta+q)}\log^{1-4q/(2\beta+q)}n$, where $c>0$ is a constant. Then, as $n\to\infty$,
\begin{align}\label{rate:f-opt}
E\left[R\{\fhat_\mathrm{opt}(\bX)-f^0(\bX)\}^2\right]=O_p\left(n^{-\frac{2\beta}{2\beta+q}}\log^\frac{8\beta}{2\beta+q}n+r_n^2+e_n^2\right).
\end{align}
\end{theorem}

Consider a degenerate single-stage learning where $\Yhat_i=Y_i^\#$ and $e_n=0$. Moreover, assume that the function $f^0$ is precisely sparse and smooth with $r_0=0$. In this case, the convergence rate \eqref{rate:f-opt} solely involves $n^{-2\beta/(2\beta+q)}\log^{8\beta/(2\beta+q)}n$. This rate is nearly optimal for non-parametric estimation when the ``active'' covariate set $Q$ is known. Notably, the conditions and outcomes in Theorem \ref{thm:imp-opt} depend solely on the sparsity level $q$ and are unrelated to the original dimension $p$. This is because the class of DNNs $\mathcal F_\mathrm{DNN}(L,U,W)$ optimized in \eqref{def:fhat-opt} encompasses ``oracle'' structures with redundant input units unconnected to any subsequent layers.

Although the statistical error of \(\fhat_\mathrm{opt}\) is nearly minimax optimal, determining the optimal sparse architecture from the collection \(\mathcal F_\mathrm{DNN}(L,U,W)\) is computationally intractable and requires extensive tuning of architectures, a problem that is NP-hard. Pruning methods are sometimes used to ensure network sparsity in practice, but they lack guarantees for reaching the optimal architecture within an acceptable timeframe and are criticized for high computational costs; see, e.g., \cite{liu2018rethinking,evci2019difficulty}. Therefore, we consider a simpler yet widely used class of functions, the multilayer perceptrons \(\mathcal F_{\mbox{\tiny MLP}}(L,H)\), which consists of only one architecture for a given pair \((L,H)\); see an illustration in Figure \ref{fig:network-MLP}. Let \(\fhat_{\mbox{\tiny MLP}}\) be the nested DNN estimate based on multilayer perceptrons:

\begin{align}\label{def:fhat-MLP}
\fhat_{\mbox{\tiny MLP}}\in\arg\min_{f\in\mathcal F_{\mbox{\tiny MLP}}(L,H):\|f\|_\infty\leq2M}n^{-1}\sum_{i=1}^nR_i\left\{\Yhat_i-f(\bX_i)\right\}^2.
\end{align}

The following theorem provides the convergence rate for \(\fhat_{\mbox{\tiny MLP}}\).

\begin{theorem}\label{thm:imp-MLP}
Let the assumptions in Theorem \ref{thm:imp-DNN} hold. Assume that $f^0(\bx)$ can be approximated by a smooth function $f_Q^0(\bx_Q):\R^q\mapsto\R$ with \eqref{bound:r_n} holds and $f_Q^0\in\mathcal W^{\beta,\infty}([-1,1]^{q})$ with some constants $q\geq1$ and $\beta\geq0$. 

Case (a): $p=O(n^{q/(2\beta+2q)}\log^{(3\beta-q)/(\beta+q)}n)$. Set $\tilde\epsilon_n:=n^{-\beta/(2\beta+2q)}\log^{4\beta/(\beta+q)}n$. Let $L$ and $H$ satisfy $L\geq c\{\log(1/\tilde\epsilon_n)+1\}$, $H\geq \lceil c\tilde\epsilon_n^{-q/\beta}\{\log(1/\tilde\epsilon_n)+1\}\rceil(L+1)$ with $L\asymp\log n$ and $H\asymp n^{q/(2\beta+2q)}\log^{(2\beta-2q)/(\beta+q)}n$, where $c>0$ is a constant. Then, as $n\to\infty$,
\begin{align}\label{rate:f-MLP-a}
E\left[R\{\fhat_{\mbox{\tiny MLP}}(\bX)-f^0(\bX)\}^2\right]=O_p\left(n^{-\frac{\beta}{\beta+q}}\log^\frac{8\beta}{\beta+q}n+r_n^2+e_n^2\right).
\end{align}

Case (b): $p\gg n^{q/(2\beta+2q)}\log^{(3\beta-q)/(\beta+q)}n$ and $p=o(n/\log^5n)$. Set $\tilde\epsilon_n:=(p/n)^{\beta/(2\beta+q)}\cdot\log^{5\beta/(2\beta+q)}n$. Let $L$ and $H$ satisfy $L\geq c\{\log(1/\tilde\epsilon_n)+1\}$, $H\geq \lceil c\tilde\epsilon_n^{-q/\beta}\{\log(1/\tilde\epsilon_n)+1\}\rceil(L+1)$ with $L\asymp\log n$ and $H\asymp(n/p)^{q/(2\beta+q)}\log^{(4\beta-3q)/(2\beta+q)}n$, where $c>0$ is a constant. Then, as $n\to\infty$,
\begin{align}\label{rate:f-MLP-b}
E\left[R\{\fhat_{\mbox{\tiny MLP}}(\bX)-f^0(\bX)\}^2\right]=O_p\left((p/n)^{\frac{2\beta}{2\beta+q}}\log^\frac{10\beta}{2\beta+q}n+r_n^2+e_n^2\right).
\end{align}
\end{theorem}

\begin{remark}[Convergence rate]\label{remark:rate-MLP}
The convergence rates of multilayer perceptrons, \eqref{rate:f-MLP-a} or \eqref{rate:f-MLP-b}, are slower than DNNs based on optimal structures, \eqref{rate:f-opt}. This discrepancy arises because the architectures of multilayer perceptrons include unnecessary edges that do not contribute to approximation power but increase estimation error due to added structural complexity. However, the multilayer perceptrons are much more computationally attractive, as the class $\mathcal{F}_{\mbox{\tiny MLP}}(L,H)$ contains only one network structure for any fixed pair of hyperparameters $(L,H)$. Conversely, finding the optimal structure over a class $\mathcal{F}_\mathrm{DNN}(L,U,W)$ with a given triple $(L,U,W)$, or over a collection of sparse networks \citep{fan2020theoretical,schmidt2020nonparametric}, is an NP-hard problem and computationally infeasible, especially for large-scale problems.

As demonstrated in Theorem \ref{thm:imp-MLP}, the convergence rate of nested multilayer perceptrons consists of three components: (1) the statistical error, which takes different analytical forms in \eqref{rate:f-MLP-a} and \eqref{rate:f-MLP-b} depending on the covariates' dimension, (2) the sparse-approximation error $r_n^2$, and (3) the first-stage error $e_n^2$. Under Case (a) above, the statistical error is $n^{-\beta/(\beta+q)}\log^{8\beta/(\beta+q)}n$, which is faster than the usual minimax non-parametric estimation rate $n^{-2\beta/(2\beta+p)}$ as long as $q<p/2$. In existing work, \cite{farrell2021deep} considered the degenerate single-stage learning problem, and their Theorem 1 provided a convergence rate $O_p(n^{-\beta/(\beta+p)}\log^8n)$ for multilayer perceptrons when $p$ is fixed. Comparing with their work, apart from the nested structure we have addressed and a slight improvement on the logarithmic terms, we provide a faster convergence rate as long as $q<p$. 

It is essential to highlight that conducting non-parametric regressions in high dimensions poses a significant challenge due to the ``curse of dimensionality.'' Nonetheless, we demonstrate the consistency of multilayer perceptrons under growing dimension $p\to\infty$ as $n\to\infty$, provided that $p=o(n/\log^5n)$. Our findings offer theoretical insights into the practical success of DNNs in high dimensions. To the best of our knowledge, this is the first result providing theoretical guarantees for multilayer perceptrons allowing for a diverging dimension without assuming any sparse structure on the marginal distribution $P_{\bX}$, even in single-stage learning problems. Instead, we consider sparse structures on the conditional distribution $P_{Y\mid\bX}$, which are more suitable for the causal inference problems addressed in this work. Moreover, our results do not rely on any type of beta-min conditions (i.e., the minimum non-zero signal is large enough), which are essential for achieving variable selection consistency and are often violated in practice.
\end{remark}

We would like to emphasize that the result in Theorem \ref{thm:imp-MLP} is also of independent interest in that it provides a user-friendly general theory towards the convergence rate of multilayer perceptrons based on nested structures. This result is also useful while establishing theoretical guarantees for the usage of DNNs in other contexts where sequential regression are needed, such as reinforcement learning and the estimation of optimal dynamic treatment regimes. Additionally, although we only considered the square loss above, the established results can also be extended to other loss functions, such as the logistic loss.


\section{Heterogenous treatment effect estimation using DNNs}\label{sec:CATE}

In this section, we employ DNNs to estimate the conditional average treatment effect (CATE) introduced in Example \ref{ex:CATE}. 

\subsection{The doubly robust DNN estimate of CATE}
Let $\S=(\bZ_i)_{i=1}^N=(\bS_i,T_i,Y_i)_{i=1}^N$ be i.i.d. samples, and let $\bZ$ be an independent copy of $\S$ with $\bS\in\R^d$, $T\in\{0,1\}$, and $Y\in\R$ representing the covariate vector, treatment indicator, and the observed outcome, respectively. Using the potential outcome framework, let $Y(t)$ be the potential outcome an individual would have received when exposed to treatment $t\in\{0,1\}$. Our goal is to estimate the CATE function $\theta_{\mbox{\tiny CATE}}(\bs):=E\{Y(1)-Y(0)\mid\bS=\bs\}$ for any given $\bs\in\R^d$. Estimating the CATE is valuable in the presence of heterogeneity and is crucial for the development of precision medicine.

Denote $\pi^0(\bs):=P(T=1\mid\bS=\bs)$ as the propensity score function and $\mu^0(t,\bs):=E(Y\mid\bS=\bs,T=j)$ as the outcome regression function for each treatment group $t\in\{0,1\}$. Assume the standard identification conditions \citep{imbens2015causal,tsiatis2006semiparametric}. 

\begin{assumption}[Identification conditions]\label{cond:ign} 
Let the following conditions hold: (a) Ignorability: $\{Y(1),Y(0)\}\independent T\mid\bS$; (b) Consistency: $Y=Y(T)$; (c) Positivity/Overlap: $\|1/\pi^0\|_\infty,\|1/(1-\pi^0)\|_\infty\leq M$ with some constant $M>0$.
\end{assumption}

\begin{algorithm} [h!] \caption{The doubly robust DNN estimator for the CATE}\label{alg:CATE}
\begin{algorithmic}[1]
\Require Observations $\S=(\bZ_i)_{i=1}^N=(\bS_i,T_i,Y_i)_{i=1}^N$ and a pair $(L,H)$.
\State Split $\S$ into two equal-sized parts $(\S_1,\S_2)$, indexed by $\mathcal I_1$ and $\mathcal I_2$, respectively. Let $n=N/2$ be an integer for the sake of simplicity.
\State Construct nuisance estimates $\pihat^{(2)}(\bs)=\pihat(\bs,\S_2)$ and $\muhat^{(2)}(t,\bs)=\muhat(t,\bs,\S_2)$ for each $t\in\{0,1\}$ based on training samples $\S_2$ using arbitrary machine learning methods.
\State For any $i\in\mathcal I_1$, define the doubly robust outcomes
\begin{equation}\label{def:DR-imp-CATE}
\Yhat_i=\muhat^{(2)}(1,\bS_i)+\frac{T_i\{Y_i-\muhat^{(2)}(1,\bS_i)\}}{\pihat^{(2)}(\bS_i)}-\muhat^{(2)}(0,\bS_i)-\frac{(1-T_i)\{Y_i-\muhat^{(2)}(0,\bS_i)\}}{1-\pihat^{(2)}(\bS_i)}.
\end{equation}
\State Construct a DNN estimate of the CATE based on training samples $(\bX_i,\Yhat_i)_{i\in\mathcal I_1}$:
$$\thetahat_{\mbox{\tiny CATE}}^1\in\arg\min_{f\in\mathcal F_{\mbox{\tiny MLP}}(L,H):\|f\|_\infty\leq2M}n^{-1}\sum_{i\in\mathcal I_1}\left\{\Yhat_i-f(\bX_i)\right\}^2.$$
\State Exchange subsamples $\S_1$ and $\S_2$, repeat Steps 2-4 above and obtain another estimate $\thetahat_{\mbox{\tiny CATE}}^2$.\\
\Return The doubly robust DNN estimator for the CATE, $\thetahat_{\mbox{\tiny CATE}}=(\thetahat_{\mbox{\tiny CATE}}^1+\thetahat_{\mbox{\tiny CATE}}^2)/2$.
\end{algorithmic}
\end{algorithm}

Under Assumption \ref{cond:ign}, the CATE can be identified through the doubly robust representation \eqref{rep:CATE}. Let $\pihat$ and $\muhat$ be estimates of the nuisance functions $\pi^0$ and $\mu$, respectively. Here, we allow flexible nuisance estimation methods in the first learning stage, including, for example, DNNs and Lasso estimators. Then, we employ the DNN in the second learning stage to estimate the CATE, using the outcome variables \eqref{def:DR-imp-CATE} constructed based on the first-stage estimates and a doubly robust score function. The detailed construction based on a cross-fitting technique is introduced in Algorithm \ref{alg:CATE}.

\subsection{Theoretical properties}\label{sec:theory-CATE}

We assume the following conditions.

\begin{assumption}[Estimation errors]\label{cond:converge-rate-CATE}
Let $\delta_N\geq0$ be a sequence satisfying
\begin{align}
\max_{t\in\{0,1\}}E_\bZ\left[\mathbbm1_{T=t}\left\{\muhat(t,\bS)-\mu^0(t,\bS)\right\}^2\left\{\pihat(\bS)-\pi^0(\bS)\right\}^2\right]&=O_p(\delta_N^4).\label{def:deltaN-CATE}
\end{align}
\end{assumption}

\begin{assumption}[Boundedness]\label{cond:bound-CATE}
Let $\bS\subseteq[-1,1]^{d_1}$ and $\|Y\|_\infty\leq M$. Additionally, let $\|\muhat\|_{\infty,P_\bZ}$, $\|1/\pihat\|_{\infty,P_\bZ}$, $\|1/(1-\pihat)\|_{\infty,P_\bZ}=O_p(1)$.
\end{assumption}

\begin{assumption}[Smoothness]\label{cond:smooth-CATE}
Assume that $\theta_{\mbox{\tiny CATE}}(\bs)$ can be approximated by a smooth function $\theta_Q(\bs_{Q}):\R^{q_\theta}\mapsto\R$ that $\|\theta_{\mbox{\tiny CATE}}(\bS)-\theta_Q(\bS_{Q})\|_\infty\leq r_N$, where $\bs_{Q}:=(\bs_{j})_{j\in Q}$, $\bS_{Q}:=(\bS_{j})_{j\in Q}$, $|Q|=q_\theta\geq1$ is a constant, $r_N=o(1)$, and $\mu_Q$ lies in the H\"{o}lder ball with some constant smoothness parameter $\beta_\theta\geq0$ that $\mu_Q\in\mathcal W^{\beta_\theta,\infty}([-1,1]^{q_\theta})$, \eqref{def:Holder}.
\end{assumption}

Assumption \ref{cond:converge-rate-CATE} assumes a product rate condition on the estimation errors of nuisance functions $\mu^0$ and $\pi^0$, with the product rate $\delta_N$ appearing in the estimation error of the CATE in Theorem \ref{thm:theta-reg}. The boundedness conditions for $\bS$ and $Y$ in Assumption \ref{cond:bound-CATE} are relatively standard in the non-parametric regression literature; see, for instance, \cite{farrell2021deep}. The uniform lower bounds on the propensity score estimates $\pihat$ and $1-\pihat$ are also commonly seen in the double machine learning literature, such as \cite{chernozhukov2017double}. These conditions, together with the uniform upper bound for $\muhat$, are easily satisfied, for example, by the DNNs as in Example \ref{ex:DNN-CATE} below. Assumption \ref{cond:smooth-CATE} is a generalization of the usual H\"{o}lder smoothness condition; we only require $\mu_Q$, a sparse approximation of the CATE, to lie in the H\"{o}lder ball. A related but different condition has been considered by \cite{fan2020theoretical,schmidt2020nonparametric}, where the true conditional mean function is assumed to be a composition of exactly sparse smooth functions. While we require the sparsity level $q > 0$ to be a constant independent of $N$ for non-parametric regression, the total dimension $d$ is allowed to grow with the sample size.

In the following lemma, we first characterize the deviation between $\Yhat$ and $Y^\#$, which is the error originated from the first learning stage.

\begin{lemma}\label{lemma:imput-err-CATE}
Let Assumptions \ref{cond:ign}, \ref{cond:converge-rate-CATE}, and \ref{cond:bound-CATE} hold. Let $\Yhat$ be an independent copy of \eqref{def:DR-imp-CATE} and $Y^\#$ is defined as in \eqref{rep:CATE}. Then, 
\begin{align*}
\Yhat-Y^\#=\Delta_1+\Delta_2,
\end{align*} 
with some $\Delta_1,\Delta_2\in\R$ satisfying
\begin{align}
E_\bZ(\Delta_1\mid\bS)=0\;\;\text{almost surely}\;\;\text{and}\;\;E_\bZ(\Delta_2^2)=O_p(\delta_N^4).\label{result:lemma-CATE}
\end{align}
\end{lemma}

Lemma \ref{lemma:imput-err-CATE} establishes a connection between the CATE estimation problem and the nested DNN results in Section \ref{sec:imp-DNN-theory}, enabling the application of Theorem \ref{thm:imp-MLP}. According to Lemma \ref{lemma:imput-err-CATE}, the first-stage error $\Yhat-Y^\#$ comprises two components: (a) a first-order estimation error term $\Delta_1$, defined as the summation of \eqref{def:Delta11}-\eqref{def:Delta14} in the Supplementary Material, depending \emph{linearly} on the first-stage estimation errors; (b) a second-order term $\Delta_2$, defined as the summation of \eqref{def:Delta21}-\eqref{def:Delta22}, depending \emph{quadratically} on the first-stage estimation errors. Benefiting from the doubly robust score \eqref{rep:CATE} that satisfies the generalized Neyman orthogonality condition \eqref{def:CNO}, the first-order error term has a conditional mean zero, as shown in \eqref{result:lemma-CATE}. Hence, as demonstrated in Theorem \ref{thm:imp-MLP}, the first-order error does not affect the convergence rate of the doubly robust nested DNN estimation, and only the smaller second-order error contributes to the final convergence rate. These results are summarized in the theorem below.

\begin{theorem}\label{thm:theta-reg}
Let Assumptions \ref{cond:ign}, \ref{cond:converge-rate-CATE}, \ref{cond:bound-CATE} and \ref{cond:smooth-CATE} hold, $d=o(N/\log^5N)$. Choose $L$ and $H$ as in Theorem \ref{thm:imp-MLP} with $p=d$, $q=q_\theta$, and $\beta=\beta_\theta$. Then, as $N\to\infty$,
\begin{align}
&E_\bZ\left[\left\{\thetahat_{\mbox{\tiny CATE}}(\bS)-\theta_{\mbox{\tiny CATE}}(\bS)\right\}^2\right]=O_p\left(w_{N,d}(q_\theta,\beta_\theta)+r_N^2+\delta_N^4\right),\;\;\mbox{where}\label{rate:thetahat}\\
&w_{N,p}(q,\beta):=N^{-\frac{2\beta}{2\beta+2q}}\log^\frac{8\beta}{\beta+q}N+(p/N)^{\frac{2\beta}{2\beta+q}}\log^\frac{10\beta}{2\beta+q}N\;\;\mbox{for any}\;\;p,q,\beta\geq0.\label{def:w_Nd}
\end{align}
\end{theorem}

\begin{remark}[Consistency and robustness]\label{remark:CATE}
Recently, \cite{foster2023orthogonal} proposed a general framework for CATE estimation and briefly explored the utilization of DNNs as an illustrative example. However, their approach necessitates the target function $\theta(\bx)$ to precisely align with the class of DNNs, therefore reducing the problem to a parametric one, as the approximation error (or misspecification error) is assumed to be exactly zero. Moreover, they presented a slower convergence rate involving an additional term $\max_{t\in\{0,1\}}E_\bZ[\mathbbm1_{T=t}\{\muhat(t,\bS)-\mu^0(t,\bS)\}^4]+E_\bZ[\{\pihat(\bS)-\pi^0(\bS)\}^4]$, compared to our established rate \eqref{rate:thetahat}. Consequently, besides the unrealistic assumption of zero approximation error, their method requires consistent estimation of both nuisance models $\mu^0$ and $\pi^0$ to attain a consistent estimate for the CATE. In contrast, we only necessitate that the true CATE function can be approximated by a sparse smooth function with $r_N=o(1)$, and the product of the first-stage error satisfies $\delta_N=o(1)$. The latter condition holds as long as either $\mu^0$ or $\pi^0$ is consistently estimated, not necessarily both -- our results provide enhanced model robustness. To our knowledge, Theorem \ref{thm:theta-reg} marks the first consistent result for CATE that accommodates both a fully non-parametric model and a growing number of dimension, without relying on any variable selection techniques.
\end{remark}

\subsection{Examples}\label{sec:examples-CATE}

In the following, we provide specific examples for the nuisance estimates $(\muhat,\pihat)$ and discuss the resulting convergence rates for CATE estimation. For the sake of simplicity, let $\theta_{\mbox{\tiny CATE}}$ be exactly sparse, i.e., $r_N$ defined in Assumption \ref{cond:smooth-CATE} is exactly zero.

\begin{example}[First-stage learning through DNNs]\label{ex:DNN-CATE}
Let both $\mu^0(0,\cdot)$ and $\mu^0(1,\cdot)$ be exactly sparse functions with a sparsity level of $q_\mu$ and H\"older smoothness level of $\beta_\mu$. Additionally, let $\pi^0$ be an exactly sparse function with a sparsity level of $q_\pi$ and H\"older smoothness level of $\beta_\pi$. Consider DNN estimates: for each $k\in\{1,2\}$,
\begin{align*}
\muhat^{(k)}(t,\cdot)&\in\arg\min_{f\in\mathcal F_{\mbox{\tiny MLP}}(L_{\mu,t},H_{\mu,t}):\|f\|_\infty\leq2M}n^{-1}\sum_{i\in\mathcal I^k}\mathbbm1_{T_i=t}\left\{Y_i-f(\bS_i)\right\}^2\;\;\forall t\in\{0,1\},\\
\fhat_\pi^{(k)}(\cdot)&\in\arg\min_{f\in\mathcal F_{\mbox{\tiny MLP}}(L_{\pi},H_{\pi}):\|f\|_\infty\leq2M}n^{-1}\sum_{i\in\mathcal I^k}\left(-T_if(\bS_i)+\log[1+\exp\{f(\bS_i)\}]\right),
\end{align*}
and $\pihat^{(k)}=\phi(\fhat_\pi^{(k)})$, with $\phi(u)=\exp(u)/\{1+\exp(u)\}$ for any $u\in\R$ denoting the logistic function. By Theorem \ref{thm:imp-MLP}, when the tuning parameters are chosen appropriately,
\begin{align*}
\max_{t\in\{0,1\}}E\left[\mathbbm1_{T=t}\{\muhat^{(k)}(t,\bS)-\mu^0(t,\bS)\}^2\right]=O_p(w_{N,d}(q_\mu,\beta_\mu)).
\end{align*}
Additionally, repeating the proof of Theorem \ref{thm:imp-MLP} using the logistic loss, we also have
\begin{align*}
E\left[\{\pihat^{(k)}(\bS)-\pi^0(\bS)\}^2\right]=O_p(w_{N,d}(q_\pi,\beta_\pi)).
\end{align*}
By construction, $\|\muhat^{(k)}-\mu^0\|_\infty,\|\pihat^{(k)}-\pi^0\|_\infty=O(1)$. Therefore, the product term $\delta_N$, \eqref{def:deltaN-CATE}, can be chosen as $\delta_N^4=O(w_{N,d}(q_\mu,\beta_\mu)\land w_{N,d}(q_\pi,\beta_\pi))$, leading to the following result for the CATE estimation:
\begin{equation}\label{rate:CATE-DNN}
E_\bZ\left[\left\{\thetahat_{\mbox{\tiny CATE}}(\bS)-\theta_{\mbox{\tiny CATE}}(\bS)\right\}^2\right]=O_p\left(w_{N,d}(q_\theta,\beta_\theta)+w_{N,d}(q_\mu,\beta_\mu)\land w_{N,d}(q_\pi,\beta_\pi)\right).
\end{equation}
Therefore, even if both outcome regression models $\mu^0(0,\cdot)$ and $\mu^0(1,\cdot)$ are relatively complex (e.g., dense and non-smooth), we can still obtain a consistent estimate for the CATE as long as the CATE itself and the propensity score are sparse and smooth. On the other hand, directly estimating the CATE as the difference between $\muhat(1,\cdot)$ and $\muhat(0,\cdot)$ is generally inconsistent. This observation underscores the enhanced robustness and accuracy achieved through the use of doubly robust regression in the second learning stage.
\end{example}

\begin{example}[First-stage learning through Lasso]
Consider Lasso and logistic Lasso estimates. For any $\bs\in\R^d$, $k\in\{1,2\}$, and $t\in\{0,1\}$, let $\muhat^{(k)}(t,\bs)=\bs^\top\bbetahat_t^{(k)}$ and $\pihat^{(k)}(\bs)=\phi(\bs^\top\bbetahat_\pi^{(k)})$, where
\begin{align*}
\bbetahat_t^{(k)}&\in\arg\min_{\bbeta\in\R^d}n^{-1}\sum_{i\in\mathcal I^k}\mathbbm1_{T_i=t}\left(Y_i-\bS_i^\top\bbeta\right)^2+\lambda_t\|\bbeta\|_1\;\;\forall t\in\{0,1\},\\
\bbetahat_\pi^{(k)}&\in\arg\min_{\bbeta\in\R^d}n^{-1}\sum_{i\in\mathcal I^k}\left[-T_i\bS_i^\top\bbeta+\log\{1+\exp(\bS_i^\top\bbeta)\}\right]+\lambda_\pi\|\bbeta\|_1,
\end{align*}
with tuning parameters $\lambda_t,\lambda_\pi\geq0$. Let $\bbeta_0^*$, $\bbeta_1^*$, and $\bbeta_\pi^*$ be the best population linear/logistic slopes, and suppose that $\|\bbeta_t^*\|_0\leq s_\mu$ for each $t\in\{0,1\}$ and $\|\bbeta_\pi^*\|_0\leq s_\pi$, where $\|\bbeta\|_0=|\{j\leq d:\bbeta_j\neq0\}|$ for any $\bbeta\in\R^d$. Additionally, assume that the linear and logistic approximation errors satisfy $E[\mathbbm1_{T=t}\{\mu^0(t,\bS)-\bS^\top\bbeta_t^*\}^2]\leq e_\mu^2$ and $E[\{\pi^0(\bS)-\phi(\bS^\top\bbeta_\pi^*)\}^2]\leq e_\pi^2$. Then, by standard high-dimensional statistics theory \citep{wainwright2019high}, we have
\begin{align*}
\max_{t\in\{0,1\}}E\left[\mathbbm1_{T=t}\{\muhat^{(k)}(t,\bS)-\mu^0(t,\bS)\}^2\right]&=O_p\left(\frac{s_\mu\log d}{N}+e_\mu^2\right),\\
E\left[\{\pihat^{(k)}(\bS)-\pi^0(\bS)\}^2\right]&=O_p\left(\frac{s_\pi\log d}{N}+e_\pi^2\right).
\end{align*}
For sub-Gaussian covariates, we can choose the product term $\delta_N$, \eqref{def:deltaN-CATE}, as $\delta_N^4=(s_\mu\log d/N+e_\mu^2)(s_\pi\log d/N+e_\pi^2)$. Therefore,
$$E_\bZ\left[\left\{\thetahat_{\mbox{\tiny CATE}}(\bS)-\theta_{\mbox{\tiny CATE}}(\bS)\right\}^2\right]=O_p\left(w_{N,d}(q_\theta,\beta_\theta)+\left(\frac{s_\mu\log d}{N}+e_\mu^2\right)\left(\frac{s_\pi\log d}{N}+e_\pi^2\right)\right).$$
Consistency of the CATE estimate requires either $e_\mu=o(1)$ or $e_\pi=o(1),$ not necessarily both. When one of the nuisance models is correctly parametrized, i.e., $e_\mu=0$ or $e_\pi=0$, and the other one is dense and non-smooth, the convergence rate above is faster than (or at least the same as) the CATE estimator using DNN nuisance estimates, which has a convergence rate \eqref{rate:CATE-DNN}. However, the approach in Example \ref{ex:DNN-CATE} is more robust as it does not require any of the parametric working models to well approximate the true models.
\end{example}

\section{Dynamic treatment effect estimation using DNNs}\label{sec:DTE}

In this section, we consider the estimation of the dynamic treatment effect (DTE) introduced in Example \ref{ex:DTE}.

\subsection{The sequential doubly robust estimate of DTE using DNNs}

\begin{algorithm} [h!] \caption{The doubly robust DNN estimator for $\mu^0$}\label{alg:DR-mu}
\begin{algorithmic}[1]
\Require Observations $\S':=(\bZ_i)_{i\in \mathcal I'}=(\bS_{1i}, T_{1i}, \bS_{2i}, T_{2i}, Y_i)_{i\in \mathcal I'}$, the treatment path of interest $\bt=(t_1,t_2)=(1,1)$, the training indices $\mathcal I'\subseteq\{1,\dots,N\}$, and a pair $(L,H)$.
\State Split $\S'$ into 2 equal-sized parts $(\S_1',\S_2')$, indexed by $(\mathcal I_1',\mathcal I_2')$, with $n:=|\mathcal I'|/2$ assuming to be an integer.
\State Construct $\rhohat^{(2)}$ using the subset of $\S_2'$ with the treatment path $t_1=1$.
\Comment {{\it Propensity for time two}}
\State Construct $\nuhat^{(2)}$ using the subset of $\S_2'$ with the treatment path $\bt=(1,1)$.
\Comment{{\it Outcome for time two}}
\State Construct the doubly robust outcomes for the samples $\S_1'$:
\begin{align}\label{def:Yhat}
\Yhat_i:=\nuhat^{(2)}(\bSbar_{2i})+\frac{T_{2i}\{Y_i(1,1)-\nuhat^{(2)}(\bSbar_{2i})\}}{\rhohat^{(2)}(\bSbar_{2i})}.
\end{align}
\State Construct a DNN estimate of the CATE based on training samples $(\bX_i,\Yhat_i)_{i\in\mathcal I_1'}$:
$$\muhat^1\in\arg\min_{f\in\mathcal F_{\mbox{\tiny MLP}}(L,H):\|f\|_\infty\leq2M}n^{-1}\sum_{i\in\mathcal I_1'}T_{1i}\left\{\Yhat_i-f(\bS_{1i})\right\}^2.$$
\State Exchange subsamples $\S_1'$ and $\S_2'$, repeat Steps 2-4 above and obtain another estimate $\muhat^2$.\\
\Return The doubly robust DNN estimator for $\mu^0$, $\muhat=(\muhat^1+\muhat^2)/2$.
\end{algorithmic}
\end{algorithm}

Let $\S=(\bZ_i)_{i=1}^N=(\bS_{1i},T_{1i},\bS_{2i},T_{2i},Y_i)_{i=1}^N$ be i.i.d. samples, and let $\bZ=(\bS_1,T_1,\bS_2,T_2,Y)$ be an independent copy of $\S$ with $\bS_t\in\R^{d_t}$ and $T_t\in\{0,1\}$ representing the covariate vector and treatment indicator at the $t$-th exposure, respectively, for each $t\in\{1,2\}$. Define $\bSbar_2:=(\bS_1^\top,\bS_2^\top)^\top$ and $\bar d:=d_1+d_2$. Consider the potential outcome framework, and let $Y(\bt)$ denote the potential outcome an individual would have received when exposed to treatment path $\bt=(t_1,t_2)\in\{0,1\}^2$, while the observed outcome is $Y\in\R$. The dynamic treatment effect (DTE) between two treatment paths $\ba,\ba'\in\{0,1\}^2$ is defined as $\theta_{\mbox{\tiny DTE}}:=E\{Y(\ba)-Y(\ba')\}$, which evaluates the average difference in the treatment effect between two paths over the entire population. In the following, we focus on the estimation of $\theta:=E\{Y(1,1)\}$ since the same procedure can be extended to the estimation of the counterfactual mean $E\{Y(\bt)\}$ for any $\bt\in\{0,1\}^2$, and the DTE can be estimated as the difference between two counterfactual means.

We first define the relevant nuisance functions necessary to identify the parameter of interest. At the first exposure, denote $\mu^0(\bs_1):=E\{Y(1,1)\mid\bS_1=\bs_1,T_1=1\}$ and $\pi^0(\bs_1):=P(T_1=1\mid\bS_1=\bs_1)$ for any $\bs\in\R^{d_1}$ as the outcome regression and propensity score functions, respectively. Similarly, at the second exposure, define $\nu^0(\bsbar_2):=E\{Y(1,1)\mid\bSbar_2=\bsbar_2,T_1=T_2=1\}$ and $\rho^0(\bsbar_2):=P(T_2=1\mid\bSbar_2=\bsbar_2,T_1=1)$ for any $\bsbar_2\in\R^{\bar d}$. Assume the following identification conditions, which are standard in the dynamic causal inference literature; see, e.g., \cite{murphy2003optimal,robins2000marginal}.

\begin{assumption}[Identification conditions]\label{cond:seq-ign} 
Let the following conditions hold: (a) Sequential ignorability: $Y(1,1)\independent T_1\mid\bS_1$ and $Y(1,1)\independent T_2\mid(\bSbar_2,T_1=1)$; (b) Consistency: $Y=Y(T_1,T_2)$; (c) Positivity/Overlap: $\|1/\pi^0\|_\infty,\|1/\rho^0\|_\infty\leq M$ with some constant $M>0$.
\end{assumption}

Under Assumption \ref{cond:seq-ign}, $\theta=E\{Y(1,1)\}$ can be identified through the doubly robust representation \eqref{rep:DTE}. Additionally, note that the propensity score functions $\pi^0$ and $\rho^0$ are directly identifiable by observable variables based on their definitions. Furthermore, we can identify the outcome regression function at the second exposure as $\nu^0(\bsbar_2)=E(Y\mid\bSbar_2=\bsbar_2,T_1=T_2=1)$ under Assumption \ref{cond:seq-ign}. Therefore, the nuisance functions $\pi$, $\rho$, and $\nu$ can be directly estimated using the observed samples. Let $\pihat$, $\rhohat$, and $\nuhat$ be the estimates at the first learning stage. Similar to Section \ref{sec:CATE}, we also allow for flexible choices for such nuisance estimates. 

\begin{algorithm} [h!] \caption{The sequential double machine learning estimator for $\theta$}\label{alg:DR-theta}
\begin{algorithmic}[1]
\Require Observations $\S=(\bZ_i)_{i=1}^N=(\bS_{1i}, T_{1i}, \bS_{2i}, T_{2i}, Y_i)_{i=1}^N$ and the treatment path of interest $\bt=(t_1,t_2)=(1,1)$.
\State For any fixed integer $K\geq2$, split the sample $\S$ into $K$ equal-sized parts $(\S_k)_{k=1}^{K}$ randomly, indexed by $I_k$. Define $\S_{-k}:=(Z_i)_{i\in I\setminus I_k}$.
\For {$k \in \{1,\cdots, K\}$}
\State Let $\mathcal{I}$ be a subset of indices of $\mathcal I_{-k}$ with the treatment path $\bt=(1,1)$.
\State Let $\mathcal{I}_1$ be a subset of indices of $\mathcal I_{-k}$ with the treatment path $t_1=1$.
\State Construct $\pihat_{-k}$ using $\mathcal I_{-k}$ samples.
\Comment{{\it Propensity for time one}}
\State Construct $\rhohat_{-k}$ using $ \mathcal{I}_1$ samples. 
\Comment {{\it Propensity for time two}}
\State Construct $\nuhat_{-k}$ using $ \mathcal{I}$ samples.
\Comment{{\it Outcome for time two}}
\State Construct $\muhat_{-k}$ as in Algorithm \ref{alg:DR-mu} using $ \mathcal{I}_1$ samples.
\Comment{{\it Outcome for time one}}
\State Construct $\thetahat^{(k)}=|I_k|^{-1}\sum_{i\in I_k}\psi(\bZ_i;\etahat_{-k})$, where $\etahat_{-k}=(\rhohat_{-k},\nuhat_{-k},\pihat_{-k},\muhat_{-k})$ and
$$\psi(\bZ;\eta):=\mu(\bS_1)+\frac{T_1\{\nu(\bSbar_2)-\mu(\bS_1)\}}{\pi(\bS_1)}+\frac{T_1T_2\{Y-\nu(\bSbar_2)\}}{\pi(\bS_1)\rho(\bSbar_2)}\;\;\mbox{for any}\;\;\eta=(\rho,\nu,\pi,\mu).$$
\EndFor\\
\Return The sequential double machine learning estimator for $\theta$, $\thetahat=K^{-1}\sum_{k=1}^{K}\thetahat^{(k)}$.
\end{algorithmic}
\end{algorithm}

The remaining nuisance function $\mu$ is not directly identifiable by observable variables, as the potential outcome $Y_i(1,1)$ is unobservable among the group $T_{1i}=1$. We consider an additional doubly robust representation for $\mu^0$, \eqref{rep:mu}, and employ the DNN using doubly robust outcomes \eqref{def:Yhat} at the second learning stage. Lastly, we perform another doubly robust estimation based on the representation \eqref{rep:DTE} at the third learning stage. The procedure is implemented using a cross-fitting technique; see details in Algorithms \ref{alg:DR-mu} and \ref{alg:DR-theta}.

\subsection{Theoretical properties}

In this section, we initially examine the theoretical properties of the proposed DNN estimate for $\mu^0$ using the results from Section \ref{sec:imp-DNN}. Subsequently, we present general theories and characterize the asymptotic behavior of the estimator for $\theta$. To establish the desired results, we first assume the following conditions.

\begin{assumption}[Estimation errors]\label{cond:converge-rate}
Let $a_N,c_N,d_N,\delta_N\geq0$ be sequences satisfying
\begin{align}
E_\bZ\left[T_1T_2\left\{\nuhat(\bSbar_2)-\nu^0(\bSbar_2)\right\}^{2}\right]&=O_p(a_N^2),\label{def:aN}\\
E_\bZ\left[\{\pihat(\bS_1)-\pi^0(\bS_1)\}^2\right]&=O_p(c_N^2),\label{def:cN}\\
E_\bZ\left[T_1\left\{\rhohat(\bSbar_2)-\rho^0(\bSbar_2)\right\}^{2}\right]&=O_p(d_N^2),\label{def:dN}\\
E_\bZ\left[T_1T_2\left\{\nuhat(\bSbar_2)-\nu^0(\bSbar_2)\right\}^2\left\{\rhohat(\bSbar_2)-\rho^0(\bSbar_2)\right\}^2\right]&=O_p(\delta_N^4).\label{def:deltaN}
\end{align}
\end{assumption}

\begin{assumption}[Boundedness]\label{cond:bound}
Let $\bS_1\subseteq[-1,1]^{d_1}$ and $\|Y\|_\infty\leq M$. Additionally, let $\|\nuhat\|_{\infty,P_\bZ}=O_p(1)$, $\|1/\rhohat\|_{\infty,P_\bZ}=O_p(1)$ and $\|1/\pihat\|_{\infty,P_\bZ}=O_p(1)$.
\end{assumption}

\begin{assumption}[Smoothness]\label{cond:smooth}
Assume that $\mu^0(\bs_1)$ can be approximated by a smooth function $\mu_Q(\bs_{1,Q}):\R^{q_\mu}\mapsto\R$ that $\|\mu^0(\bS_1)-\mu_Q(\bS_{1,Q})\|_\infty\leq r_N$, where $\bs_{1,Q}:=(\bs_{1j})_{j\in Q}$, $\bS_{1,Q}:=(\bS_{1j})_{j\in Q}$, $|Q|=q_\mu\geq1$ is a constant, $r_N=o(1)$, and $\mu_Q$ lies in the H\"{o}lder ball with some constant smoothness parameter $\beta_\mu\geq0$ that $\mu_Q\in\mathcal W^{\beta_\mu,\infty}([-1,1]^{q_\mu})$, \eqref{def:Holder}.
\end{assumption}

Assumptions \ref{cond:converge-rate}-\ref{cond:smooth} are analogous to Assumptions \ref{cond:converge-rate-CATE}-\ref{cond:smooth-CATE} under the dynamic setup. In Assumption \ref{cond:converge-rate}, in addition to the product rate condition \eqref{def:deltaN}, which has an analogous version in Assumption \ref{cond:converge-rate-CATE}, we also denote the convergence rates of the nuisance estimates $\nuhat$, $\pihat$, and $\rhohat$ as $a_N$, $c_N$, and $d_N$, respectively. Note that we do not require $a_N, c_N, d_N = o(1)$; that is, the nuisance models are possibly misspecified.

The following lemma characterizes the error from the doubly robust outcomes when estimating $\mu^0$ in the second learning stage.
\begin{lemma}\label{lemma:imput-err}
Let Assumptions \ref{cond:seq-ign}, \ref{cond:converge-rate}, and \ref{cond:bound} hold. Let $\Yhat$ be an independent copy of \eqref{def:Yhat} and $Y^\#$ is defined as in \eqref{rep:mu}. Then,
\begin{align*}
\Yhat-Y^\#=\Delta_1+\Delta_2,
\end{align*} 
with some $\Delta_1,\Delta_2\in\R$ satisfying
\begin{align*}
E_\bZ(\Delta_1\mid\bS_1,T_1=1)=0\;\;\text{almost surely}\;\;\text{and}\;\;E_\bZ(\Delta_2^2\mid T_1=1)=O_p(\delta_N^4).
\end{align*}
\end{lemma}


Similar to Lemma \ref{lemma:imput-err-CATE}, the first-order estimation error $\Delta_1$ has a conditional mean of zero and therefore does not contribute to the convergence rate of $\muhat$; only the smaller second-order error $\Delta_2$ contributes. These results are characterized in the theorem below, leveraging the general nested DNN results in Theorem \ref{thm:imp-MLP}.

\begin{theorem}\label{cor:mu-DR}
Let Assumptions \ref{cond:seq-ign}, \ref{cond:converge-rate}, \ref{cond:bound}, and \ref{cond:smooth} hold, $d_1=o(n/\log^5n)$. Choose $L$ and $H$ as in Theorem \ref{thm:imp-MLP} with $p=d_1$, $q=q_\mu$, and $\beta=\beta_\mu$. Then, as $N\to\infty$,
\begin{align}
E_\bZ\left[T_1\{\muhat(\bS_1)-\mu^0(\bS_1)\}^2\right]=O_p\left(\bar b_N^2+\delta_N^4\right)\;\;\mbox{and}\;\;\bar b_N^2:=w_{N,d_1}(q_\mu,\beta_\mu)+r_N^2,\label{rate:muhat}
\end{align}
where the function $w_{N,p}$ is defined in \eqref{def:w_Nd}.
\end{theorem}

According to Theorem \ref{cor:mu-DR}, $\muhat$ is consistent as long as $\bar b_N=o(1)$ and $\delta_N=o(1)$. The former holds when $\mu^0$ can be well approximated by a sparse and smooth function, and the latter occurs as long as either $\nu^0$ or $\rho^0$ can be consistently estimated.

The following theorem establishes the consistency and asymptotic normality of the final sequential double machine learning estimator $\thetahat$.

\begin{theorem}\label{thm:normal-DR}
Let Assumptions \ref{cond:seq-ign}, \ref{cond:converge-rate}, \ref{cond:bound}, and \ref{cond:smooth} hold, $d_1=o(N/\log^5N)$. Choose $L$ and $H$ as in Theorem \ref{thm:imp-MLP} with $p=d_1$, $q=q_\mu$, and $\beta=\beta_\mu$. Define $\bar b_N$ as in \eqref{rate:muhat}. Then, as $N\to\infty$,
\begin{align}\label{rate:consistency}
\thetahat-\theta=O_p\left(a_Nd_N+\bar b_Nc_N+\delta_N^2c_N+N^{-1/2}\right).
\end{align}
Additionally, assume $a_N,c_N,d_N,r_N,\delta_N=o(1)$,
\begin{align}\label{rate:prod}
a_Nd_N=o\left(N^{-1/2}\right),\;\;\bar b_Nc_N=o\left(N^{-1/2}\right),\;\;\text{and}\;\;\delta_N^2c_N=o\left(N^{-1/2}\right).
\end{align}
Define $\sigmahat^2:=N^{-1}\sum_{k=1}^K\sum_{i\in\mathcal I_k}\{\psi(\bZ_i;\etahat_{-k})-\thetahat\}^2$ and let $\sigma^2:=\Var\{\psi(\bZ;\eta^0)\}>c$ with some constant $c>0$. Then, as $N\to\infty$, in distribution,
\begin{align}\label{normal:MLP}
\sqrt N \sigma^{-1}(\thetahat-\theta)\to N(0,1)\;\;\text{and}\;\;\sigmahat^2=\sigma^2\{1+o_p(1)\}.
\end{align}
\end{theorem}

\begin{remark}[Consistency and asymptotic normality]
As far as we are aware, Theorem \ref{thm:normal-DR} also represents the first consistent and asymptotically normal result for DTE estimation that allows for non-parametric models with a growing number of dimension.

We first discuss when the DTE estimator is consistent. Under Assumption \ref{cond:bound}, we can choose $\delta_N^2=O(a_N\land d_N)$. Therefore, together with \eqref{rate:consistency}, $\thetahat$ is consistent when all the following conditions hold: (a) $a_N=o(1)$ or $d_N=o(1)$, i.e., either $\nu^0$ or $\rho^0$ can be consistently estimated, (b) $\bar b_N=o(1)$ or $c_N=o(1)$, i.e., either $\mu^0$ can be consistently estimated under an oracle case with observable $Y^\#$ or $\pi^0$ can be consistently estimated.

As for asymptotic normality, we require the product rate conditions $a_Nd_N+\bar b_Nc_N=o(N^{-1/2})$. Since the estimation of $\mu^0$ is based on doubly robust outcomes \eqref{def:Yhat}, which involve nuisance estimates $\nuhat$ and $\rhohat$, its estimation error also depends on the estimation error of such nuisance functions, as demonstrated in Theorem \ref{cor:mu-DR}. As a result, we also require an additional product rate condition $\delta_N^2c_N=o(N^{-1/2})$, where $\delta_N^2$ itself, defined in \eqref{def:deltaN}, is also a product rate capturing the estimation error at the first learning stage. As $\delta_N^2=O(a_N\land d_N)$, a sufficient condition for the last product rate condition to hold is either $a_Nc_N=o(N^{-1/2})$ or $c_Nd_N=o(N^{-1/2})$. When a tighter upper bound is obtained for $\delta_N$, such a requirement can be further relaxed, as seen in Example \ref{ex:Lasso} below.
\end{remark}

\subsection{Examples}\label{sec:examples}

In the following, we provide specific examples for the nuisance estimates $(\pihat, \rhohat, \nuhat)$ and discuss the estimation and inference results for the DTE. Similar to Section \ref{sec:examples-CATE}, assume that $\mu^0$ is exactly sparse for the sake of simplicity, i.e., $r_N$ defined in Assumption \ref{cond:smooth} is exactly zero.

\begin{example}[First-stage learning through DNNs]\label{ex:DNN}
Consider DNN estimates at the first learning stage. Let $\pi^0$, $\rho^0$, and $\nu^0$ be exactly sparse functions with sparsity levels of $q_\pi$, $q_\rho$, and $q_\nu$ and H\"older smoothness levels of $\beta_\pi$, $\beta_\rho$, and $\beta_\nu$, respectively. Construct the nuisance estimates $\pihat_{-k}$, $\rhohat_{-k}$, and $\nuhat_{-k}$ in Steps 5-7 of Algorithm \ref{alg:DR-theta} as follows:
\begin{align*}
\fhat_{-k}^\pi&\in\arg\min_{f\in\mathcal F_{\mbox{\tiny MLP}}(L_{\pi},H_{\pi}):\|f\|_\infty\leq2M}|\mathcal I_{-k}|^{-1}\sum_{i\in\mathcal I_{-k}}\left(-T_{1i}f(\bS_{1i})+\log[1+\exp\{f(\bS_{1i})\}]\right),\\
\fhat_{-k}^\rho&\in\arg\min_{f\in\mathcal F_{\mbox{\tiny MLP}}(L_{\rho},H_{\rho}):\|f\|_\infty\leq2M}|\mathcal I_{-k}|^{-1}\sum_{i\in\mathcal I_{-k}}T_{1i}\left(-T_{2i}f(\bSbar_{2i})+\log[1+\exp\{f(\bSbar_{2i})\}]\right),\\
\nuhat_{-k}&\in\arg\min_{f\in\mathcal F_{\mbox{\tiny MLP}}(L_{\nu},H_{\nu}):\|f\|_\infty\leq2M}n^{-1}\sum_{i\in\mathcal I^k}T_{1i}T_{2i}\left\{Y_i-f(\bSbar_{2i})\right\}^2,
\end{align*}
with $\pihat_{-k}=\phi(\fhat_{-k}^\pi)$, $\rhohat_{-k}=\phi(\fhat_{-k}^\rho)$, and $\phi(u)=\exp(u)/\{1+\exp(u)\}$ for any $u\in\R$. Construct nuisance estimates $\rhohat^{(2)}$ and $\nuhat^{(2)}$ of Algorithm \ref{alg:DR-mu} analogously using different training samples. Similar to Example \ref{ex:DNN-CATE}, the convergence rates \eqref{def:aN}-\eqref{def:dN} can be chosen as $a_N^2=w_{N,\bar d}(q_\nu,\beta_\nu)$, $c_N^2=w_{N,d_1}(q_\pi,\beta_\pi)$, and $d_N^2=w_{N,\bar d}(q_\rho,\beta_\rho)$, where the function $w_{N,p}$ is defined in \eqref{def:w_Nd}. Since $\delta_N^2=O(a_N\land d_N)$, \eqref{rate:consistency} leads to a consistency rate
\begin{align*}
\thetahat-\theta&=O_p\left(N^{-1/2}+w_{N,\bar d}^{1/2}(q_\nu,\beta_\nu)w_{N,\bar d}^{1/2}(q_\rho,\beta_\rho)\right)\\
&\qquad+O_p\left(\left\{w_{N,d_1}^{1/2}(q_\mu,\beta_\mu)+w_{N,\bar d}^{1/2}(q_\nu,\beta_\nu)\land w_{N,\bar d}^{1/2}(q_\rho,\beta_\rho)\right\}w_{N,d_1}^{1/2}(q_\pi,\beta_\pi)\right).
\end{align*}
The asymptotic normality \eqref{normal:MLP} requires $a_Nd_N+w_{N,d_1}^{1/2}(q_\mu,\beta_\mu)c_N+\delta_N^2c_N=o(N^{-1/2})$. For the sake of simplicity, consider the case with $d_1\ll N^\frac{q_\mu}{2\beta_\mu+2q_\mu}\land N^\frac{q_\pi}{2\beta_\pi+2q_\pi}$ and $\bar d\ll N^\frac{q_\nu}{2\beta_\nu+2q_\nu}\land N^\frac{q_\rho}{2\beta_\rho+2q_\rho}$. Then, \eqref{normal:MLP} holds when (a) $\beta_\mu\beta_\pi>q_\mu q_\pi$, (b) $\beta_\nu\beta_\rho>q_\nu q_\rho$, and (c) $\beta_\nu\beta_\pi>q_\nu q_\pi$.
\end{example}

\begin{example}[First-stage learning through Lasso and DNNs]\label{ex:Lasso}
In the following, we explore scenarios where a significant number of covariates are collected prior to the second exposure, specifically focusing on cases where $\bar d=d_1+d_2\gg N\gg d_1$. In such instances, where the total dimension substantially exceeds the sample size, DNNs become less suitable for estimating the nuisance functions associated with the second exposure, namely $\nu^0$ and $\rho^0$. Instead, we consider linear and logistic regression with $\ell_1$-regularization. For any $\bsbar_2\in\R^{\bar d}$ and $k\leq K$, let $\nuhat_{-k}(\bsbar_2)=\bsbar_2^\top\bbetahat_\nu^{-k}$ and $\rhohat_{-k}(\bsbar_2)=\phi(\bsbar_2^\top\bbetahat_\rho^{-k})$, where
\begin{align*}
\bbetahat_\nu^{-k}&\in\arg\min_{\bbeta\in\R^d}n^{-1}\sum_{i\in\mathcal I_{-k}}T_{1i}T_{2i}\left(Y_i-\bSbar_{2i}^\top\bbeta\right)^2+\lambda_\nu\|\bbeta\|_1,\\
\bbetahat_\rho^{-k}&\in\arg\min_{\bbeta\in\R^d}n^{-1}\sum_{i\in\mathcal I_{-k}}T_{1i}\left[-T_{2i}\bSbar_{2i}^\top\bbeta+\log\{1+\exp(\bSbar_{2i}^\top\bbeta)\}\right]+\lambda_\rho\|\bbeta\|_1,
\end{align*}
with tuning parameters $\lambda_\nu,\lambda_\rho\geq0$. Let $\bbeta_\nu^*$ and $\bbeta_\rho^*$ represent the best population linear/logistic slopes, with sparsity levels $\|\bbeta_\nu^*\|_0\leq s_\nu$ and $\|\bbeta_\rho^*\|_0\leq s_\rho$, respectively. Additionally, let the linear and logistic approximation errors satisfy $E[T_1T_2\{\nu^0(\bSbar_2)-\bSbar_2^\top\bbeta_\nu^*\}^2]\leq e_\nu^2$ and $E[T_1\{\rho^0(\bSbar_2)-\phi(\bSbar_2^\top\bbeta_\rho^*)\}^2]\leq e_\rho^2$. Then, we can set $a_N$, $d_N$, and $\delta_N$, defined in \eqref{def:aN}, \eqref{def:dN}, and \eqref{def:deltaN}, as follows: $a_N^2=s_\nu\log\bar d/N+e_\nu^2$, $d_N^2=s_\rho\log\bar d/N+e_\nu^2$, and $\delta_N^2=a_Nd_N$. 

Considering that the dimension of covariates at the first exposure $d_1$ is relatively small compared to $N$, we employ DNNs to estimate $\pi^0$ (as well as $\mu^0$ in the second learning stage). As discussed in Example \ref{ex:DNN}, if $\pi^0$ is sparse and a smooth function with sparsity level $q_\pi$ and H\"older smoothness level $\beta_\pi$, we can set $c_N$, defined in \eqref{def:cN}, as $c_N^2=w_{N,d_1}(q_\pi,\beta_\pi)$. According to \eqref{rate:consistency}, we have the following rate of estimation for the final parameter of interest:
\begin{align*}
\thetahat-\theta&=O_p\left(\left(\sqrt\frac{s_\nu\log\bar d}{N}+e_\nu\right)\left(\sqrt\frac{s_\rho\log\bar d}{N}+e_\rho\right)\right)\\
&\qquad+O_p\left(w_{N,d_1}^{1/2}(q_\mu,\beta_\mu)w_{N,d_1}^{1/2}(q_\pi,\beta_\pi)+N^{-1/2}\right).
\end{align*}
Additionally, let $e_\nu=e_\rho=0$, then the asymptotic normality \eqref{normal:MLP} is ensured when $s_\nu s_\rho=o(N/\log^2\bar d)$ and $w_{N,d_1}(q_\mu,\beta_\mu)w_{N,d_1}(q_\pi,\beta_\pi)=o(N^{-1})$. The latter condition holds, for instance, when (a) $d_1\ll N^\frac{q_\mu}{2\beta_\mu+2q_\mu}\land N^\frac{q_\pi}{2\beta_\pi+2q_\pi}$ and $\beta_\mu\beta_\pi>q_\mu q_\pi$, or (b) $N^\frac{q_\mu}{2\beta_\mu+2q_\mu}\lor N^\frac{q_\pi}{2\beta_\pi+2q_\pi}\ll d_1\ll N^{\frac{4\beta_\mu\beta_\pi-q_\mu q_\pi}{8\beta_\mu\beta_\pi+2\beta_\mu q_\pi+2\beta_\pi q_\mu}-c}$ with any constant $c>0$.
\end{example}


\section{Discusssion}\label{sec:dis}

Traditional non-parametric regression methods have demonstrated effectiveness in addressing low-dimensional problems, offering minimax optimal convergence rates under smooth functional classes. However, they encounter significant challenges when confronted with large covariate dimensions $p$, known as the ``curse of dimensionality.'' While deep neural networks (DNNs) are frequently employed in such scenarios, their empirical success often lacks theoretical assurance. This research aims to bridge this gap by providing convergence rates for DNNs as $p$ grows with the sample size, assuming a sparse structure on the conditional mean function. Unlike traditional non-parametric methods, DNNs have the capability to adaptively learn the sparse structure of the regression function by learning representations from the data. This approach shows promise for addressing the challenges posed by high-dimensional data, particularly for causal inference problems where a large number of covariates are typically collected to reduce confounding bias in observational studies.

We explore the application of DNNs in intricate causal inference problems necessitating multi-stage learning. While our theoretical framework is broad, we specifically demonstrate its effectiveness in estimating the conditional average treatment effect (CATE) and dynamic treatment effect (DTE), representing two- and three-stage learning problems, respectively. These results extend to general multi-stage learning problems, even when more than three learning stages are involved, albeit necessitating additional, similarly detailed proofs.

Our theoretical contributions primarily center on user-friendly multilayer perceptrons, known for their straightforward structures and minimal tuning demands. Future research avenues may explore more sophisticated network architectures to further enhance overall estimation accuracy while maintaining manageable computational complexity.

 

\appendix

\bibliographystyle{plainnat}
\bibliography{ref}


\renewcommand{\thetheorem}{S.\arabic{theorem}}
\renewcommand{\thelemma}{S.\arabic{lemma}}

\clearpage\newpage  
\par\bigskip         
\begin{center}
\textbf{\uppercase{Supplement to ``Causal inference through multi-stage learning and doubly robust deep neural networks''}}
\end{center}

\par\medskip
This supplementary document provides the proofs for the theoretical results presented in the main document, including Theorems \ref{thm:imp-DNN}-\ref{thm:normal-DR}, along with auxiliary lemmas and their proofs. All results and notation are consistent with those in the main text unless stated otherwise. 

To simplify the exposition, we begin by listing some shorthand notations used throughout the supplementary document. For any sequence $f_i$, define $E_n(f_i):=n^{-1}\sum_{i=1}^nf_i$. For any functional class $\mathcal F$, define $R_n\mathcal F:=\sup_{f\in\mathcal F}n^{-1}\sum_{i=1}^n\delta_if(\bZ_i)$, where $\delta_i\in\{-1,1\}$ with equal probability $1/2$. We denote $E_\delta(R_n\mathcal F)$ as the empirical Rademacher complexity, where the expecation is taken conditional on $(Z_i)_{i=1}^n$; the Rademacher complexity, $E(R_n\mathcal F)$, is defined through taking the expectation over both $(Z_i)_{i=1}^n$ and $(\delta_i)_{i=1}^n$. 

\section{Proof of the results in Section \ref{sec:imp-DNN-theory}}

\subsection{Auxiliary lemmas}\label{sec:lemma}
Define $\Delta_{1i}:=\Delta_1(\bZ_i;\S^2)$, $\Delta_{2i}:=\Delta_2(\bZ_i;\S^2)$ and
\begin{align}\label{def:f_bar}
\bar f\in\arg\min_{f\in\mathcal F_\mathrm{DNN}:\|f\|_\infty\leq2M}\|f-f^0\|_\infty.
\end{align}

\begin{lemma}\label{lemma:bound_basic}
Let the assumptions in Theorem \ref{thm:imp-DNN} hold. Then,
\begin{align}
&E_n\left[R_i\{\fhat(\bX_i)-f^0(\bX_i)\}^2\right]\nonumber\\
&\qquad\leq E_n\left[R_i\{f^0(\bX_i)-\bar f(\bX_i)\}^2\right]+2\sqrt{E_n(R_i\Delta_{2i}^2)E_n\left[R_i\{\fhat(\bX_i)-\bar f(\bX_i)\}^2\right]}\nonumber\\
&\qquad\qquad+2(E_n-E_\bZ)\left[R_i\{\fhat(\bX_i)-\bar f(\bX_i)\}\{Y_i^\#-f^0(\bX_i)+\Delta_{1i}\}\right]\label{bound:Enmuhat-mu0}
\end{align}
and
\begin{align*}
E_\bX\left[R\{\fhat(\bX)-f^0(\bX)\}^2\right]\leq E\left[R\{f^0(\bX)-\bar f(\bX)\}^2\right]+W_1+W_2,
\end{align*}
where
\begin{align}
W_1&:=(E_n-E_\bZ)\left[R_i\{\fhat(\bX_i)-\bar f(\bX_i)\}\{2Y_i^\#-\fhat(\bX_i)-\bar f(\bX_i)+2\Delta_{1i}\}\right],\label{def:W1}\\
W_2&:=2\sqrt{E_n(R_i\Delta_{2i}^2)E_n\left[R_i\{\fhat(\bX_i)-\bar f(\bX_i)\}^2\right]}.\label{def:W2}
\end{align}
\end{lemma}

\begin{lemma}\label{lemma:C1-C2}
For any $t>0$, consider any $r>0$ that satisfies
\begin{align}
&\text{C1}: r^2\geq24ME(R_n\mathcal F_r),\label{def:C1}\\
&\text{C2}: r^2\geq\frac{32M^2t}{n},\label{def:C2}
\end{align}
where 
\begin{align}
\mathcal F_r:=\{g=R(f-\bar f): f\in\mathcal F_\mathrm{DNN}, \|Rf\|_\infty\leq2M, E\{g^2(\bX)\}\leq r^2\}.\label{def:Fr}
\end{align}
Then, with probability at least $1-e^{-t}$,
\begin{align}
\sup_{g\in\mathcal F_r}E_n\{g^2(\bX_i)\}\leq4r^2.\label{bound:supEn}
\end{align}
\end{lemma}

\begin{lemma}\label{lemma:O_p}
Let the assumptions in Theorem \ref{thm:imp-DNN} hold. Then, for any $t>0$, there exists some $M_t\geq M\geq0$ such that
\begin{align}
&P(E_1^c)<t^{-1},\;\;\text{where}\;\;E_1=E_1(t):=\left\{\|R\Delta_1\|_{\infty,P_\bZ}\leq M_t\right\};\label{def:A1}\\
&P(E_2^c)<t^{-1},\;\;\text{where}\;\;E_2=E_2(t):=\left\{E_n(R_i\Delta_{2i}^2)\leq M_t^2e_n^2\right\}.\label{def:A2}
\end{align}
\end{lemma}

\begin{lemma}\label{lemma:C1-C3}
Let the assumptions in Theorem \ref{thm:imp-DNN} hold and $n>\max\{(2eM)^2,CWL\log W\}$, where $C>0$ is some constant. Let $r=r_0>0$ satisfies $r_0\geq1/n$, Conditions C1-C2 defined in \eqref{def:C1}-\eqref{def:C2}, as well as
\begin{align}
\text{C3}: r^2\geq E_\bX\left[R_i\{\fhat(\bX)-\bar f(\bX)\}^2\right].\label{def:C3}
\end{align}
Then, for any $t>0$, on the event $E_1\cap E_2$, with probability at least $1-3e^{-t}$,
\begin{align*}
E_\bX\left[R\{\fhat(\bX)-\bar f(\bX)\}^2\right]\leq c_t\left\{r_0\left(\sqrt{\frac{WL\log W\log n}{n}}+\sqrt\frac{t}{n}+e_n\right)+\frac{t}{n}+\epsilon_n^2\right\},
\end{align*}
where $c_t>0$ is a constant and $\epsilon_n$ is defined as \eqref{def:eps}.
\end{lemma}

\begin{lemma}\label{lemma:rstar}
Define
\begin{align}
r_*:=\inf\{r>0:24ME(R_n\mathcal F_s)<s^2,\forall s\geq r\}.\label{def:rstar}
\end{align}
Then, with some constant $K>0$,
\begin{align}
r_*\leq K\sqrt\frac{WL\log W\log n}{n}.\label{bound:r_star}
\end{align}
\end{lemma}

\subsection{Proof of the auxiliary lemmas}

\begin{proof}[Proof of Lemma \ref{lemma:bound_basic}]

Observe that
\begin{align}
&E_n\left[R_i\{\fhat(\bX_i)-f^0(\bX_i)\}^2\right]\nonumber\\
&\qquad=E_n\left[R_i\{\fhat(\bX_i)-Y_i^\#\}^2\right]-E_n\left[R_i\{\fhat(\bX_i)-Y_i^\#\}^2-R_i\{\fhat(\bX_i)-f^0(\bX_i)\}^2\right]\nonumber\\
&\qquad=E_n\left[R_i\{\fhat(\bX_i)-Y_i^\#\}^2\right]-E_\bZ\left[R\{\fhat(\bX)-Y^\#\}^2-R\{\fhat(\bX)-f^0(\bX)\}^2\right]\nonumber\\
&\qquad\qquad-(E_n-E_\bZ)\left[R_i\{\fhat(\bX_i)-Y_i^\#\}^2-R_i\{\fhat(\bX_i)-f^0(\bX_i)\}^2\right],\label{eq:Enmuhat-mu0}
\end{align}
where
\begin{align}
&E_\bZ\left[R\{\fhat(\bX)-Y^\#\}^2-R\{\fhat(\bX)-f^0(\bX)\}^2\right]\nonumber\\
&\qquad=E_\bZ\left[R\{Y^\#- f(\bX)\}^2\right]+2E_\bZ\left[R\{\fhat(\bX)-f^0(\bX)\}\{f^0(\bX)-Y^\#\}\right].\label{eq:Enmuhat-mu0'}
\end{align}
By definition, $f^0(\bX)=E(Y^\#\mid\bX,R=1)$. Hence, using the tower rule, we have
\begin{align*}
&E_\bZ\left[R\{\fhat(\bX)-f^0(\bX)\}\{f^0(\bX)-Y^\#\}\right]\\
&\qquad=E_\bZ\left(E_\bZ\left[\{\fhat(\bX)-f^0(\bX)\}\{f^0(\bX)-Y^\#\}\mid\bX,R=1\right]P(R=1\mid\bX)\right)=0.
\end{align*}
Together with \eqref{eq:Enmuhat-mu0} and \eqref{eq:Enmuhat-mu0'}, we have
\begin{align}
&E_n\left[R_i\{\fhat(\bX_i)-f^0(\bX_i)\}^2\right]\nonumber\\
&\qquad=E_n\left[R_i\{\fhat(\bX_i)-Y_i^\#\}^2\right]-E_\bZ\left[R\{Y^\#- f(\bX)\}^2\right]\nonumber\\
&\qquad\qquad-(E_n-E_\bZ)\left[R_i\{\fhat(\bX_i)-Y_i^\#\}^2-R_i\{\fhat(\bX_i)-f^0(\bX_i)\}^2\right].\label{eq:Enmuhat-mu0''}
\end{align}
In addition,
\begin{align}
&E_n\left[R_i\{\fhat(\bX_i)-Y_i^\#\}^2\right]\nonumber\\
&\quad=E_n\left\{R_i(\Yhat_i-Y_i^\#)^2\right\}+E_n\left[R_i\{\Yhat_i-\fhat(\bX_i)\}^2\right]-2E_n\left[R_i(\Yhat_i-Y_i^\#)\{\Yhat_i-\fhat(\bX_i)\}\right]\nonumber\\
&\quad\overset{(i)}{\leq}E_n\left\{R_i(\Yhat_i-Y_i^\#)^2\right\}+E_n\left[R_i\{\Yhat_i-\bar f(\bX_i)\}^2\right]-2E_n\left[R_i(\Yhat_i-Y_i^\#)\{\Yhat_i-\fhat(\bX_i)\}\right]\nonumber\\
&\quad=2E_n\left[R_i(\Yhat_i-Y_i^\#)\{\fhat(\bX_i)-\bar f(\bX_i)\}\right]+E_\bZ\left[R\{Y^\#-\bar f(\bX)\}^2\right]\nonumber\\
&\quad\qquad+(E_n-E_\bZ)\left[R_i\{Y_i^\#-\bar f(\bX_i)\}^2\right],\label{bound:construct}
\end{align}
where (i) holds by the constructions of $\fhat$ and $\bar f$; see \eqref{def:fhat} and \eqref{def:f_bar}. Since $f^0(\bX)=E(Y^\#\mid\bX,R=1)$, we have
\begin{align*}
&E_\bZ\left[R\{Y^\#-f^0(\bX)\}\{f^0(\bX)-\bar f(\bX)\}\right]\\
&\qquad=E_\bZ\left(E_\bZ\left[\{Y^\#-f^0(\bX)\}\{f^0(\bX)-\bar f(\bX)\}\mid\bX,R=1\right]P(R=1\mid\bX)\right)=0,
\end{align*}
and hence
\begin{align*}
E_\bZ\left[R_i\{Y^\#-\bar f(\bX)\}^2\right]=E_\bZ\left[R\{Y^\#-f^0(\bX)\}^2\right]+E_\bZ\left[R_i\{f^0(\bX)-\bar f(\bX)\}^2\right].
\end{align*}
Together with \eqref{eq:Enmuhat-mu0''} and \eqref{bound:construct}, we have
\begin{align}
&E_n\left[R_i\{\fhat(\bX_i)-f^0(\bX_i)\}^2\right]\nonumber\\
&\qquad=(E_n-E_\bZ)\left(R_i\left[\{Y_i^\#-\bar f(\bX_i)\}^2-\{\fhat(\bX_i)-Y_i^\#\}^2+\{\fhat(\bX_i)-f^0(\bX_i)\}^2\right]\right)\nonumber\\
&\qquad\qquad+E_\bZ\left[R_i\{f^0(\bX)-\bar f(\bX)\}^2\right]+2E_n\left[R_i(\Yhat_i-Y_i^\#)\{\fhat(\bX_i)-\bar f(\bX_i)\}\right]\nonumber\\
&\qquad=2(E_n-E_\bZ)\left[R_i\{\fhat(\bX_i)-\bar f(\bX_i)\}\{Y_i^\#-f^0(\bX_i)\}\right]\nonumber\\
&\qquad\qquad+E_n\left[R_i\{f^0(\bX_i)-\bar f(\bX_i)\}^2\right]+2E_n\left[R_i(\Yhat_i-Y_i^\#)\{\fhat(\bX_i)-\bar f(\bX_i)\}\right].\label{eq:Enmuhat-mu0'''}
\end{align}
Here,
\begin{align}
&E_n\left[R_i(\Yhat_i-Y_i^\#)\{\fhat(\bX_i)-\bar f(\bX_i)\}\right]\nonumber\\
&\qquad=E_n\left[R_i\Delta_{1i}\{\fhat(\bX_i)-\bar f(\bX_i)\}\right]+E_n\left[R_i\Delta_{2i}\{\fhat(\bX_i)-\bar f(\bX_i)\}\right]\nonumber\\
&\qquad\leq(E_n-E_\bZ)\left[R_i\Delta_{1i}\{\fhat(\bX_i)-\bar f(\bX_i)\}\right]\nonumber\\
&\qquad\qquad+\sqrt{E_n(R_i\Delta_{2i}^2)}\sqrt{E_n\left[R_i\{\fhat(\bX_i)-\bar f(\bX_i)\}^2\right]},\label{bound:prod1}
\end{align}
by the Cauchy-Schwarz inequality and the fact that
\begin{align*}
E_\bZ\left[R\Delta_1\{\fhat(\bX)-\bar f(\bX)\}\right]=E_\bZ\left[\{\fhat(\bX)-\bar f(\bX)\}E_\bZ(\Delta_1\mid\bX,R=1)P(R=1\mid\bX)\right]=0.
\end{align*}
Combining \eqref{eq:Enmuhat-mu0'''} and \eqref{bound:prod1}, we conclude that \eqref{bound:Enmuhat-mu0} holds. Moreover,
\begin{align*}
&E_\bX\left[R_i\{\fhat(\bX)-f^0(\bX)\}^2\right]\\
&\qquad=E_n\left[R_i\{\fhat(\bX_i)-f^0(\bX_i)\}^2\right]-(E_n-E_\bX)\left[R_i\{\fhat(\bX_i)-f^0(\bX_i)\}^2\right]\\
&\qquad\leq E\left[R\{f^0(\bX)-\bar f(\bX)\}^2\right]+2\sqrt{E_n(R_i\Delta_{2i}^2)}\sqrt{E_n\left[R_i\{\fhat(\bX_i)-\bar f(\bX_i)\}^2\right]}\\
&\qquad\qquad+(E_n-E_\bZ)\left[2R_i\{\fhat(\bX_i)-\bar f(\bX_i)\}\{Y_i^\#-f^0(\bX_i)+\Delta_{1i}\}\right]\\
&\qquad\qquad+(E_n-E_\bZ)\left[R_i\{f^0(\bX_i)-\bar f(\bX_i)\}^2-R_i\{\fhat(\bX_i)-f^0(\bX_i)\}^2\right]\\
&\qquad=E\left[R\{f^0(\bX)-\bar f(\bX)\}^2\right]+2\sqrt{E_n(R_i\Delta_{2i}^2)}\sqrt{E_n\left[R_i\{\fhat(\bX_i)-\bar f(\bX_i)\}^2\right]}\\
&\qquad\qquad+(E_n-E_\bZ)\left[R_i\{\fhat(\bX_i)-\bar f(\bX_i)\}\{2Y_i^\#-\fhat(\bX_i)-\bar f(\bX_i)+2\Delta_{1i}\}\right].
\end{align*}
\end{proof}

\begin{proof}[Proof of Lemma \ref{lemma:C1-C2}]
For any $g\in\mathcal F_r$, there exists some $f\in\mathcal F_\mathrm{DNN}$ and $\|Rf\|_\infty\leq2M$ such that $g=R(f-\bar f)$. Let $h_1(g,\bX):=g^2(\bX)$. Then,
\begin{align*}
|h_1(g,\bX)|&= R\left\{f(\bX)-\bar f(\bX)\right\}^2\leq16M^2,\\
\mathrm{Var}\{h_1(g,\bZ)\}&\leq E\left[R\left\{f(\bX)-\bar f(\bX)\right\}^4\right]\leq16M^2E\left[R\left\{f(\bX)-\bar f(\bX)\right\}^2\right]\leq16M^2r^2.
\end{align*}
By Theorem 2.1 of \cite{bartlett2005local}, with probability at least $1-e^{-t}$,
\begin{align*}
\sup_{f\in\mathcal F_r}(E_n-E)\{h_1(g,\bX_i)\}\leq3ER_n\{h_1(g,\bX_i):g\in\mathcal F_r\}+4Mr\sqrt\frac{2t}{n}+\frac{64M^2t}{3n}.
\end{align*}
Hence, with probability at least $1-e^{-t}$,
\begin{align}
&\sup_{f\in\mathcal F_r}E_n\{h_1(g,\bX_i)\}\leq\sup_{f\in\mathcal F_r}(E_n-E)\{h_1(g,\bX_i)\}+\sup_{f\in\mathcal F_r}E\{h_1(g,\bX_i)\}\nonumber\\
&\qquad\leq3E\left[R_n\{h_1(g,\bX_i):g\in\mathcal F_r\}\right]+4Mr\sqrt\frac{2t}{n}+\frac{64M^2t}{3n}+r^2.\label{bound:Eng1}
\end{align}
For any $g_1,g_2\in\mathcal F_r$,
\begin{align*}
\left|h_1(g_1,\bX)-h_1(g_2,\bX)\right|=\left|g_1(\bX)+g_2(\bX)\right|\left|g_1(\bX)-g_2(\bX)\right|\leq4M\left|g_1(\bX)-g_2(\bX)\right|.
\end{align*}
By Lemma 2 of \cite{farrell2021deep},
\begin{align*}
E_\delta\left[R_n\{h_1(g,\bX_i):g\in\mathcal F_r\}\right]\leq8ME_\delta(R_n\mathcal F_r).
\end{align*}
By the law of total expectation and together with \eqref{bound:Eng1}, with probability at least $1-e^{-t}$,
\begin{align*}
\sup_{f\in\mathcal F_r}E_n\{h_1(g,\bX_i)\}&\leq24ME(R_n\mathcal F_r)+4Mr\sqrt\frac{2t}{n}+\frac{64M^2t}{3n}+r^2\leq4r^2,
\end{align*}
when Conditions C1 and C2 hold. That is, \eqref{bound:supEn} holds with probability at least $1-e^{-t}$.
\end{proof}

\begin{proof}[Proof of Lemma \ref{lemma:O_p}]
Firstly, \eqref{def:A1} holds since $\|R\Delta_1\|_{\infty,P_\bZ}=O_p(1)$. In addition, we have
\begin{align*}
E_{\S^1}\{E_n(R_i\Delta_{2i}^2)\}=E_\bZ(R\Delta_2^2)=O_p(e_n^2).
\end{align*}
By Markov's inequality,
\begin{align*}
E_n(R_i\Delta_{2i}^2)=O_p(e_n^2),
\end{align*}
and hence \eqref{def:A2} holds.
\end{proof}

\begin{proof}[Proof of Lemma \ref{lemma:C1-C3}]

Recall that $W_2$ is defined as \eqref{def:W2}. On the event $E_2$, \eqref{def:A2}, by Lemma \ref{lemma:C1-C2}, $W_2$ satisfies
\begin{align}
W_2\leq2\sqrt{E_n(R_i\Delta_{2i}^2)}\sqrt{\sup_{g\in\mathcal F_{r_0}}E_n\{g^2(\bX_i)\}}\leq2M_te_nr_0,\label{bound:W2}
\end{align}
with probability at least $1-e^{-t}$.
Now, let us consider $W_1$ defined as \eqref{def:W1}. For any $g\in\mathcal F_{r_0}$, there exists some $f\in\mathcal F_\mathrm{DNN}$ and $\|Rf\|_\infty\leq2M$ such that $g=R(f-\bar f)$, where $\mathcal F_{r_0}$ and $\bar f$ are defined through \eqref{def:Fr} and \eqref{def:f_bar}, respectively. Let 
\begin{align*}
h_2(g,\bX):=&g(\bX)\{2Y^\#-2\bar f(\bX)+2\Delta_1-g(\bX)\}\\
=&R\{f(\bX)-\bar f(\bX)\}\{2Y^\#- f(\bX)-\bar f(\bX)+2\Delta_1\}.
\end{align*}
Then,
\begin{align}
W_1\leq\sup_{g\in\mathcal F_{r_0}}(E_n-E_\bZ)\{h_2(g,\bZ_i)\}.\label{bound:W1}
\end{align}
Note that $\|Rf\|_\infty,\|R\bar f\|_\infty,\|RY^\#\|_\infty\leq2M$. On the event $E_1$, \eqref{def:A1}, we have
\begin{align*}
|h_2(g,\bX)|&\leq(\|Rf\|_\infty+\|R\bar f\|_\infty)\left(2\|RY^\#\|_\infty+2\|R\Delta_1\|_{\infty,P_\bZ}+\|Rf\|_\infty+\|R\bar f\|_\infty\right)\\
&\leq4M(2M+2M_t+4M)\leq32M_t^2,
\end{align*}
and
\begin{align*}
&\mathrm{Var}\{h_2(g,\bX)\}\leq E_\bZ\left[R\{ f(\bX)-\bar f(\bX)\}^2\{2Y^\#- f(\bX)-\bar f(\bX)+2\Delta_1\}^2\right]\\
&\qquad\leq\left(2\|RY^\#\|_\infty+2\|R\Delta_1\|_{\infty,P_\bZ}+\|Rf\|_\infty+\|R\bar f\|_\infty\right)^2E\left[R\{f(\bX)-\bar f(\bX)\}^2\right]\\
&\qquad\leq(2M+2M_t+4M)^2r_0^2\leq64M_t^2r_0^2.
\end{align*}
By Theorem 2.1 of \cite{bartlett2005local}, on the event $E_1$, with probability at least $1-2e^{-t}$,
\begin{align}
\sup_{g\in\mathcal F_{r_0}}(E_n-E_\bZ)\{h_2(g,\bX_i)\}&\leq6E_\delta R_n\{h_2(g,\bX_i):g\in\mathcal F_{r_0}\}+8M_tr_0\sqrt\frac{2t}{n}+\frac{736M_t^2t}{3n}.\label{bound:En-Eg2}
\end{align}
For any $g_1,g_2\in\mathcal F_{r_0}$, there exists $ f_1,f_2\in\mathcal F_\mathrm{DNN}$ and $\|Rf_1\|_\infty,\|Rf_2\|_\infty\leq2M$ such that $g_1(\bX)=R\{f_1(\bX)-\bar f(\bX)\}$ and $g_2(\bX)=R\{f_2(\bX)-\bar f(\bX)\}$. Hence,
\begin{align*}
&|h_2(g_1,\bZ)-h_2(g_2,\bZ)|= |g_1(\bX)-g_2(\bX)||2Y^\#-2\bar f(\bX)+2\Delta_1- g_1(\bX)- g_2(\bX)|\\
&\qquad=R|f_1(\bX)-f_2(\bX)||2Y^\#+2\Delta_1- f_1(\bX)- f_2(\bX)|\\
&\qquad\leq R|f_1(\bX)-f_2(\bX)|\left(2\|RY^\#\|_\infty+2\|R\Delta_1\|_{\infty,P_\bZ}+\|Rf_1\|_\infty+\| Rf_2\|_\infty\right)\\
&\qquad\leq8M_tR|f_1(\bX)-f_2(\bX)|.
\end{align*}
By Lemma 2 of \cite{farrell2021deep},
\begin{align}
E_\delta R_n\{h_2(g,\bX_i):g\in\mathcal F_{r_0}\}\leq16M_tE_\delta(R_n\mathcal F_{r_0}).\label{bound:Rng2}
\end{align}
For any $r>0$, define
\begin{align*}
\widetilde{\mathcal F}_{r}:=\{g:g(\bZ)=R\{f(\bX)-\bar f(\bX)\},f\in\mathcal F_\mathrm{DNN}, \|Rf\|_\infty\leq2M, E_n\{g^2(\bZ)\}\leq r^2\}.
\end{align*}
Then, by Lemma \ref{lemma:C1-C2}, $\mathcal F_{r_0}\subseteq\widetilde{\mathcal F}_{2r_0}$. Together with \eqref{bound:W1}, \eqref{bound:En-Eg2} and \eqref{bound:Rng2}, on the event $E_1$, with probability at least $1-2e^{-t}$,
\begin{align}
W_1&\leq96M_tE_\delta(R_n\mathcal F_{r_0})+8M_tr_0\sqrt\frac{2t}{n}+\frac{736M_t^2t}{3n}\nonumber\\
&\leq96M_tE_\delta(R_n\widetilde{\mathcal F}_{2r_0})+8M_tr_0\sqrt\frac{2t}{n}+\frac{736M_t^2t}{3n}.\label{bound:W1'}
\end{align}
By Lemmas 3 and 4 of \cite{farrell2021deep} and repeating the proof of Section A.2.2 theirin, when $n>\mathrm{Pdim}(\mathcal F_\mathrm{DNN})$,
\begin{align}
E_\delta(R_n\widetilde{\mathcal F}_{2r_0})&\leq32r_0\sqrt{\frac{\mathrm{Pdim}(\mathcal F_\mathrm{DNN})}{n}\left(\log\frac{2eM}{r_0}+\frac{3}{2}\log n\right)}\nonumber\\
&\leq64r_0\sqrt{\frac{\mathrm{Pdim}(\mathcal F_\mathrm{DNN})}{n}\log n},\label{bound:RnFtil}
\end{align}
whenever $r_0\geq1/n$ and $n\geq(2eM)^2$. Here, $\mathrm{Pdim}(\mathcal F_\mathrm{DNN})$ denotes the pseudo-dimension of the network $\mathcal F_\mathrm{DNN}$. By Theorem 7 of \cite{bartlett2019nearly}, with some constant $C>0$,
\begin{align*}
\mathrm{Pdim}(\mathcal F_\mathrm{DNN})\leq CWL\log W,
\end{align*}
and we can see that $n>\mathrm{Pdim}(\mathcal F_\mathrm{DNN})$ holds since $n>CWL\log W$.
Together with \eqref{bound:W1'} and \eqref{bound:RnFtil}, on the event $E_1$, with probability at least $1-2e^{-t}$,
\begin{align}
W_1\leq8M_tr_0\left(768\sqrt{\frac{CWL\log W\log n}{n}}+\sqrt\frac{2t}{n}\right)+\frac{736M_t^2t}{3n}.\label{bound:W1''}
\end{align}
Combining with Lemma \ref{lemma:bound_basic}, \eqref{bound:W2}, \eqref{bound:W1''} and the fact that $E\left[\{f^0(\bX)-\bar f(\bX)\}^2\right]\leq\epsilon_n^2$, we have
\begin{align*}
&E_\bX\left[R\{\fhat(\bX)-\bar f(\bX)\}^2\right]\leq2E_\bX\left[R\{\fhat(\bX)-f^0(\bX)\}^2\right]+2E\left[R\{f^0(\bX)-\bar f(\bX)\}^2\right]\\
&\;\;\leq4E\{f^0(\bX)-\bar f(\bX)\}^2+2W_1+2W_2\\
&\;\;\leq c_t\left\{r_0\left(\sqrt{\frac{WL\log W\log n}{n}}+\sqrt\frac{t}{n}+e_n\right)+\frac{t}{n}+\epsilon_n^2\right\},
\end{align*}
on the event $E_1\cap E_2$, with probability at least $1-3e^{-t}$ and some constant $c_t>0$.
\end{proof}

\begin{proof}[Proof of Lemma \ref{lemma:rstar}]
By the definition \eqref{def:rstar}, we know that $r=2r_*$ satisfies Condition C1, \eqref{def:C1}. Define
\begin{align*}
A:=\{E_n\{g^2(\bZ_i)\}\leq16r_*^2,\;\forall g\in\mathcal F_{2r^*}\}.
\end{align*}

Case 1: $r_*^2\geq8M^2\log n/n$. Then, $r=2r_*$ satisfies Condition C2, \eqref{def:C2}, with $t=\log n$. By Lemma \ref{lemma:C1-C2},
\begin{align}
P(A)\geq1-n^{-1}.\label{bound:PA}
\end{align}
By the definition \eqref{def:rstar}, there exists some constant $v\in(1/2,1]$ such that
\begin{align*}
(vr_*)^2&\leq24ME(R_n\mathcal F_{vr_*})\leq24ME(R_n\mathcal F_{2r_*})\\
&\leq24ME\left\{E_\delta(R_n\widetilde{\mathcal F}_{4r_*})\mathbbm1_A+4M(1-\mathbbm1_A)\right\}\\
&\overset{(i)}{\leq}1536Mr_*\sqrt{\frac{CWL\log W\log n}{n}}+\frac{96M^2}{n}\\
&\overset{(ii)}{\leq}1536Mr_*\sqrt{\frac{CWL\log W\log n}{n}}+\frac{48Mr_*}{\sqrt{2n\log n}},
\end{align*}
where (i) holds by Lemmas 3 and 4 of \cite{farrell2021deep} and Theorem 7 of \cite{bartlett2019nearly}, as well as \eqref{bound:PA}; (ii) holds since $r_*^2\geq8M^2\log n/n$. Hence, \eqref{bound:r_star} holds with some constant $K>0$.

Case 2: $r_*^2\leq8M^2\log n/n$. We also have \eqref{bound:r_star} holds with some constant $K>0$.
\end{proof}

\subsection{Proof of the main results}

\begin{proof}[Proof of Theorem \ref{thm:imp-DNN}]
Define
\begin{align}
\bar r&:=(K+4c_t)\sqrt\frac{WL\log W\log n}{n}+4c_t\left(\sqrt\frac{\log n}{n}+e_n\right)\nonumber\\
&\qquad+\sqrt{(32M^2+2c_t)\left(\frac{\log n}{n}+\epsilon_n^2\right)}+\frac{1}{n}.\label{def:rbar}
\end{align}
By Lemma \ref{lemma:rstar}, $\bar r>r_*$. For any $s\geq\bar r>r_*$, by the construction \eqref{def:rstar}, $r=s$ satisfies Condition C1, \eqref{def:C1}. In addition, $r=s\geq\bar r$ also satisfies $r\geq1/n$, as well as Condition C2, \eqref{def:C2}, with $t=\log n$.

In the following, we apply Lemma \ref{lemma:C1-C3} repeatedly. Choose $l\geq1$ be the smallest ingeter larger than $\log_2(16M^2n^2)$. Then,
\begin{align*}
l\leq1+\log_2(16M^2n)=5+2\log_2M+2\log_2n.
\end{align*}
Additionally, we have
\begin{align}
2^l\bar r^2\geq16M^2n^2\bar r^2\geq16M^2\geq E_\bX\left[R\{\fhat(\bX)-\bar f(\bX)\}^2\right],\label{bound:j=l}
\end{align}
since $\|R\fhat\|_\infty,\|R\bar f\|_\infty\leq2M$. Now, for each $j\in\{1,2,\dots,l\}$, we show that, on the event $E_1\cap E_2$,
\begin{align*}
E_\bX\left[R\{\fhat(\bX)-\bar f(\bX)\}^2\right]\leq2^{j-1}\bar r^2\;\;\text{with probability at least $1-3n^{-1}$},
\end{align*}
if we have
\begin{align}\label{bound:j}
E_\bX\left[R\{\fhat(\bX)-\bar f(\bX)\}^2\right]\leq2^j\bar r^2.
\end{align}
Starting with $j=l$, as shown in \eqref{bound:j=l}, we have \eqref{bound:j} holds. For any $j\in\{1,2,\dots,l\}$, assume that \eqref{bound:j} holds. That is, $r=2^j\bar r$ satisfies Condition C3, \eqref{def:C3}. Since $2^j\bar r>\bar r$, we know that $r=2^j\bar r$ also satisfies $r\geq1/n$ and Conditions C1, C2 with $t=\log n$. By Lemma \ref{lemma:C1-C3}, on the event $E_1\cap E_2$, with probability at least $1-3n^{-1}$,
\begin{align*}
&E_\bX\left[R\{\fhat(\bX)-\bar f(\bX)\}^2\right]\\
&\qquad\leq c_t2^j\bar r\left(\sqrt{\frac{WL\log W\log n}{n}}+\sqrt\frac{\log n}{n}+e_n\right)+\frac{c_t\log n}{n}+c_t\epsilon_n^2\\
&\qquad\overset{(i)}{\leq}\frac{2^{j-1}\bar r^2}{2}+\frac{\bar r^2}{2}\leq\frac{2^{j-1}\bar r^2}{2}+\frac{2^{j-1}\bar r^2}{2}=2^{j-1}\bar r^2,
\end{align*}
where (i) holds by the construction \eqref{def:rbar}. Applying the above strategy repeatedly from $j=l$ to $j=1$, we conclude that, on the event $E_1\cap E_2$,
\begin{align*}
E_\bX\left[R\{\fhat(\bX)-\bar f(\bX)\}^2\right]\leq\bar r^2,
\end{align*}
with probability at least
\begin{align*}
1-\frac{3l}{n}\geq1-\frac{3\log_2(16M^2n^2)}{n}=1-\frac{12+6\log_2M+6\log_2n}{n}=1-o(1).
\end{align*}
Additionally, since $P(E_1\cap E_2)<1-2/t$ and recall the definition of $\bar r$, \eqref{def:rbar}, we conclude that
\begin{align*}
E_\bX\left[R\{\fhat(\bX)-\bar f(\bX)\}^2\right]=O_p\left(\frac{WL\log W\log n}{n}+e_n^2+\epsilon_n^2\right).
\end{align*}
Therefore,
\begin{align*}
&E_\bX\left[R\{\fhat(\bX)-f^0(\bX)\}^2\right]\leq2E_\bX\left[R\{\fhat(\bX)-\bar f(\bX)\}^2\right]+2E\left[R\{f^0(\bX)-\bar f(\bX)\}^2\right]\\
&\qquad=O_p\left(\frac{WL\log W\log n}{n}+e_n^2+\epsilon_n^2\right).
\end{align*}
\end{proof}

\begin{proof}[Proof of Theorem \ref{thm:imp-opt}]
Define the following approximation error
\begin{align*}
\epsilon_{n,\mathrm{opt}}:=\inf_{f\in\mathcal F_\mathrm{DNN}(L,U,W):\|f\|_\infty\leq2M}\|f-f^0\|_\infty.
\end{align*}
Repeating the proof of Theorem \ref{thm:imp-DNN}, we have
\begin{align}
&E\left[R\{\fhat_\mathrm{opt}(\bX)-f^0(\bX)\}^2\right]=O_p\left(\frac{WL\log W\log n}{n}+e_n^2+\epsilon_{n,\mathrm{opt}}^2\right)\nonumber\\
&\qquad=O_p\left(n^{-\frac{2\beta}{2\beta+q}}\log^\frac{8\beta}{2\beta+q}n+e_n^2+\epsilon_{n,\mathrm{opt}}^2\right)\label{bound:opt-temp},
\end{align}
since $L\asymp\log n$ and $W\asymp n^{\frac{q}{2\beta+q}}\log^{1-\frac{4q}{2\beta+q}}n$. Set $\bar\epsilon_n=n^{-\frac{\beta}{2\beta+q}}\log^\frac{4\beta}{2\beta+q}n$. Then, the chosen $L,U,W$ satisfy $L\geq c\{\log(1/\bar\epsilon_n)+1\}$ and $U,W\geq c\bar\epsilon_n^{-q/\beta}\{\log(1/\bar\epsilon_n)+1\}$, where $c$ is the constant defined in Theorem 1 of \cite{yarotsky2017error}. Define $\mathcal F_\mathrm{DNN}^q(L,U,W)$ as the DNNs defined through \eqref{def:zkl} with $q$ input units, $L$ hidden layers, $U$ computation units and $W$ weights. Then, by Theorem 1 of \cite{yarotsky2017error}, we have
\begin{align*}
\inf_{f\in\mathcal F_\mathrm{DNN}^q(L,U,W)}\|f(\bX_Q)-f_Q^0(\bX_Q)\|_\infty\leq\bar\epsilon_n.
\end{align*}
Together with the assumption that $\|f^0(\bX)-f_Q^0(\bX_Q)\|_\infty\leq r_n$, we have
\begin{align*}
\inf_{f\in\mathcal F_\mathrm{DNN}^q(L,U,W)}\|f(\bX_Q)-f^0(\bX)\|_\infty\leq \bar\epsilon_n+r_n.
\end{align*}
For any $f\in\mathcal F_\mathrm{DNN}^q(L,U,W)$, we can construct a function $g\in\mathcal F_\mathrm{DNN}(L,U,W)$ such that, apart from containing additional $p-q$ unconnected input units, it shares the same architecture and parameters as $f$. The constructed $g\in\mathcal F_\mathrm{DNN}(L,U,W)$ satisfies $f(\bX_Q)=g(\bX)$. Hence,
\begin{align*}
\inf_{f\in\mathcal F_\mathrm{DNN}(L,U,W)}\|f(\bX)-f^0(\bX)\|_\infty&\leq\inf_{f\in\mathcal F_\mathrm{DNN}^q(L,U,W)}\|f(\bX_Q)-f^0(\bX)\|_\infty\\
&\leq\bar\epsilon_n+r_n=o(1).
\end{align*}
Since $\|Y^\#\|_\infty\leq M$, we have $\|f^0\|_\infty\leq M$, and hence
\begin{align*}
\inf_{f\in\mathcal F_\mathrm{DNN}(L,U,W):\|f\|_\infty>2M}\|f-f^0\|_\infty\geq M>\bar\epsilon_n+r_n,
\end{align*}
for large enough $N$. Therefore, we have
\begin{align*}
\epsilon_{n,\mathrm{opt}}&=\inf_{f\in\mathcal F_\mathrm{DNN}(L,U,W):\|f\|_\infty\leq2M}\|f-f^0\|_\infty\nonumber\\
&=\inf_{f\in\mathcal F_\mathrm{DNN}(L,U,W)}\|f-f^0\|_\infty\leq \bar\epsilon_n+r_n.
\end{align*}
Together with \eqref{bound:opt-temp} and the choice that $\bar\epsilon_n=n^{-\frac{\beta}{2\beta+q}}\log^\frac{4\beta}{2\beta+q}n$, we conclude that
\begin{align*}
E\left[R\{\fhat_\mathrm{opt}(\bX)-f^0(\bX)\}^2\right]=O_p\left(n^{-\frac{2\beta}{2\beta+q}}\log^\frac{8\beta}{2\beta+q}n+r_n^2+e_n^2\right).
\end{align*}
\end{proof}

\begin{proof}[Proof of Theorem \ref{thm:imp-MLP}]
Define the following approximation error
\begin{align*}
\epsilon_{n,{\mbox{\tiny MLP}}}:=\inf_{f\in\mathcal F_{\mbox{\tiny MLP}}(L,H):\|f\|_\infty\leq2M}\|f-f^0\|_\infty.
\end{align*}
Note that the multilayer perceptrons $\mathcal F_{\mbox{\tiny MLP}}(L,H)$ consist of $\bar W=(L-1)H^2+H(p+1)$ weights. 

Case (a): $p=O(n^\frac{q}{2\beta+2q}\log^\frac{3\beta-q}{\beta+q}n)$. By choosing $L\asymp\log n$ and $H\asymp n^\frac{q}{2\beta+2q}\log^\frac{2\beta-2q}{\beta+q}n$, we have $\bar W\asymp n^\frac{q}{\beta+q}\log^\frac{5\beta-3q}{\beta+q}n+pn^\frac{q}{2\beta+2q}\log^\frac{2\beta-2q}{\beta+q}n\asymp n^\frac{q}{\beta+q}\log^\frac{5\beta-3q}{\beta+q}n=o(n)$ since $p=O(n^\frac{q}{2\beta+2q}\log^\frac{3\beta-q}{\beta+q}n)$ and $\beta\geq1$. Repeating the proof of Theorem \ref{thm:imp-DNN}, we have
\begin{align}
&E\left[R\{\fhat_{\mbox{\tiny MLP}}(\bX)-f^0(\bX)\}^2\right]=O_p\left(\frac{\bar WL\log\bar W\log n}{n}+e_n^2+\epsilon_{n,{\mbox{\tiny MLP}}}^2\right)\nonumber\\
&\qquad=O_p\left(n^{-\frac{\beta}{\beta+q}}\log^\frac{8\beta}{\beta+q}n+e_n^2+\epsilon_{n,{\mbox{\tiny MLP}}}^2\right).\label{bound:MLP-temp}
\end{align}
By Lemma 1 of \cite{farrell2021deep},
\begin{align*}
\mathcal F_\mathrm{DNN}(L,U,W)\subseteq\mathcal F_{\mbox{\tiny MLP}}(L,WL+U).
\end{align*}
Set $\tilde\epsilon_n=n^{-\frac{\beta}{2\beta+2q}}\log^\frac{4\beta}{\beta+q}n$. Then, we have $\tilde\epsilon_n=o(1)$ since $\beta\geq1$. Choose $U=W=\lceil c\tilde\epsilon_n^{-q/\beta}\{\log(1/\tilde\epsilon_n)+1\}\rceil$, where $c$ is the constant defined in Theorem 1 of \cite{yarotsky2017error}. Since $H\geq\lceil c\tilde\epsilon_n^{-q/\beta}\{\log(1/\tilde\epsilon_n)+1\}\rceil(L+1)=WL+U$, we have
\begin{align}\label{embed}
\mathcal F_\mathrm{DNN}(L,U,W)\subseteq\mathcal F_{\mbox{\tiny MLP}}(L,WL+U)\subseteq\mathcal F_{\mbox{\tiny MLP}}(L,H).
\end{align}
Since $L\geq c\{\log(1/\tilde\epsilon_n)+1\}$ and $U,W\geq c\tilde\epsilon_n^{-q/\beta}\{\log(1/\tilde\epsilon_n)+1\}$, by Theorem 1 of \cite{yarotsky2017error}, we have
\begin{align*}
\inf_{f\in\mathcal F_\mathrm{DNN}^q(L,U,W)}\|f(\bX_Q)-f_Q^0(\bX_Q)\|_\infty\leq\tilde\epsilon_n.
\end{align*}
Together with the assumption that $\|f^0(\bX)-f_Q^0(\bX_Q)\|_\infty\leq r_n$, we have
\begin{align*}
\inf_{f\in\mathcal F_\mathrm{DNN}^q(L,U,W)}\|f(\bX_Q)-f^0(\bX)\|_\infty\leq \tilde\epsilon_n+r_n.
\end{align*}
For any $f\in\mathcal F_\mathrm{DNN}^q(L,U,W)$, as demonstrated in the proof of Theorem \ref{thm:imp-opt}, there exists some $g\in\mathcal F_\mathrm{DNN}(L,U,W)$ that satisfies $f(\bX_Q)=g(\bX)$. Hence,
\begin{align*}
\inf_{f\in\mathcal F_\mathrm{DNN}(L,U,W)}\|f(\bX)-f^0(\bX)\|_\infty&\leq\inf_{f\in\mathcal F_\mathrm{DNN}^q(L,U,W)}\|f(\bX_Q)-f^0(\bX)\|_\infty\\
&\leq \tilde\epsilon_n+r_n=o(1).
\end{align*}
Together with \eqref{embed}, we have
\begin{align*}
&\inf_{f\in\mathcal F_{\mbox{\tiny MLP}}(L,H)}\|f-f^0\|_\infty\leq\inf_{f\in\mathcal F_{\mbox{\tiny MLP}}(L,WL+U)}\|f-f^0\|_\infty\\
&\qquad\leq\inf_{f\in\mathcal F_\mathrm{DNN}(L,U,W)}\|f-f^0\|_\infty\leq \tilde\epsilon_n+r_n.
\end{align*}
Since $\|f^0\|_\infty\leq M$, we know that
\begin{align*}
\inf_{f\in\mathcal F_{\mbox{\tiny MLP}}(L,H):\|f\|_\infty>2M}\|f-f^0\|_\infty\geq M>\tilde\epsilon_n+r_n,
\end{align*}
for large enough $N$. Therefore, we have
\begin{align*}
\epsilon_{n,{\mbox{\tiny MLP}}}&=\inf_{f\in\mathcal F_{\mbox{\tiny MLP}}(L,H):\|f\|_\infty\leq2M}\|f-f^0\|_\infty\\
&=\inf_{f\in\mathcal F_{\mbox{\tiny MLP}}(L,H)}\|f-f^0\|_\infty\leq \tilde\epsilon_n+r_n.
\end{align*}
Together with \eqref{bound:MLP-temp} and the choice that $\tilde\epsilon_n=n^{-\frac{\beta}{2\beta+2q}}\log^\frac{4\beta}{\beta+q}n$, we conclude that
\begin{align*}
E\left[\{\fhat_{\mbox{\tiny MLP}}(\bX)-f^0(\bX)\}^2\right]=O_p\left(n^{-\frac{\beta}{\beta+q}}\log^\frac{8\beta}{\beta+q}n+r_n^2+e_n^2\right).
\end{align*}

Case (b): $p\gg n^\frac{q}{2\beta+2q}\log^\frac{3\beta-q}{\beta+q}n$ and $p=o(n/\log^5n)$. By choosing $L\asymp\log n$ and $H\asymp (n/p)^\frac{q}{2\beta+q}\log^\frac{4\beta-3q}{2\beta+q}n$, we have $\bar W\asymp (n/p)^\frac{2q}{2\beta+q}\log^\frac{10\beta-5q}{2\beta+q}n+n^\frac{q}{2\beta+q}p^\frac{2\beta}{2\beta+q}\log^\frac{4\beta-3q}{2\beta+q}n\asymp n^\frac{q}{\beta+q}\log^\frac{5\beta-3q}{\beta+q}n=o(n)$ since $p\gg n^\frac{q}{2\beta+2q}\log^\frac{3\beta-q}{\beta+q}n$ and $\beta\geq1$. Repeating the proof of Theorem \ref{thm:imp-DNN}, we have
\begin{align}
&E\left[R\{\fhat_{\mbox{\tiny MLP}}(\bX)-f^0(\bX)\}^2\right]=O_p\left(\frac{\bar WL\log\bar W\log n}{n}+e_n^2+\epsilon_{n,{\mbox{\tiny MLP}}}^2\right)\nonumber\\
&\qquad=O_p\left((p/n)^{\frac{2\beta}{2\beta+q}}\log^\frac{10\beta}{2\beta+q}n+e_n^2+\epsilon_{n,{\mbox{\tiny MLP}}}^2\right).\label{bound:MLP-temp2}
\end{align}
Set $\tilde\epsilon_n=(p/n)^\frac{\beta}{2\beta+q}\log^\frac{5\beta}{2\beta+q}n$. Then, we have $\tilde\epsilon_n=o(1)$ since $p=o(n/\log^5n)$. Repeating the proof of Case (a), we also have
\begin{align*}
\epsilon_{n,{\mbox{\tiny MLP}}}\leq\tilde\epsilon_n+r_n.
\end{align*}
Together with \eqref{bound:MLP-temp2} and the choice that $\tilde\epsilon_n=(p/n)^\frac{\beta}{2\beta+q}\log^\frac{5\beta}{2\beta+q}n$, we conclude that
\begin{align*}
E\left[R\{\fhat_{\mbox{\tiny MLP}}(\bX)-f^0(\bX)\}^2\right]=O_p\left((p/n)^{\frac{2\beta}{2\beta+q}}\log^\frac{10\beta}{2\beta+q}n+r_n^2+e_n^2\right).
\end{align*}

\end{proof}

\section{Proof of the results in Section \ref{sec:CATE}}
In this section, we denote
\begin{align*}
\Yhat&:=\muhat^{(2)}(1,\bS)+\frac{T\{Y-\muhat^{(2)}(1,\bS)\}}{\pihat^{(2)}(\bS)}-\muhat^{(2)}(0,\bS)-\frac{(1-T)\{Y-\muhat^{(2)}(0,\bS)\}}{1-\pihat^{(2)}(\bS)},\\
Y^\#&:=\mu^0(1,\bS)+\frac{T\{Y-\mu^0(1,\bS)\}}{\pi^0(\bS)}-\mu^0(0,\bS)-\frac{(1-T)\{Y-\mu^0(0,\bS)\}}{1-\pi^0(\bS)}.
\end{align*}
Then, we have the representation
\begin{align*}
\Yhat-Y^\#=\Delta_1+\Delta_2,
\end{align*}
where $\Delta_1=\Delta_{1,1}+\Delta_{1,2}+\Delta_{1,3}+\Delta_{1,4}$, $\Delta_2=\Delta_{2,1}+\Delta_{2,2}$, and
\begin{align}
\Delta_{1,1}&:=\left\{1-\frac{T}{\pi^0(\bS)}\right\}\{\muhat^{(2)}(1,\bS)-\mu^0(1,\bS)\},\label{def:Delta11}\\
\Delta_{1,2}&:=\left\{\frac{T}{\pihat^{(2)}(\bS)}-\frac{T}{\pi^0(\bS)}\right\}\{Y(1)-\mu^0(1,\bS)\},\label{def:Delta12}\\
\Delta_{1,3}&:=-\left\{1-\frac{1-T}{1-\pi^0(\bS)}\right\}\{\muhat^{(2)}(0,\bS)-\mu^0(0,\bS)\},\label{def:Delta13}\\
\Delta_{1,4}&:=-\left\{\frac{1-T}{1-\pihat^{(2)}(\bS)}-\frac{1-T}{1-\pi^0(\bS)}\right\}\{Y(0)-\mu^0(0,\bS)\},\label{def:Delta14}\\
\Delta_{2,1}&:=\left\{\frac{T}{\pihat^{(2)}(\bS)}-\frac{T}{\pi^0(\bS)}\right\}\{\muhat^{(2)}(1,\bS)-\mu^0(1,\bS)\},\label{def:Delta21}\\
\Delta_{2,2}&:=-\left\{\frac{1-T}{1-\pihat^{(2)}(\bS)}-\frac{1-T}{1-\pi^0(\bS)}\right\}\{\muhat^{(2)}(0,\bS)-\mu^0(0,\bS)\}.\label{def:Delta22}
\end{align}

\begin{proof}[Proof of Lemma \ref{lemma:imput-err-CATE}]
Under Assumption \ref{cond:ign},
\begin{align}
&E_\bZ(\Delta_{1,1}\mid\bS)=E_\bZ\left[\left\{1-\frac{T}{\pi^0(\bS)}\right\}\{\muhat^{(2)}(1,\bS)-\mu^0(1,\bS)\}\mid\bS\right]=0,\label{eq:ETR1-CATE}\\
&E_\bZ(\Delta_{1,2}\mid\bS)=E_\bZ\left[\left\{\frac{T}{\pihat^{(2)}(\bS)}-\frac{T}{\pi^0(\bS)}\right\}\{Y(1)-\mu^0(1,\bS)\}\mid\bS\right]\nonumber\\
&\qquad=E_\bZ\left\{\frac{T}{\pihat^{(2)}(\bS)}-\frac{T}{\pi^0(\bS)}\mid\bS\right\}E_\bZ\{Y(1)-\mu^0(1,\bS)\mid\bS\}=0.\label{eq:ETR2}
\end{align}
Similarly, we also have $E_\bZ(\Delta_{1,3}\mid\bS)=0$ and $E_\bZ(\Delta_{1,4}\mid\bS)=0$. Therefore, we conclude that
$$E_\bZ(\Delta_1\mid\bS)=0\;\;\mbox{almost surely}.$$
Additionally,
\begin{align*}
E_\bZ(\Delta_{2,1}^2)&\leq\|1/\pi^0\|_\infty^2\|1/\pihat^{(2)}\|_{\infty,P_\bZ}^2E\left[T\left\{\pihat^{(2)}(\bS)-\pi^0(\bS)\right\}^2\left\{\muhat^{(2)}(1,\bS)-\mu^0(1,\bS)\right\}^2\right]\\
&=O_p(\delta_N^4),\\
E_\bZ(\Delta_{2,2}^2)&\leq\|1/(1-\pi^0)\|_\infty^2\|1/(1-\pihat^{(2)})\|_{\infty,P_\bZ}^2\\
&\qquad\cdot E\left[(1-T)\left\{\pihat^{(2)}(\bS)-\pi^0(\bS)\right\}^2\left\{\muhat^{(2)}(0,\bS)-\mu^0(0,\bS)\right\}^2\right]\\
&=O_p(\delta_N^4),
\end{align*}
under Assumptions \ref{cond:ign}, \ref{cond:converge-rate-CATE}, and \ref{cond:bound-CATE}. Note that $\Delta_{2,1}\Delta_{2,2}=0$, we conclude that
\begin{align*}
E_\bZ(\Delta_2^2)=E_\bZ(\Delta_{2,1}^2)+E_\bZ(\Delta_{2,2}^2)=O_p(\delta_N^4).
\end{align*}
\end{proof}

\begin{proof}[Proof of Theorem \ref{thm:theta-reg}]
We first show that
$$\theta_{\mbox{\tiny CATE}}(\bs)=E(Y^\#\mid\bS=\bs),\;\;\forall\bs\in\R^d.$$
Under Assumption \ref{cond:ign},
\begin{align*}
&E(Y^\#\mid\bS)=\theta_{\mbox{\tiny CATE}}(\bS)+E\left[\frac{T\{Y-\mu^0(1,\bS)\}}{\pi^0(\bS)}-\frac{(1-T)\{Y-\mu^0(0,\bS)\}}{1-\pi^0(\bS)}\mid\bS\right]\\
&\;\;=\theta_{\mbox{\tiny CATE}}(\bS)+\frac{E(T\mid\bS)E\{Y(1)-\mu^0(1,\bS)\mid\bS\}}{\pi^0(\bS)}-\frac{E(1-T\mid\bS)E\{Y(0)-\mu^0(0,\bS)\mid\bS\}}{1-\pi^0(\bS)}\\
&\;\;=\theta_{\mbox{\tiny CATE}}(\bS).
\end{align*}
Since $\|Y\|_\infty\leq M$ under Assumption \ref{cond:bound-CATE}, the conditional mean function also satisfies $\|\mu^0\|_\infty\leq M$. Hence, under Assumptions \ref{cond:ign} and \ref{cond:bound-CATE},
\begin{align*}
\|Y^\#\|_\infty&=\left\|\mu^0(1,\bS)+\frac{T\{Y-\mu^0(1,\bS)\}}{\pi^0(\bS)}-\mu^0(0,\bS)-\frac{(1-T)\{Y-\mu^0(0,\bS)\}}{1-\pi^0(\bS)}\right\|_\infty\\
&\leq2\|\mu^0\|_\infty+(\|1/\pi^0\|_\infty+\|1/(1-\pi^0)\|_\infty)\left(\|Y\|_\infty+\|\mu^0\|_\infty\right)=O(1),\\
\|\Delta_{1,1}\|_{\infty,P_\bZ}&=\left\|\left\{1-\frac{T}{\pi^0(\bS)}\right\}\{\muhat^{(2)}(1,\bS)-\mu^0(1,\bS)\}\right\|_{\infty,P_\bZ}\\
&\leq(1+\|1/\pi^0\|_\infty)(\|\muhat^{(2)}\|_{\infty,P_\bZ}+\|\mu^0\|_\infty)=O_p(1),\\
\|\Delta_{1,2}\|_{\infty,P_\bZ}&=\left\|\left\{\frac{T}{\pihat^{(2)}(\bS)}-\frac{T}{\pi^0(\bS)}\right\}\{Y(1)-\mu^0(1,\bS)\}\right\|_{\infty,P_\bZ}\\
&\leq(\|1/\pihat^{(2)}\|_{\infty,P_\bZ}+\|1/\pi^0\|_\infty)(\|Y\|_\infty+\|\mu^0\|_\infty)=O_p(1),\\
\|\Delta_{1,3}\|_{\infty,P_\bZ}&=\left\|-\left\{1-\frac{1-T}{1-\pi^0(\bS)}\right\}\{\muhat^{(2)}(0,\bS)-\mu^0(0,\bS)\}\right\|_{\infty,P_\bZ}\\
&\leq\{1+\|1/(1-\pi^0)\|_\infty\}(\|\muhat^{(2)}\|_{\infty,P_\bZ}+\|\mu^0\|_\infty)=O_p(1),\\
\|\Delta_{1,4}\|_{\infty,P_\bZ}&=\left\|-\left\{\frac{1-T}{1-\pihat^{(2)}(\bS)}-\frac{1-T}{1-\pi^0(\bS)}\right\}\{Y(0)-\mu^0(0,\bS)\}\right\|_{\infty,P_\bZ}\\
&\leq(\|1/(1-\pihat^{(2)})\|_{\infty,P_\bZ}+\|1/(1-\pi^0)\|_\infty)(\|Y\|_\infty+\|\mu^0\|_\infty)=O_p(1).
\end{align*}
It follows that
$$\|\Delta_1\|_{\infty,P_\bZ}=O_p(1).$$
By Theorem \ref{thm:imp-MLP} (with $R\equiv1$) and Lemma \ref{lemma:imput-err-CATE}, as $N\to\infty$,
\begin{align*}
&E_\bZ\left[\{\thetahat_{\mbox{\tiny CATE}}^1(\bS)-\theta_{\mbox{\tiny CATE}}(\bS)\}^2\right]\nonumber\\
&\qquad=O_p\left(N^{-\frac{2\beta_\theta}{2\beta_\theta+2q_\theta}}\log^\frac{8\beta_\theta}{\beta_\theta+q_\theta}N+(d/N)^{\frac{2\beta_\theta}{2\beta_\theta+q_\theta}}\log^\frac{10\beta_\theta}{2\beta_\theta+q_\theta}N+r_N^2+\delta_N^4\right).
\end{align*}
Repeating the same procedure, we also have the same consistency rate for $\thetahat_{\mbox{\tiny CATE}}^2$. Therefore, $\thetahat_{\mbox{\tiny CATE}}=(\thetahat_{\mbox{\tiny CATE}}^1+\thetahat_{\mbox{\tiny CATE}}^2)/2$ satisfies \eqref{rate:thetahat}.
\end{proof}

\section{Proof of the results in Section \ref{sec:DTE}}

\subsection{Auxiliary lemmas}
\begin{lemma}\label{lemma:gen}
Let Assumptions \ref{cond:seq-ign}, \ref{cond:converge-rate} and \ref{cond:bound} hold. In addition, assume
\begin{align}\label{def:bN}
E\left[T_1\left\{\muhat(\bS_1)-\mu^0(\bS_1)\right\}^2\right]=O_p(b_N^2)
\end{align}
with some sequence $b_N\geq0$. Then, as $N\to\infty$,
\begin{align*}
\thetahat-\theta=O_p\left(a_Nd_N+b_Nc_N+N^{-1/2}\right).
\end{align*}
\end{lemma}

\begin{lemma}\label{lemma:normal}
Let Assumptions \ref{cond:seq-ign}, \ref{cond:converge-rate} and \ref{cond:bound} hold. In addition, assume \eqref{def:bN} holds with some sequence $b_N\geq0$, $a_Nd_N+b_Nc_N=o(N^{-1/2})$ and $a_N,b_N,c_N,d_N=o(1)$. Then, as $N\to\infty$,
\begin{align*}
\sqrt N \sigma^{-1}(\thetahat-\theta)\to N(0,1)
\end{align*}
in distribution, provided that $\sigma^2=\Var\{\psi(\bZ;\eta^0)\}>c>0$.
\end{lemma}

\subsection{Proof of the auxiliary lemmas}

\begin{proof}[Proof of Lemma \ref{lemma:gen}]
By Lemma S.8 of \cite{bradic2024high},
\begin{align*}
\theta=E\{\psi(\bZ_i;\eta^0)\}.
\end{align*}
Hence, we have
\begin{align}\label{eq:rep_Wk}
\thetahat^{(k)}-\theta=|I_k|^{-1}\sum_{i\in I_k}\psi(\bZ_i;\etahat_{-k})-\theta=W_{k,1}+W_{k,2},
\end{align}
where
\begin{align}
W_{k,1}&=|I_k|^{-1}\sum_{i\in I_k}\psi(\bZ_i;\eta^0)-E\{\psi(\bZ_i;\eta^0)\},\\
W_{k,2}&=|I_k|^{-1}\sum_{i\in I_k}\left\{\psi(\bZ_i;\etahat_{-k})-\psi(\bZ_i;\eta^0)\right\}.
\end{align}
By Lemma S.10 of \cite{bradic2024high},
\begin{align}
W_{k,1}=O_p\left(N^{-1/2}\right).
\end{align}
In addition, since $\psi(\bZ_i;\etahat_{-k})-\psi(\bZ_i;\eta^0)$, $i\in I_k$, are independent and identically distributed conditional on $\S_{-k}$, 
\begin{align*}
&E_{\S_k}(W_{k,2})=E_{\S_k}\left\{\psi(\bZ_i;\etahat_{-k})-\psi(\bZ_i;\eta^0)\right\}\\
&\quad=E_{\S_k}\left[\frac{T_{1i}T_{2i}\{Y_i-\nuhat_{-k}(\bSbar_{2i})\}}{\pihat_{-k}(\bS_{1i})\rhohat_{-k}(\bSbar_{2i})}-\frac{T_{1i}T_{2i}\{Y_i-\nu^0(\bSbar_{2i})\}}{\pi^0(\bS_{1i})\rho^0(\bSbar_{2i})}+\frac{T_{1i}\{\nuhat_{-k}(\bSbar_{2i})-\muhat_{-k}(\bSbar_{1i})\}}{\pihat_{-k}(\bS_{1i})}\right]\\
&\quad\qquad+E_{\S_k}\left[\muhat_{-k}(\bS_{1i})-\frac{T_{1i}\{\nu^0(\bSbar_{2i})-\mu^0(\bSbar_{1i})\}}{\pi^0(\bS_{1i})}-\mu^0(\bS_{1i})\right]\\
&\quad\overset{(i)}{=}E_{\S_k}\left[\frac{T_{1i}\rho^0(\bSbar_{2i})\{\nu^0(\bSbar_{2i})-\nuhat_{-k}(\bSbar_{2i})\}}{\pihat_{-k}(\bS_{1i})\rhohat_{-k}(\bSbar_{2i})}+\frac{T_{1i}\{\nuhat_{-k}(\bSbar_{2i})-\muhat_{-k}(\bSbar_{1i})\}}{\pihat_{-k}(\bS_{1i})}\right]\\
&\quad\qquad+E_{\S_k}\left[\muhat_{-k}(\bS_{1i})-\frac{T_{1i}\{\nu^0(\bSbar_{2i})-\mu^0(\bSbar_{1i})\}}{\pi^0(\bS_{1i})}-\mu^0(\bS_{1i})\right]\\
&\quad\overset{(ii)}{=}R_{k,1}+R_{k,2}+R_{k,3},
\end{align*}
where (i) holds by the tower rule under Assumption \ref{cond:seq-ign}; (ii) holds through rearranging and $R_{k,1},R_{k,2},R_{k,3}$ are defined as
\begin{align*}
R_{k,1}&:=E_{\S_k}\left[\left\{1-\frac{T_{1i}}{\pihat_{-k}(\bS_{1i})}\right\}\left\{\muhat_{-k}(\bS_{1i})-\mu^0(\bS_{1i})\right\}\right],\\
R_{k,2}&:=E_{\S_k}\left[\frac{T_{1i}}{\pihat_{-k}(\bS_{1i})}\left\{1-\frac{\rho^0(\bSbar_{2i})}{\rhohat_{-k}(\bSbar_{2i})}\right\}\left\{\nuhat_{-k}(\bSbar_{2i})-\nu^0(\bSbar_{2i})\right\}\right],\\
R_{k,3}&:=E_{\S_k}\left[\left\{\frac{T_{1i}}{\pihat_{-k}(\bS_{1i})}-\frac{T_{1i}}{\pi^0(\bS_{1i})}\right\}\left\{\nu^0(\bSbar_{2i})-\mu^0(\bS_{1i})\right\}\right].
\end{align*}
By the tower rule, we have
\begin{align*}
R_{k,1}&=E_{\S_k}\left[\frac{T_{1i}}{\pihat_{-k}(\bS_{1i})\pi^0(\bS_{1i})}\left\{\pihat_{-k}(\bS_{1i})-\pi^0(\bS_{1i})\right\}\left\{\muhat_{-k}(\bS_{1i})-\mu^0(\bS_{1i})\right\}\right]\\
&\overset{(i)}{\leq}M^2\left(E_{\S_k}\left\{\pihat_{-k}(\bS_{1i})-\pi^0(\bS_{1i})\right\}^2E_{\S_k}\left[T_{1i}\left\{\muhat_{-k}(\bS_{1i})-\mu^0(\bS_{1i})\right\}^2\right]\right)^{1/2}\\
&=O_p(b_Nc_N),
\end{align*}
where (i) holds by the Cauchy-Schwarz inequality and under Assumption \ref{cond:seq-ign}. Similarly, by the tower rule,
\begin{align*}
R_{k,2}&=E_{\S_k}\left[\frac{T_{1i}T_{2i}}{\pihat_{-k}(\bS_{1i})\rhohat_{-k}(\bSbar_{2i})\rho^0(\bSbar_{2i})}\left\{\rhohat_{-k}(\bSbar_{2i})-\rho^0(\bSbar_{2i})\right\}\left\{\nuhat_{-k}(\bSbar_{2i})-\nu^0(\bSbar_{2i})\right\}\right]\\
&\overset{(i)}{\leq}M^3\left(E_{\S_k}\left[T_{1i}\left\{\rhohat_{-k}(\bSbar_{2i})-\rho^0(\bSbar_{2i})\right\}^2\right]E_{\S_k}\left[T_{1i}T_{2i}\left\{\nuhat_{-k}(\bSbar_{2i})-\nu^0(\bSbar_{2i})\right\}^2\right]\right)^{1/2}\\
&=O_p(a_Nd_N),
\end{align*}
where (i) holds by the Cauchy-Schwarz inequality and under Assumption \ref{cond:seq-ign}. In addition, we notice that
\begin{align*}
&E\{\nu^0(\bSbar_2)\mid\bS_1,T_1=1\}\overset{(i)}{=}E[E\{Y(1,1)\mid\bSbar_2,T_1=T_2=1\}\mid\bS_1,T_1=1]\\
&\qquad\overset{(ii)}{=}E[E\{Y(1,1)\mid\bSbar_2,T_1=1\}\mid\bS_1,T_1=1]\overset{(iii)}{=}E\{Y(1,1)\mid\bS_1,T_1=1\}\\
&\qquad\overset{(iv)}{=}E\{Y(1,1)\mid\bS_1\}\overset{(v)}{=}\mu^0(\bS_1),
\end{align*}
where (i) and (v) hold by the definitions of $\nu^0(\cdot)$ and $\mu^0(\cdot)$; (ii) and (iv) holds under Assumption \ref{cond:seq-ign}; (iii) holds by the tower rule. Therefore, using the tower rule, we also have
\begin{align*}
R_{k,3}&=E_{\S_k}\left(E_{\S_k}\left[\left\{\frac{1}{\pihat_{-k}(\bS_{1i})}-\frac{1}{\pi^0(\bS_{1i})}\right\}\left\{\nu^0(\bSbar_{2i})-\mu^0(\bS_{1i})\right\}\mid\bS_{1i},T_{1i}=1\right]\pi^0(\bS_{1i})\right)\\
&=E_{\S_k}\left[\left\{\frac{1}{\pihat_{-k}(\bS_{1i})}-\frac{1}{\pi^0(\bS_{1i})}\right\}\left\{\mu^0(\bS_{1i})-\mu^0(\bS_{1i})\right\}\pi^0(\bS_{1i})\right]=0.
\end{align*}
To sum up, we have
\begin{align}\label{bound:bias}
E_{\S_k}(W_{k,2})=R_{k,1}+R_{k,2}+R_{k,3}=O_p(a_Nd_N+b_Nc_N).
\end{align}
On the other hand, we note that
\begin{align*}
&\|\psi(\bZ_i;\etahat_{-k})\|_{\infty,P_{\S_k}}\\
&\quad=\left\|\frac{T_{1i}T_{2i}\{Y_i-\nuhat_{-k}(\bSbar_{2i})\}}{\pihat_{-k}(\bS_{1i})\rhohat_{-k}(\bSbar_{2i})}+\frac{T_{1i}\{\nuhat_{-k}(\bSbar_{2i})-\muhat_{-k}(\bSbar_{1i})\}}{\pihat_{-k}(\bS_{1i})}+\muhat_{-k}(\bS_{1i})\right\|_{\infty,P_{\S_k}}=O_p(1),
\end{align*}
and
\begin{align}
\|\psi(\bZ_i;\eta^0)\|_\infty&=\left|\frac{T_{1i}T_{2i}\{Y_i-\nu^0(\bSbar_{2i})\}}{\pi^0(\bS_{1i})\rho^0(\bSbar_{2i})}+\frac{T_{1i}\{\nu^0(\bSbar_{2i})-\mu^0(\bSbar_{1i})\}}{\pi^0(\bS_{1i})}+\mu^0(\bS_{1i})\right|\nonumber\\
&=O(1)\label{bound:psi}
\end{align}
under Assumptions \ref{cond:seq-ign} and \ref{cond:bound}. Hence,
\begin{align*}
&\Var_{\S_k}(W_{k,2})=|I_k|^{-1}\Var_{\S_k}\left\{\psi(\bZ_i;\etahat_{-k})-\psi(\bZ_i;\eta^0)\right\}\\
&\qquad\leq|I_k|^{-1}E_{\S_k}\left\{\psi(\bZ_i;\etahat_{-k})-\psi(\bZ_i;\eta^0)\right\}^2=O_p(|I_k|^{-1})=O_p\left(\frac{1}{N}\right).
\end{align*}
By Chebyshev's inequality,
\begin{align*}
W_{k,2}=O_p\left(a_Nd_N+b_Nc_N+N^{-1/2}\right),
\end{align*}
and hence
\begin{align*}
\thetahat^{(k)}-\theta=\Delta_{1,k}+\Delta_{2,k}=O_p\left(a_Nd_N+b_Nc_N+N^{-1/2}\right).
\end{align*}
Finally, for fixed $K>0$, we conclude that
\begin{align*}
\thetahat-\theta=K^{-1}\sum_{k=1}^K(\thetahat^{(k)}-\theta)=O_p\left(a_Nd_N+b_Nc_N+N^{-1/2}\right).
\end{align*}
\end{proof}

\begin{proof}[Proof of Lemma \ref{lemma:normal}]
Observe that
\begin{align*}
\psi(\bZ_i;\etahat_{-k})-\psi(\bZ_i;\eta^0)=\sum_{j=1}^7\bar R_{k,j,i},
\end{align*}
where
\begin{align*}
\bar R_{k,1,i}&:=-\frac{T_{1i}T_{2i}Y_i\{\rhohat_{-k}(\bSbar_{2i})-\rho^0(\bSbar_{2i})\}}{\pihat_{-k}(\bS_{1i})\rhohat_{-k}(\bSbar_{2i})\rho^0(\bSbar_{2i})},\\
\bar R_{k,2,i}&:=-\frac{T_{1i}T_{2i}Y_i\{\pihat_{-k}(\bS_{1i})-\pi^0(\bS_{1i})\}}{\rho^0(\bSbar_{2i})\pihat_{-k}(\bS_{1i})\pi^0(\bSbar_{2i})},\\
\bar R_{k,3,i}&:=\left\{1-\frac{T_{2i}}{\rhohat_{-k}(\bSbar_{2i})}\right\}\frac{T_{1i}\{\nuhat_{-k}(\bSbar_{2i})-\nu^0(\bSbar_{2i})\}}{\pihat_{-k}(\bS_{1i})},\\
\bar R_{k,4,i}&:=-\left\{1-\frac{T_{2i}}{\rhohat_{-k}(\bSbar_{2i})}\right\}\frac{T_{1i}\nu^0(\bSbar_{2i})\{\pihat_{-k}(\bS_{1i})-\pi^0(\bS_{1i})\}}{\pihat_{-k}(\bS_{1i})\pi^0(\bS_{1i})},\\
\bar R_{k,5,i}&:=\frac{T_{1i}T_{2i}\nu^0(\bSbar_{2i})\{\rhohat_{-k}(\bSbar_{2i})-\rho^0(\bSbar_{2i})\}}{\pi^0(\bS_{1i})\rho^0(\bSbar_{2i})\rhohat_{-k}(\bSbar_{2i})},\\
\bar R_{k,6,i}&:=\left\{1-\frac{T_{1i}}{\pihat_{-k}(\bS_{1i})}\right\}\{\muhat_{-k}(\bS_{1i})-\mu^0(\bS_{1i})\},\\
\bar R_{k,7,i}&:=\frac{T_{1i}\mu^0(\bS_{1i})\{\pihat_{-k}(\bS_{1i})-\pi^0(\bS_{1i})\}}{\pi^0(\bS_{1i})\pihat_{-k}(\bS_{1i})}.
\end{align*}
By the tower rule,
\begin{align*}
&E_{S_k}\left[T_{1i}T_{2i}\{\nuhat_{-k}(\bSbar_{2i})-\nu^0(\bSbar_{2i})\}^2\right]=E_{S_k}\left[T_{1i}\rho^0(\bSbar_{2i})\{\nuhat_{-k}(\bSbar_{2i})-\nu^0(\bSbar_{2i})\}^2\right]\\
&\qquad\geq M^{-1}E_{S_k}\left[T_{1i}\{\nuhat_{-k}(\bSbar_{2i})-\nu^0(\bSbar_{2i})\}^2\right],
\end{align*}
under Assumption \ref{cond:seq-ign}. Similarly, we also have
\begin{align*}
&E_{S_k}\left[T_{1i}\{\muhat_{-k}(\bS_{1i})-\mu^0(\bS_{1i})\}^2\right]=E_{S_k}\left[T_{1i}\pi^0(\bS_{1i})\{\muhat_{-k}(\bS_{1i})-\mu^0(\bS_{1i})\}^2\right]\\
&\qquad\geq M^{-1}E_{S_k}\left[\{\muhat_{-k}(\bS_{1i})-\mu^0(\bS_{1i})\}^2\right].
\end{align*}
Therefore,
\begin{align*}
E_{S_k}\left[T_{1i}\{\nuhat_{-k}(\bSbar_{2i})-\nu^0(\bSbar_{2i})\}^2\right]&\leq ME_{S_k}\left[T_{1i}T_{2i}\{\nuhat_{-k}(\bSbar_{2i})-\nu^0(\bSbar_{2i})\}^2\right]=O_p(a_N^2),\\
E_{S_k}\left[\{\muhat_{-k}(\bS_{1i})-\mu^0(\bS_{1i})\}^2\right]&\leq ME_{S_k}\left[T_{1i}\{\muhat_{-k}(\bS_{1i})-\mu^0(\bS_{1i})\}^2\right]=O_p(b_N^2).
\end{align*}
When \eqref{def:bN} holds, under Assumptions \ref{cond:seq-ign}, \ref{cond:converge-rate} and \ref{cond:bound}, we have
\begin{align*}
E_{\S_k}(\bar R_{k,1,i}^2)&=O_p\left(E_{\S_k}\left[T_{1i}\{\rhohat_{-k}(\bSbar_{2i})-\rho^0(\bSbar_{2i})\}^2\right]\right)=O_p(d_N^2),\\
E_{\S_k}(\bar R_{k,2,i}^2)&=O_p\left(E_{\S_k}\left[\{\pihat_{-k}(\bS_{1i})-\pi^0(\bS_{1i})\}^2\right]\right)=O_p(c_N^2),\\
E_{\S_k}(\bar R_{k,3,i}^2)&=O_p\left(E_{\S_k}\left[T_{1i}\{\nuhat_{-k}(\bSbar_{2i})-\nu^0(\bSbar_{2i})\}^2\right]\right)=O_p(a_N^2),\\
E_{\S_k}(\bar R_{k,4,i}^2)&=O_p\left(E_{\S_k}\left[\{\pihat_{-k}(\bS_{1i})-\pi^0(\bS_{1i})\}^2\right]\right)=O_p(c_N^2),\\
E_{\S_k}(\bar R_{k,5,i}^2)&=O_p\left(E_{\S_k}\left[T_{1i}\{\rhohat_{-k}(\bSbar_{2i})-\rho^0(\bSbar_{2i})\}^2\right]\right)=O_p(d_N^2),\\
E_{\S_k}(\bar R_{k,6,i}^2)&=O_p\left(E_{\S_k}\left[\{\muhat_{-k}(\bS_{1i})-\mu^0(\bS_{1i})\}^2\right]\right)=O_p(b_N^2),\\
E_{\S_k}(\bar R_{k,7,i}^2)&=O_p\left(E_{\S_k}\left[\{\pihat_{-k}(\bS_{1i})-\pi^0(\bS_{1i})\}^2\right]\right)=O_p(c_N^2).
\end{align*}
Recall the representation \eqref{eq:rep_Wk}, now we have
\begin{align}
&\Var_{\S_k}(W_{k,2})\leq|I_k|^{-1}E_{\S_k}\left\{\psi(\bZ_i;\etahat_{-k})-\psi(\bZ_i;\eta^0)\right\}^2\nonumber\\
&\qquad\leq7|I_k|^{-1}\sum_{j=1}^7E_{\S_k}(\bar R_{k,j,i}^2)=O_p\left(\frac{a_N^2+b_N^2+c_N^2+d_N^2}{N}\right).\label{bound:psi-diff}
\end{align}
Together with \eqref{bound:bias}, by Chebyshev's inequality,
\begin{align*}
W_{k,2}=O_p\left(a_Nd_N+b_Nc_N+\frac{a_N+b_N+c_N+d_N}{\sqrt N}\right)=o_p\left(N^{-1/2}\right)
\end{align*}
when $a_Nd_N+b_Nc_N=o_p(N^{-1/2})$ and $a_N+b_N+c_N+d_N=o_p(1)$. Hence,
\begin{align*}
&\thetahat-\theta=K^{-1}\sum_{k=1}^K(\thetahat^{(k)}-\theta)=K^{-1}\sum_{k=1}^K(W_{k,1}+W_{k,2})\\
&\qquad=N^{-1}\sum_{i=1}^N\psi(\bZ_i;\eta^0)-\theta+o_p\left(N^{-1/2}\right).
\end{align*}
Since \eqref{bound:psi} holds and $\sigma>c>0$, the Lyapunov's condition 
\begin{align*}
\frac{E|\psi(\bZ;\eta^0)-\theta|^{2+\delta}}{N^{\delta/2}\sigma^{2+\delta}}=o(1)
\end{align*}
holds with any $\delta>0$. By Lyapunov's central limit theorem,
\begin{align*}
\sqrt N \sigma^{-1}\left\{N^{-1}\sum_{i=1}^N\psi(\bZ_i;\eta^0)-\theta\right\}\to N(0,1),
\end{align*}
and hence
\begin{align*}
\sqrt N \sigma^{-1}(\thetahat-\theta)=\sqrt N \sigma^{-1}\left\{N^{-1}\sum_{i=1}^N\psi(\bZ_i;\eta^0)-\theta\right\}+o_p(1)\to N(0,1).
\end{align*}
\end{proof}

\subsection{Proof of the main results}

In the following, we denote
\begin{align*}
Y^\#&:=\nu^0(\bSbar_2)+\frac{T_2\{Y(1,1)-\nu^0(\bSbar_2)\}}{\rho^0(\bSbar_2)},\\
\Yhat&:=\nuhat^{(2)}(\bSbar_2)+\frac{T_2\{Y(1,1)-\nuhat^{(2)}(\bSbar_2)\}}{\rhohat^{(2)}(\bSbar_2)}.
\end{align*}
Consider the following representation:
\begin{align*}
\Yhat-Y^\#=\Delta_1+\Delta_2,
\end{align*}
where $\Delta_1=\Delta_{1,1}+\Delta_{1,2}$ and
\begin{align*}
\Delta_{1,1}&:=\left\{1-\frac{T_2}{\rho^0(\bSbar_2)}\right\}\{\nuhat^{(2)}(\bSbar_2)-\nu^0(\bSbar_2)\},\\
\Delta_{1,2}&:=\left\{\frac{T_2}{\rhohat^{(2)}(\bSbar_2)}-\frac{T_2}{\rho^0(\bSbar_2)}\right\}\{Y(1,1)-\nu^0(\bSbar_2)\},\\
\Delta_2&:=\left\{\frac{T_2}{\rhohat^{(2)}(\bSbar_2)}-\frac{T_2}{\rho^0(\bSbar_2)}\right\}\{\nuhat^{(2)}(\bSbar_2)-\nu^0(\bSbar_2)\}.
\end{align*}

\begin{proof}[Proof of Lemma \ref{lemma:imput-err}]
By the law of total expectation,
\begin{align}
&E_\bZ(T_1\Delta_{1,1}\mid\bS_1)\nonumber\\
&\qquad=E_\bZ\left(E_\bZ\left[T_1\left\{1-\frac{T_2}{\rho^0(\bSbar_2)}\right\}\{\nuhat^{(2)}(\bSbar_2)-\nu^0(\bSbar_2)\}\mid\bSbar_2,T_1\right]\mid\bS_1\right)\nonumber\\
&\qquad=E_\bZ\left(E_\bZ\left[\left\{1-\frac{T_2}{\rho^0(\bSbar_2)}\right\}\mid\bSbar_2,T_1=1\right]\{\nuhat^{(2)}(\bSbar_2)-\nu^0(\bSbar_2)\}E_\bZ (T_1\mid\bSbar_2)\mid\bS_1\right)\nonumber\\
&\qquad=E_\bZ\left[\left\{1-\frac{\rho^0(\bSbar_2)}{\rho^0(\bSbar_2)}\right\}\{\nuhat^{(2)}(\bSbar_2)-\nu^0(\bSbar_2)\}E_\bZ (T_1\mid\bSbar_2)\mid\bS_1\right]=0.\label{eq:ETR1}
\end{align}
Similarly, let $\xi:=Y(1,1)-\nu^0(\bSbar_2)$, then
\begin{align}
&E_\bZ(T_1\Delta_{1,2}\mid\bS_1)\nonumber\\
&\qquad=E_\bZ\left(E_\bZ\left[T_1\left\{\frac{T_2}{\rhohat^{(2)}(\bSbar_2)}-\frac{T_2}{\rho^0(\bSbar_2)}\right\}\{Y(1,1)-\nu^0(\bSbar_2)\}\mid\bSbar_2,T_1\right]\mid\bS_1\right)\nonumber\\
&\qquad=E_\bZ\left(\left\{\frac{1}{\rhohat^{(2)}(\bSbar_2)}-\frac{1}{\rho^0(\bSbar_2)}\right\}\{\nu^0(\bSbar_2)-\nu^0(\bSbar_2)\}E_\bZ (T_1T_2\mid\bSbar_2)\mid\bS_1\right)=0.\label{eq:ETR2}
\end{align}
By \eqref{eq:ETR1} and \eqref{eq:ETR2},
\begin{align*}
E_\bZ(T_1\Delta_1\mid\bS_1)=E_\bZ (T_1\Delta_{1,1}\mid\bS_1)+E_\bZ (T_1\Delta_{1,2}\mid\bS_1)=0.
\end{align*}
Note that
\begin{align*}
E_\bZ(T_1\Delta_1\mid\bS_1)&=E_\bZ (T_1\Delta_1\mid\bS_1,T_1=1)P(T_1=1\mid\bS_1)\\
&\qquad+E_\bZ(T_1\Delta_1\mid\bS_1,T_1=0)P(T_1=0\mid\bS_1)\\
&=E_\bZ(T_1\Delta_1\mid\bS_1,T_1=1)\pi(\bS_1)
\end{align*}
and $\|1/\pi^0\|_\infty\leq M$ under Assumption \ref{cond:seq-ign}, hence
\begin{align*}
E_\bZ(\Delta_1\mid\bS_1,T_1=1)=0\;\;\text{almost surely}.
\end{align*}
Additionally,
\begin{align*}
&E_\bZ(T_1\Delta_2^2)\\
&\quad\leq\|1/\rho^0\|_\infty^2\|1/\rhohat^{(2)}\|_{\infty,P_\bZ}^2E\left[T_1T_2\left\{\rhohat^{(2)}(\bSbar_2)-\rho^0(\bSbar_2)\right\}^2\left\{\nuhat^{(2)}(\bSbar_2)-\nu^0(\bSbar_2)\right\}^2\right]\\
&\quad=O_p(\delta_N^4),
\end{align*}
under Assumptions \ref{cond:seq-ign}, \ref{cond:converge-rate}, and \ref{cond:bound}. Note that
\begin{align*}
E_\bZ(T_1\Delta_2^2)&=E_\bZ (T_1\Delta_2^2\mid T_1=1)P(T_1=1)+E_\bZ (T_1\Delta_2^2\mid T_1=0)P(T_1=0)\\
&=E_\bZ(T_1\Delta_2^2\mid T_1=1)P(T_1=1)
\end{align*}
and $P(T_1=1)=E\{\pi(\bS_1)\}\geq M^{-1}>0$ under Assumption \ref{cond:seq-ign}, hence
\begin{align*}
E_\bZ(T_1\Delta_2^2\mid T_1=1)=O_p(\delta_N^4).
\end{align*}
\end{proof}

\begin{proof}[Proof of Theorem \ref{cor:mu-DR}]
We first show that
\begin{align*}
\mu^0(\bs_1)=E(Y^\#\mid \bS_1=\bs_1,T_1=1),\;\;\forall\bs_1\in\R^{d_1}.
\end{align*}
Indeed,
\begin{align*}
&E(Y^\#\mid\bS_1,T_1=1)=E\left[\nu^0(\bSbar_2)+\frac{T_2\{Y(1,1)-\nu^0(\bSbar_2)\}}{\rho^0(\bSbar_2)}\mid\bS_1,T_1=1\right]\\
&\qquad\overset{(i)}{=}E\left(E\left[\nu^0(\bSbar_2)+\frac{T_2\{Y(1,1)-\nu^0(\bSbar_2)\}}{\rho^0(\bSbar_2)}\mid\bSbar_2,T_1=1\right]\mid\bS_1,T_1=1\right)\\
&\qquad\overset{(ii)}{=}E\left[\nu^0(\bSbar_2)+\frac{E(T_2\mid\bSbar_2,T_1=1)E\{Y(1,1)-\nu^0(\bSbar_2)\mid\bSbar_2,T_1=1\}}{\rho^0(\bSbar_2)}\mid\bS_1,T_1=1\right]\\
&\qquad\overset{(iii)}{=}\mu^0(\bS_1),
\end{align*}
where (i) holds by the law of total expectation; (ii) holds under Assumption \ref{cond:seq-ign}; (iii) holds since $\nu^0(\bSbar_2)=E\{Y(1,1)\mid\bSbar_2,T_1=1\}$ and $E\{\nu^0(\bSbar_2)\mid\bS_1,T_1=1\}=E\{E[Y(1,1)\mid\bSbar_2,T_1=1\}\mid\bS_1,T_1=1]=\mu^0(\bS_1)$.

In addition, observe that
\begin{align*}
\|T_1Y^\#\|_\infty&\overset{(i)}{=}\left\|T_1\nu^0(\bSbar_2)+\frac{T_1T_2\{Y-\nu^0(\bSbar_2)\}}{\rho^0(\bSbar_2)}\right\|_\infty\\
&\leq\|\nu^0\|_\infty+\|1/\rho^0\|_\infty\left(\|Y\|_\infty+\|\nu^0\|_\infty\right)\overset{(ii)}=O(1),
\end{align*}
where (i) holds under Assumption \ref{cond:seq-ign}; (ii) holds under Assumptions \ref{cond:seq-ign} and \ref{cond:bound}. Additionally,
\begin{align*}
\|T_1\Delta_1\|_{\infty,P_\bZ}&\overset{(i)}{=}\biggr\|\left\{T_1-\frac{T_1T_2}{\rho^0(\bSbar_2)}\right\}\{\nuhat^{(2)}(\bSbar_2)-\nu^0(\bSbar_2)\}\\
&\qquad+\left\{\frac{T_1T_2}{\rhohat^{(2)}(\bSbar_2)}-\frac{T_1T_2}{\rho^0(\bSbar_2)}\right\}\{Y-\nu^0(\bSbar_2)\}\biggr\|_{\infty,P_\bZ}\\
&\leq(1+\|1/\rho^0\|_\infty)(\|\nuhat^{(2)}\|_{\infty,P_\bZ}+\|\nu^0\|_\infty)\\
&\qquad+(\|1/\rhohat^{(2)}\|_{\infty,P_\bZ}+\|1/\rho^0\|_\infty)\left(\|Y\|_\infty+\|\nu^0\|_\infty\right)\\
&\overset{(ii)}{=}O_p(1),
\end{align*}
where (i) holds under Assumption \ref{cond:seq-ign}; (ii) holds under Assumptions \ref{cond:seq-ign}, \ref{cond:bound}, and the fact that $\|\nu^0\|_\infty\leq M$ as $\|Y\|_\infty\leq M$.

Together with Lemma \ref{lemma:imput-err}, we have verified the required conditions to apply Theorem \ref{thm:imp-MLP}. Hence, it follows that
\begin{align*}
E\left[T_1\{\muhat^1(\bS_1)-\mu^0(\bS_1)\}^2\right]=O_p\left(n^{-\frac{\beta_\mu}{\beta_\mu+q_\mu}}\log^8 n+d_1n^{-\frac{2\beta_\mu+q_\mu}{2\beta_\mu+2q_\mu}}\log^5n+r_N^2+\delta_N^4\right).
\end{align*}
Repeating the same procedure above, we also have the same consistency rate for $\muhat^2(\bS_1)$. Therefore, the average $\muhat$ satisfies \eqref{rate:muhat}.
\end{proof}

\begin{proof}[Proof of Theorem \ref{thm:normal-DR}]
By Theorem \ref{cor:mu-DR} and note that $n\asymp N$, we know that \eqref{def:bN} is satisfied with $b_N^2=\bar b_N^2+\delta_N^4$. By Lemma \ref{lemma:gen}, we have
\begin{align*}
\thetahat-\theta=O_p\left(a_Nd_N+b_Nc_N+N^{-1/2}\right)=O_p\left(a_Nd_N+\bar b_Nc_N+\delta_N^2c_N+N^{-1/2}\right).
\end{align*}
Now, we additionally assume that $a_N,c_N,d_N,r_N,\delta_N=o(1)$, $a_Nd_N+\bar b_Nc_N+\delta_N^2c_N=o(N^{-1/2})$ and $\sigma^2>c>0$. Then, we also have $b_N=o(1)$ and $b_Nc_N=o(N^{-1/2})$. Hence, Lemma \ref{lemma:normal} implies that, as $N\to\infty$, in distribution,
\begin{align}\label{result:rN-normal}
\thetahat-\theta=O_p\left(N^{-1/2}\right)\;\;\text{and}\;\;\sqrt N \sigma^{-1}(\thetahat-\theta)\to N(0,1).
\end{align}

We now show that $\sigmahat^2=\sigma^2\{1+o_p(1)\}$. Firstly, by \eqref{bound:psi}, we have $\sigma^2=\Var\{\psi(\bZ;\eta^0)\}=O(1)$. Together with the assumption that $\sigma^2>c>0$, we have $\sigma\asymp1$. By \eqref{bound:psi-diff}, we have
\begin{align*}
E_{\S_k}\left\{\psi(\bZ_i;\etahat_{-k})-\psi(\bZ_i;\eta^0)\right\}^2=O_p(a_N^2+b_N^2+c_N^2+d_N^2)=o_p(1).
\end{align*}
Hence,
\begin{align*}
E_{\S_k}\left[|I_k|^{-1}\sum_{i\in I_k}\{\psi(\bZ_i;\etahat_{-k})-\psi(\bZ_i;\eta^0)\}^2\right]=E_{\S_k}\left\{\psi(\bZ_i;\etahat_{-k})-\psi(\bZ_i;\eta^0)\right\}^2=o_p(1).
\end{align*}
By Markov's inequality,
\begin{align*}
|I_k|^{-1}\sum_{i\in I_k}\{\psi(\bZ_i;\etahat_{-k})-\psi(\bZ_i;\eta^0)\}^2=o_p(1).
\end{align*}
By Lemma S.14 of \cite{bradic2024high}, together with \eqref{result:rN-normal} and note that $E\{\psi(\bZ;\eta^0)-\theta\}^4=O(1)$ using the fact \eqref{bound:psi}, we conclude that
\begin{align*}
\sigmahat^2=\sigma^2\{1+o_p(1)\}.
\end{align*}
\end{proof}


\end{document}

%% file: Definitions.tex
\def\bx{\mathbf{x}}
\def\bX{\mathbf{X}}

\def\bZ{\mathbf{Z}}

\def\bt{\mathbf{t}}

\def\balpha{\boldsymbol{\alpha}}
\def\bbeta{\boldsymbol{\beta}}

\def\etahat{\widehat{\eta}}

\def\ba{\boldsymbol{a}}

\def\bbetahat{\widehat{\bbeta}}

\def\Yhat{\widehat{Y}}

\def\fhat{\widehat{f}}
\def\hhat{\widehat{h}}

\def\independenT#1#2{\mathrel{\rlap{$#1#2$}\mkern4mu{#1#2}}}

\def\Var{\mbox{Var}}

\def\bT{\mathbf{T}}

\def\bS{\mathbf{S}}
\def\bs{\mathbf{s}}

\def\bSbar{\overline{\bS}}
\def\bsbar{\overline{\bs}}

\def\S{\mathbb{S}}

\def\R{\mathbb{R}}

\def\S{\mathbb{S}}

\def\bt{\mathbf{t}}

\def \hs2{\hspace{2mm}}

\numberwithin{table}{section}
\numberwithin{equation}{section}

\definecolor{jcolor}{RGB}{041,122,000}
\definecolor{darkred}{RGB}{100,000,000}
\definecolor{purple}{RGB}{200,000,200}

\def\boxit#1{\vbox{\hrule\hbox{\vrule\kern6pt  \vbox{\kern6pt#1\kern6pt}\kern6pt\vrule}\hrule}}

\def\muhat{\widehat{\mu}}

\def\muhat{\widehat{\mu}}

\def\muhat{\widehat{\mu}}

\def\bT{\mathbf{T}}
\def\bS{\mathbf{S}}

\def\pihat{\widehat{\pi}}
\def\rhohat{\widehat{\rho}}

\def\nuhat{\widehat{\nu}}

\def\bT{\mathbf{T}}

\def\sigmahat{\widehat{\sigma}}

\def\thetahat{\widehat{\theta}}